\documentclass[twoside]{article}

\usepackage[accepted]{aistats2025}
\usepackage[round]{natbib}

\usepackage{afterpage}
\usepackage{graphicx} 
\usepackage{amsthm}
\usepackage{amsmath}
\usepackage{amsfonts}
\usepackage{dsfont, pifont}
\newcommand{\indicator}{\mathds{1}}
\usepackage{hyperref}
\usepackage{xcolor}
\usepackage{booktabs}
\usepackage{nicefrac}
\usepackage{subcaption}
\usepackage{rotating}
\usepackage{adjustbox}
\usepackage{multirow}
\usepackage{tabularx}
\usepackage{todonotes}
\usepackage{tikz}
\newtheorem{theorem}{Theorem}

\DeclareMathOperator*{\argmax}{arg\,max}
\DeclareMathOperator*{\argmin}{arg\,min}
\newcommand{\RS}[0]{\textsc{RS}}
\newcommand{\CC}[0]{\textsc{CC}}
\newcommand{\LCE}[0]{\textsc{LCE}}
\newcommand{\OVA}[0]{\textsc{OVA}}
\newcommand{\SP}[0]{\textsc{SP}}
\newcommand{\DT}[0]{\textsc{DT}}
\newcommand{\ASM}[0]{\textsc{ASM}}

\newcommand{\ok}{\overline{\kappa}}
\newcommand{\tatd}{\tau_{\mathtt{ATD}}}
\newcommand{\htatd}{\hat{\tau}_{\mathtt{ATD}}}
\newcommand{\tcatd}{\tau_{\mathtt{CATD}}}
\newcommand{\htcatd}{\hat{\tau}_{\mathtt{CATD}}}

\newcommand{\trdd}{\tau_{\mathtt{RD}}}
\newcommand{\htrdd}{\hat{\tau}_{\mathtt{RD}}}

\newcommand{\Acc}{\widehat{\mathtt{Acc}}}
\newcommand{\cprob}{\overset{p}{\longrightarrow}}

\usepackage{siunitx}
\newtheorem{assumption}{Assumption}
\makeatother
\newtheorem{proposition}{Proposition}
\newtheorem{scenario}{Scenario}
\theoremstyle{definition}
\newtheorem{exmp}{Example} 

\begin{document}

%

%
\runningauthor{Palomba, Pugnana, {\'{A}}lvarez, Ruggieri}

\twocolumn[

\aistatstitle{A Causal Framework for Evaluating Deferring Systems}

\aistatsauthor{ Filippo Palomba\rm{*} \And Andrea Pugnana\rm{*}$^{, \dagger}$ \And  Jos{\'{e}} M. {\'{A}}lvarez \And  Salvatore Ruggieri}

\aistatsaddress{ Princeton University \And  University of Pisa\\University of Trento \And KU Leuven \And University of Pisa } ]

\begin{abstract}
    Deferring systems extend supervised Machine Learning (ML) models with the possibility to defer predictions to human experts. However, evaluating the impact of a deferring strategy on system accuracy is still an overlooked area. This paper fills this gap by evaluating deferring systems through a causal lens. We link the potential outcomes framework for causal inference with deferring systems, which allows to identify the causal impact of the deferring strategy on predictive accuracy. We distinguish two scenarios. In the first one, we have access to both the human and ML model predictions for the deferred instances. Here, we can identify the individual causal effects for deferred instances and the aggregates of them. In the second one, only human predictions are available for the deferred instances. Here, we can resort to regression discontinuity designs to estimate a local causal effect. We evaluate our approach on synthetic and real datasets for seven deferring systems from the literature. 
\end{abstract}

\def\thefootnote{*}\footnotetext{Equal contribution. Order is chosen alphabetically.}
\def\thefootnote{$\dagger$}\footnotetext{Work submitted while at the University of Pisa.}
\def\thefootnote{\arabic{footnote}}
\section{INTRODUCTION}
\label{sec:introduction}

Learning to defer (LtD) \citep{DBLP:conf/nips/MadrasPZ18} extends supervised learning by allowing Machine Learning (ML) models to defer the prediction to a human expert. The extended models -- called \textit{deferring systems} -- aim at obtaining the best from the combination of AI and human expert predictions, thus reducing potentially harmful mistakes.
The LtD research field is blooming, with novel approaches continuously appearing (e.g., \cite{DBLP:conf/aistats/MozannarLWSDS23,DBLP:conf/nips/CaoM0W023,DBLP:conf/aistats/LiuCZF024,DBLP:conf/icml/WeiC024}). 
However, all these works evaluate deferring systems by looking at the final accuracy obtained by the human-AI team.
This accuracy-based view on evaluation is limited as it does not consider the \textit{causal effect} of introducing the deferring strategy on the system's predictive performance.
When evaluating the impact of a new drug or a new policy, e.g., one has not only to observe a positive change in the outcome of interest, but also has to ensure the change in that outcome was due to the performed intervention~\citep{DBLP:journals/widm/NogueiraPRPG22}.
Similarly, due to mounting regulatory pressure,
policymakers are interested in understanding the causal effect of introducing a deferring system, in particular, within a high-stake decision-making process \citep{DBLP:journals/ethicsit/AlvarezCEFFFGMPLRSSZR24}.

Consider the following two examples: 
\textit{(Ex1)} an online platform introduces a new deferring system to moderate its content for hate speech, meaning most content moderation is still automated, but a small part is now handled by humans; 
\textit{(Ex2)} a hospital introduces a new deferring system for diagnosing a disease, meaning that medical doctors will still handle part of the patients, while an ML model will diagnose the remaining cases.
After some months, the stakeholders of the online platform in \textit{Ex1} may ask the developers of the deferring system to quantify the causal effects of \textit{deferring to humans} instead of automatic content moderation. 
Similarly, the stakeholders of the hospital in \textit{Ex2} may ask for the causal effects of \textit{deferring to the ML model} instead of full human decision-making.
Both examples require a causal inference approach because the goal is to estimate the causal effect of a variable (the introduction of a deferring system) on another one (the predictive performance)~\citep{Pearl2000}.

In this work, we link deferring systems with the causal inference framework of \emph{potential outcomes} \citep{rubin1974EstimatingCausalEffects} by mapping concepts from the former to the latter. We distinguish two scenarios. In the first one, we can access the ML model predictions for both deferred and non-deferred instances, and the human predictions only for deferred ones. 
\textit{Ex1} belongs to such a scenario. In this context, various causal quantities \textit{of deferring to humans} can be readily identified and estimated. In the second scenario described by \textit{Ex2}, we can access the ML model predictions only for the non-deferred instances and the human predictions only for the deferred ones. In this context, we rely on \emph{Regression Discontinuity} (RD) design \citep{thistlethwaite1960RegressiondiscontinuityAnalysisAlternative} to identify and estimate a local causal effect, where local refers to the boundary of the deferring decision.
Such a local causal effect covers both the causal effect \textit{of deferring to humans} and the one \textit{of deferring to the ML model}, as they are one the opposite of the other.
\textit{Ex2} belongs to such a scenario. 

Our contributions are threefold:
$(i)$ we frame the evaluation of deferring systems as a causal inference problem using the potential outcomes framework; 
$(ii)$ we investigate two scenarios and for each show which causal effects can be identified and estimated; and
$(iii)$ we evaluate the proposed approach on five datasets -- a synthetic one and four real-world ones -- using seven deferring systems from the literature. 
We introduce causal inference and deferring systems in Section \ref{sec:background}. We bridge the two frameworks and investigate the two scenarios of causal estimation in Section \ref{sec:methodology}. We report experiments in Section \ref{sec:exp}. We conclude in Section~\ref{sec:conclusions}.

\section{BACKGROUND}\label{sec:background}

\subsection{Causal Inference} 
\label{sec:background-causal}
The core task of causal inference is to estimate the \textit{causal effect} of a binary \textit{treatment} random variable $D \in\{0,1\}$ on another discrete or continuous \textit{outcome} random variable $O\in\mathcal{O}$.
Let us consider a random sample 
$\{D_i, O_i\}_{i=1}^n$ of i.i.d.~variables, where the subscript $i$ denotes a specific instance/unit~$i$. We denote realizations of such random variables with lowercase letters.
A formal definition of a causal effect is given by the Neyman-Rubin causal framework \citep{neyman1923applications,rubin1974EstimatingCausalEffects} through the notion of \textit{potential outcomes}. 
A potential outcome $O(d)\in\mathcal{O},d\in\{0,1\}$ is a random variable representing the value that the outcome variable $O$ would take when the treatment variable is set to $d$. Accordingly, the (individual) causal effect of $D$ on $O$ for unit $i$ is defined as $\tau_{i}=O_i(1)-O_i(0)$.\footnote{
We make the \textit{stable unit treatment value assumption} (SUTVA) \citep{rubin1978BayesianInferenceCausal}. It requires each unit's potential outcome not to depend on the treatment assignment of other units, ruling out interference between units.}

If we were able to observe the joint distribution of $(O(0), O(1))$, then the causal effect of each unit could be readily estimated from a dataset of observations. However, for each unit $i$, only one among $O_i(1)$ and $O_i(0)$ can be typically observed. This is called the ``fundamental problem of causal inference" \citep{holland1986StatisticsCausalInference}. It occurs since the observed outcome $O_i$ and the potential outcomes are related by $O_i = D_i \cdot O_i(1) + (1-D_i)\cdot O_i(0)$. 
In other words, if a unit $i$ is assigned to the treatment ($D_i=1$), then the potential outcome $O_i(0)$ is counterfactual in nature, and we would not observe it. Symmetrically, $O_i(1)$ is counterfactual for units not assigned to treatment ($D_i=0$). For this reason, researchers are often interested in less granular causal quantities, such as:
\begin{align}\label{eq:causalestimands}
\begin{split}
    \tau_{\mathtt{ATE}}:= \mathbb{E}[O(1)-O(0)], \qquad \text{and}\\
    \tau_{\mathtt{ATT}}:= \mathbb{E}[O(1)-O(0)\mid D=1],
\end{split}
\end{align}
known as the \textit{average treatment effect} (ATE) and the \textit{average treatment effect on the treated} (ATT), respectively.\footnote{The subscripts of $O_i$ and $D_i$ can be omitted in the expectation since variables are\textit{ i.i.d.}} 
Despite being more general than the individual causal effect, the causal estimands in \eqref{eq:causalestimands} cannot be estimated from a dataset of observations unless some assumptions are imposed, as the distribution of $(D,O(0),O(1))$ is (\textit{i}) unknown and (\textit{ii}) generally impossible to learn from the data because of the fundamental problem of causal inference. 
Several methodologies have been proposed to use context-dependent knowledge to model $(D,O(0),O(1))$. See \cite{abadie2018EconometricMethodsProgram} for a recent review.

The RD design~\citep{thistlethwaite1960RegressiondiscontinuityAnalysisAlternative} is one of such methodologies. 
In the canonical RD design, units are assigned a score $V\in\mathbb{R}$, known as \textit{running variable}, and ranked according to it. A unit $i$ whose running variable $V_i$ is greater or equal than a \textit{cutoff} value $\xi$ is assigned to treatment, otherwise it does not receive the treatment. It follows that the treatment assignment is known, deterministic, and can be described by $D_i=\indicator\{V_i\geq \xi\}$. This knowledge of the assignment process can be exploited to identify and estimate causal effects. Indeed, if we can assume that units in the vicinity of the cutoff are similar, then the RD design can be used to identify:
\[\tau_{\mathtt{RD}}:= \mathbb{E}[O(1) - O(0)\mid V=\xi].\]
This quantity can be interpreted as a version of $\tau_{\mathtt{ATT}}$ ``local" at the cutoff.  The above heuristics was formalized by \citet{hahn2001IdentificationEstimationTreatment} in terms of potential outcomes through the following assumption. 
\begin{assumption}[RD-continuity]
\label{ass:continuity}
    The expected potential outcomes are continuous at the cutoff, namely, there exist:
    \[\lim_{v\to \xi} \mathbb{E}[O(0)\mid V=v] \qquad\text{and}\qquad \lim_{v\to \xi}\mathbb{E}[O(1)\mid V=v]. \]
\end{assumption}
Assumption \ref{ass:continuity} requires the average potential outcomes not to change abruptly in a small neighbourhood around the cutoff; hence their left and right limits exist and are equal. Under Assumption \ref{ass:continuity}, $\tau_{\mathtt{RD}}$ can be identified from the data, as the next theorem shows.
\begin{theorem}[Theorem 3 from \cite{hahn2001IdentificationEstimationTreatment}]\label{thm:rdidentifiability}
    Let Assumption \ref{ass:continuity} hold. Then:
        \[\lim_{v\to \xi_+}\mathbb{E}[O\mid V=v]-\lim_{v\to \xi_-}\mathbb{E}[O\mid V=v] = \tau_{\mathtt{RD}}.\]
\end{theorem}
Theorem \ref{thm:rdidentifiability} is ``local'' in nature, as it shows that the average treatment effect on the treated can be identified for a specific sub-population of units,~namely those with $V=\xi$. 

\subsection{Deferring Systems}
\label{sec:backgroundDS}
Let $\mathcal{X}\subseteq\mathbb{R}^q$ be a $q$-dimensional input space, $\mathcal{Y} = \{ 1, \ldots, m\}$ be the target space and  $P(\mathbf{X}, Y)$ be the probability distribution over $\mathcal{X}\times \mathcal{Y}$. 
Given a hypothesis space $\mathcal{F}$ of functions that map $\mathcal{X}$ to $\mathcal{Y}$, the goal of supervised learning is to find the hypothesis $f\in\mathcal{F}$ that minimizes the \textit{risk}:
\begin{equation*}
    R(f) =
\mathbb{E}[l(f(\mathbf{X}),Y)] = \int_{\mathcal{X}\times\mathcal{Y}}l(f(\mathbf{x}),y)d\mathcal{P}(\mathbf{x},y) ,
\end{equation*}
where  $l:\mathcal{Y}\times \mathcal{Y} \to \mathbb{R}$ is a user-specified loss function and $\mathcal{P}$ is the probability measure linked to the joint distribution $P$ over the space $\mathcal{X}\times\mathcal{Y}$.
Here, $f$ represents a ML model (i.e.,~a predictor).
Because $P(\mathbf{X}, Y)$ is generally unknown, it is typically assumed that we have access to a set of realizations, called a \textit{training set}, of an \textit{i.i.d.}~random sample over $P(\mathbf{X}, Y)$.
The training set is used to learn a predictor $\hat{f}$, 
such that 
$\hat{f} \in \argmin_{f\in \mathcal{F}} \widehat{R}(f)$, with $\widehat{R}(f)$ denoting the empirical counterpart of the risk $R(f)$ over the training set.

Since the predictor $\hat{f}$ can make mistakes, we can extend the above canonical setting by allowing the ML model to defer difficult cases to another predictor. We consider a human expert as another predictor $h:\mathcal{Z}\to\mathcal{Y}$, where $\mathcal{Z}$ is possibly a higher dimensional space than $\mathcal{X}$. To keep the notation simple, we consider the case where $\mathcal{Z}=\mathcal{X}$. The mechanism that determines who provides the prediction is called the \textit{policy function} (or rejector/deferring strategy) and can be formally defined as a (binary) mapping $g:\mathcal{X}\to\{0,1\}$. 
We define the \textit{deferring system} $\vartheta$, also called the~\textit{the human-AI team}, as a triplet $(f,g,h)$, such that:
\begin{equation*}
    \vartheta(\mathbf{x}) = (f,g, h)(\mathbf{x}) = \begin{cases}
        f(\mathbf{x}) \quad\text{ if }\quad g(\mathbf{x})=0\\
        h(\mathbf{x}) \quad\text{ if }\quad g(\mathbf{x})=1
    \end{cases}
\end{equation*}
meaning, if $g(\mathbf{x})=0$, the prediction is provided by the ML model, while if $g(\mathbf{x})=1$, the human expert takes care of the prediction. 
We assume a \emph{single} human expert to defer the prediction to, thus excluding generalizations that pick which expert to defer to \citep{DBLP:conf/aistats/VermaBN23,DBLP:conf/nips/MaoMM023}.
Let $\mathcal{G}$ be the set of all the policy functions and $\mathcal{L}(f,g)$ the expected risk of the whole deferring system, namely:
\begin{align}
\begin{split}
    \mathcal{L}(f,g) =    \int_{\mathcal{X}\times\mathcal{Y}}l_{ML}\left(f(\mathbf{x}),y\right)\left(1-g(\mathbf{x})\right)d\mathcal{P}(\mathbf{x},y) + \\\int_{\mathcal{X}\times\mathcal{Y}}l_{H}(h(\mathbf{x}),y)g(\mathbf{x})d\mathcal{P}(\mathbf{x},y),
\end{split}
    \label{eq:loss}
\end{align}
with $l_{ML}$ (resp., $l_{H}$) referring to the loss associated with the ML model (resp., human expert). The goal becomes finding the best $f\in\mathcal{F}$ and $g\in\mathcal{G}$ such that:
\begin{equation*}
    \min_{f\in\mathcal{F}, g\in\mathcal{G}}\mathcal{L}(f,g) \quad \text{s.t.}\quad \mathbb{E}[g(\mathbf{X})]\leq 1-c, 
\end{equation*}
where $c\in[0,1]$ is a \textit{target coverage}, i.e.,~a user-specified minimum fraction of instances for which the ML model is selected to make predictions.

Most methods design the deferring strategy through a \textit{reject score} function $k:\mathcal{X}\to\mathbb{R}$, which estimates whether the human expert prediction is more likely to be correct than the one of the ML model~\citep{DBLP:conf/aistats/MozannarLWSDS23}. 
High values of $k(\mathbf{x})$ correspond to cases where the human expert is preferable, i.e.,~it is more likely to provide a correct prediction. Hence, we can set a threshold $\ok$ over $k(\mathbf{x})$ to define the policy function as $g(\mathbf{x}) = \indicator\{k(\mathbf{x})\geq\ok\}$. \citet{DBLP:conf/nips/OkatiDG21} show that such a thresholding strategy is optimal. In practice, one can estimate such a threshold in various ways. For instance, if there are no coverage constraints, a linear search procedure can be run by selecting the $\ok$ that maximizes accuracy over a validation set  \citep{DBLP:conf/aistats/MozannarLWSDS23}. Otherwise, one can consider a \textit{coverage-calibration} procedure by setting $\ok$ as the $c^{\mathit{th}}$-percentile of the reject score values over a validation set \citep{Benchmark2024}, as shown in Figure \ref{fig:coverage_cal_example}. To highlight the relationship between $\ok$ and $c$, we denote the estimated threshold for a target coverage $c$ as $\ok_c$, i.e.,~$\ok_c$ is such that $\mathbb{E}[\indicator\{k(\mathbf{X})\geq \ok_c\}]=(1-c)$.
\subsection{Related Work}\label{sec:relatedwork}
By viewing the introduction of human-AI teams as an intervention on a ML-based or on a human-based decision flow, our work bridges causal inference for policy evaluation with deferring systems. 
To the best of our knowledge, the only related work is \citet{DBLP:conf/nips/ChoeGR23}, which estimates the accuracy of 
abstaining classifiers on the abstained instances under the assumption that the abstention policy is stochastic. Such an assumption is impractical in the context of deferring systems. Further, abstaining classifiers do not account for deferring to humans.\,Our work addresses both these shortcomings. See Appendix \ref{sec:appRelatedWork} for additional related work.

\paragraph{Policy Evaluation.}
\noindent
Causal inference methods are used to evaluate the effects of treatments/policies, and inform policymakers~\citep{abadie2018EconometricMethodsProgram}.
Randomized control trials (RCT), i.e.,~experiments that assign units to treatment randomly, are the gold standard for inferring average causal effects (such as $\tau_{\mathtt{ATE}}$) in many fields including healthcare \citep{stolberg2004randomized}, education \citep{carlana2022}, 
and finance \citep{banerjee2015miracle}. When relying on RCTs is not possible, e.g., it is too costly or unethical, one can resort to other techniques using observational data \citep{angrist2009MostlyHarmlessEconometrics}. 
The RD design has been used to assess the effectiveness of treatments in several fields, such as healthcare \citep{boon2021regression, cattaneo2023guide}, criminal behavior \citep*{pinotti2017ClickingHeavenDoor}, education \citep*{angrist1999UsingMaimonidesRule, duflo2011PeerEffectsTeacher}, public economics  \citep*{lalive2008HowExtendedBenefits, battistin2009RetirementConsumptionPuzzle, coviello2014PublicityRequirementsPublic}, and corporate finance \citep*{flammer2015DoesCorporateSocial}.

\paragraph{Deferring Systems Applications.}
\noindent
In recent years, deferring systems have been deployed in several domains~\citep{RP25}.
For instance, \citet{DBLP:conf/sdm/PlasMVD23} implement a deferring system for sleep stage scoring, which can be used to allow physicians to focus on critical patients. \citet{Strong25} present a pilot for disorder classification using a Large Language Model as the deferring system.
\citet{DBLP:conf/bigdataconf/CianciGGKMPRR23} adopt selective classification \citep{DBLP:journals/jmlr/El-YanivW10} for uncertainty self-assessment of a credit scoring ML model, with the purpose of informing the human decision maker. \citet{DBLP:conf/aaai/BondiKSCBCPD22} study a deferring system for evaluating the presence of animals in photo traps, showing that the performance of the deferring system is influenced by how the deferral choice is communicated to humans. \citet{DBLP:conf/iui/HemmerWSVVS23} present a user study to evaluate how deferring impacts human performance in an image classification task.
We refer to \citet{punzi2024ai} for a recent survey on hybrid decision-making.

\section{EVALUATING DEFERRING SYSTEMS}\label{sec:methodology}
\begin{table*}[!th]
    \centering
    \caption{Deferring systems under the Potential Outcomes lens.}
\resizebox{.7\textwidth}{!}{
    \begin{tabular}{ccc}
      & \textbf{POTENTIAL OUTCOMES}  & \textbf{LtD} \\
      \midrule
     \textit{running variable} & $V_i$ & $K_i$ \\
    \textit{cutoff} & $\xi$ & $\ok_c$\\
     \textit{treatment} & $D_i$ & $G_i$ \\
     \textit{outcome} & $O_i$  & $T_i$ \\
     \textit{potential outcomes} & $O_i(d),d\in\{0,1\}$& $T_i(g),g\in\{0,1\}$\\
     $\tau_{i}$
     &$O_i(1)-O_i(0)$ & $T_i(1)-T_i(0)$\\
    \end{tabular}
    }
    \label{tab: RD to LtD mapping}
\end{table*}
\begin{figure*}[t!]
\begin{minipage}[t]{.315\textwidth}
    \includegraphics[scale=.25]{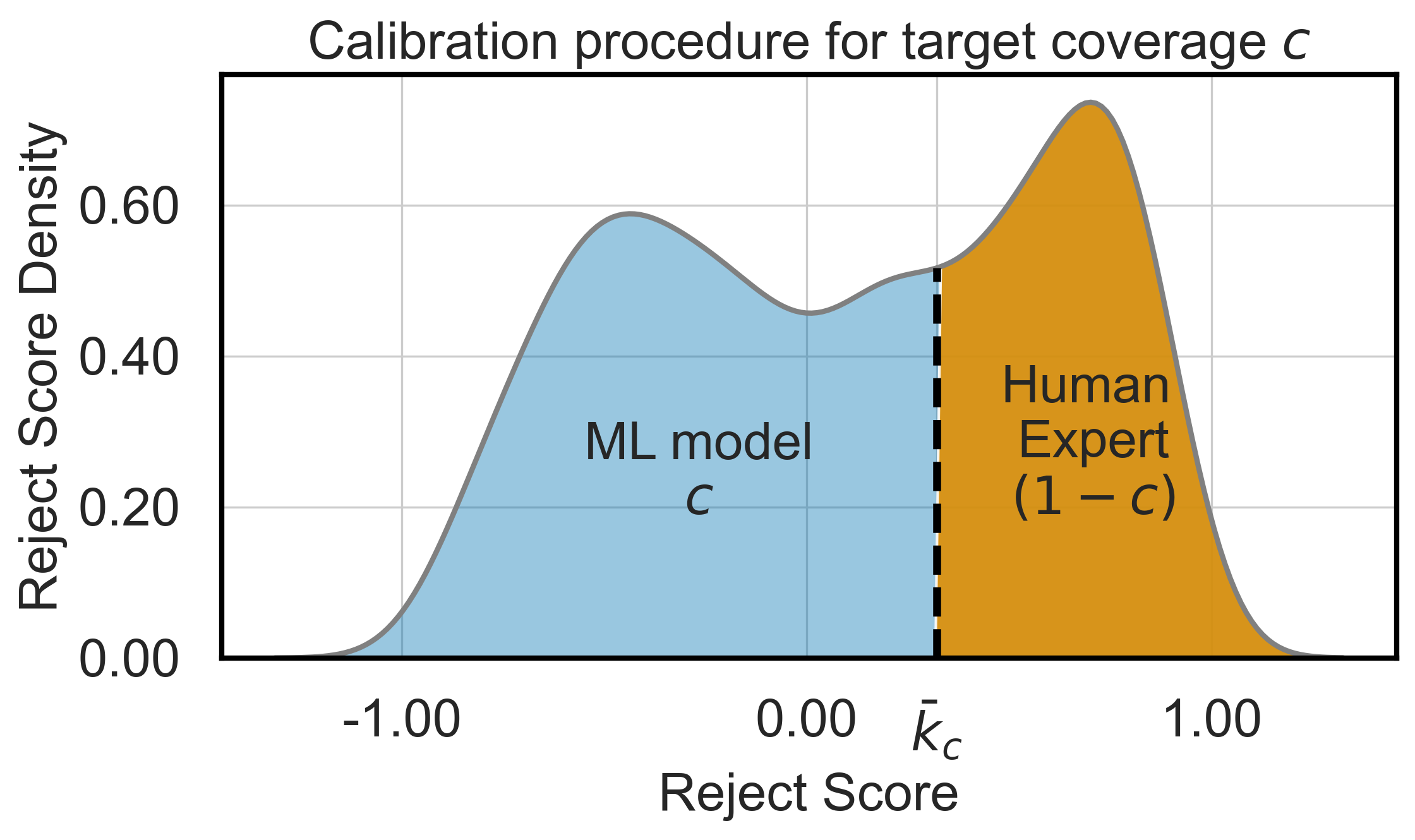}
    \caption{
    In blue, the
    $({c})\%$ of instances 
    assigned to the ML model;
    in orange, the $(1-{c})\%$ instances assigned to the
    the human.
    }
    \label{fig:coverage_cal_example}
\end{minipage}
\hfill
\begin{minipage}[t]{.325\textwidth}
    \centering
    \includegraphics[scale=.25]{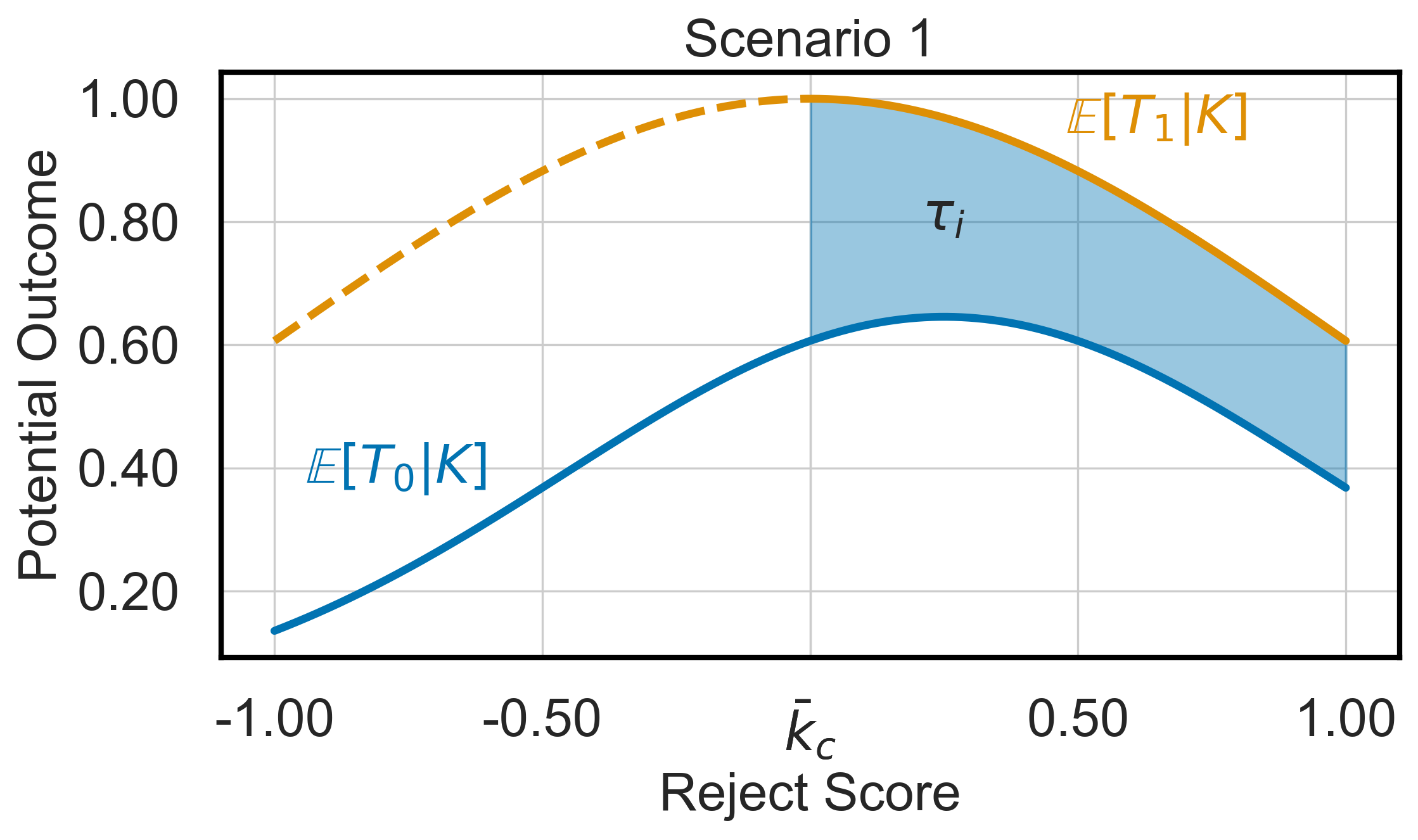}
    \caption{Scenario~\ref{scenario1} assumptions: thick (dashed) lines are observed (unobserved) values. The coloured area represents where the effects can be estimated.
}
    \label{fig:atdIdentification}
\end{minipage}    
\hfill
\begin{minipage}[t]{.315\textwidth}
    \centering
    \includegraphics[scale=.25]{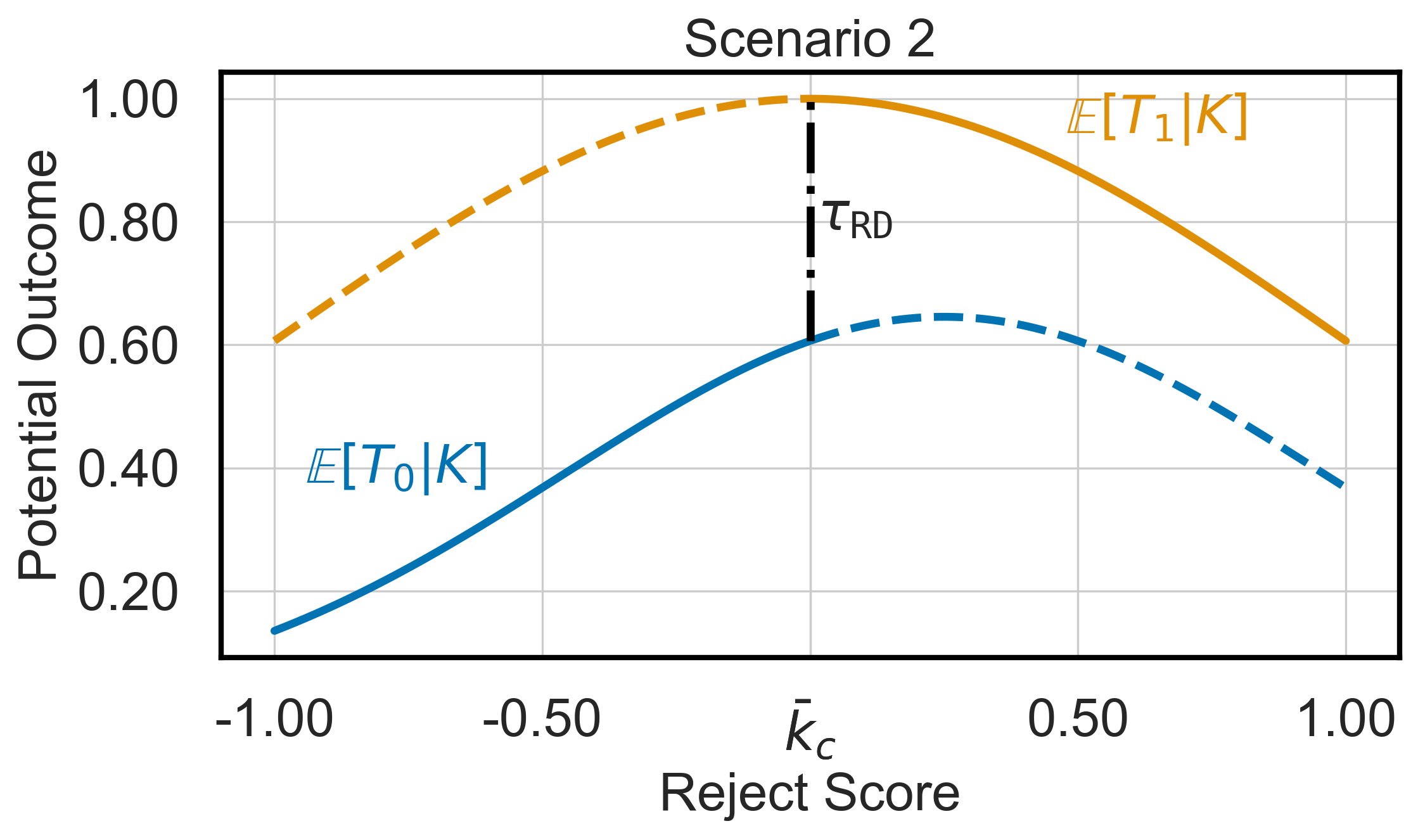}
    \caption{Scenario~\ref{scenario2} assumptions: thick (dashed) lines are observed (unobserved) values.~We can estimate $\trdd$ at the cutoff value.
    }
    \label{fig:rddIdentification}
\end{minipage}    
\end{figure*}

\noindent
\textbf{Problem Statement.}
We consider the problem of measuring the causal contribution, in terms of predictive performance, of deferring the prediction to a human expert in a given deferring system based on a reject score function. We assume a given set of realizations~$\mathcal{D}_n=\{(\mathbf{x}_i, y_i)\}_{i=1}^n$, called a \textit{test set}, of an \textit{i.i.d.}~random sample $\{(\mathbf{X}_i, Y_i)\}_{i=1}^n$ over  $P(\mathbf{X}, Y)$, and we define the test sample accuracy of a generic predictor $m$ as $\Acc_m := \frac{1}{n}\sum_{i=1}^n\indicator\{m(\mathbf{x}_i)=y_i\}$.

\noindent
\paragraph{Methodology.}
To tackle the problem above, we bridge deferring systems with the potential outcomes framework. The key observation is that the reject score maps to the running variable, and the predictive performance of the ML model and of the human expert map to the potential outcomes. 

Table~\ref{tab: RD to LtD mapping} links each variable of a deferring system to the potential outcomes framework. We have that: $(i)$ \emph{the reject score} $K_i=k(\mathbf{X}_i)$ is the running variable; $(ii)$ the \emph{threshold}~$\ok_c \in\mathcal{K}$ is the cutoff; $(iii)$ the  \emph{policy function} $G_i=\indicator\{K_i\geq \ok_c\}$ is the treatment assignment: if $G_i=1$, the human expert provides the prediction $h(\mathbf{X}_i)$, otherwise the ML model provides the prediction $f(\mathbf{X}_i)$; $(iv)$ \emph{the outcome} is $T_i=u(\vartheta(\mathbf{X}_i, Y_i))$, where $u:\mathcal{Y}\times\mathcal{Y}\to\mathbb{R}$ is a function taking as inputs the deferring system prediction $\vartheta(\mathbf{X}_i)$ and its target variable $Y_i$, e.g.,~ if $u(\vartheta(\mathbf{X}_i), Y_i)=\indicator\{\vartheta(\mathbf{X}_i)=Y_i\}$, the outcome is the correctness of the prediction; $(v)$ $T_i(0) = u\left(f(\mathbf{X}_i),Y_i\right)$ and $T_i(1)= u\left(h(\mathbf{X}_i),Y_i\right)$ are \emph{the potential outcomes}, e.g.,~if $u(f(\mathbf{X}_i, Y_i))=\indicator\{f(\mathbf{X}_i)=Y_i\}$ and if $u(h(\mathbf{X}_i, Y_i))=\indicator\{h(\mathbf{X}_i)=Y_i\}$, then the potential outcome are the correctness of the ML model prediction and of the human prediction, respectively; finally, $(vi)$ \emph{the individual causal effect} $\tau_{i}$ is the difference between $T_i(1)$ and $T_i(0)$. Notice that the outcome and the potential outcomes satisfy the expected condition $T_i=G_i \cdot T_i(1)+(1-G_i) \cdot T_i(0)$.

We distinguish two scenarios that allow for the identification of causal effects under Table~\ref{tab: RD to LtD mapping}. 
In both scenarios, we assume that the human predictions $h(\mathbf{X}_i)$ can be accessed only for the deferred instances, i.e.,~if $G_i=1$. 
In the rest of the paper, we consider $u\left(\vartheta(\mathbf{X}_i), Y_i\right)=\indicator\{\vartheta(\mathbf{X}_i) = Y_i\}$ for a direct interpretation of our results in terms of accuracy, but these results can be directly extended to other functions $u$.

\begin{scenario}
\label{scenario1}
    The ML model predictions $f(\mathbf{X}_i)$ can be accessed for \textbf{the whole random sample}.
\end{scenario}
This scenario covers cases in which the ML model can be called without any cost or side effects.
E.g., when the development team runs an internal evaluation of the deferring system. Here, we can identify the causal effects \textit{of deferring to the human} (Section \ref{sec:s1}). 
\textit{Ex1} from Section~\ref{sec:introduction} falls under this scenario. 
\begin{scenario}
\label{scenario2}
    The ML model predictions $f(\mathbf{X}_i)$ can be accessed \textbf{only for the non-deferred instances}, i.e.,~if $G_i=0$.
\end{scenario}
This scenario covers cases in which model invocation is costly, may have side effects, or discloses sensitive data.
E.g., when in the owner of the deferring system is reluctant to share the deferred ML predictions during an external audit.\footnote{Since these predictions are not the system's actual output, releasing them might not be legally binding.}
Here, we can identify the causal effects \textit{of deferring to the human} locally to the deferring threshold (Section \ref{sec:s2}). Moreover, this scenario also covers cases in which the intervention to be evaluated is the introduction of the ML model within a human decision-making process, as in \textit{Ex2} from Section~\ref{sec:introduction}. Importantly, in \textit{Ex2}, we are not interested in the causal effect of deferring from the ML model to the human expert, but rather \textit{of deferring from the human expert to the ML model}. Therefore, Scenario 1 does not apply if we reverse the role of the ML model and the human expert because we cannot assume having the human expert's predictions for the cases assigned to the ML model. However, since the local causal effect of deferring from the human expert to the ML model is the opposite of the one of deferring from the ML model to the human expert, we can rely on Scenario 2 for estimating the effect.

\subsection{Scenario 1: deferring systems as an almost perfect causal inference design}\label{sec:s1}

In this scenario, we are in \textit{the ideal situation} in which both potential outcomes are observed for the deferred instances ($G_i=1)$, as both the ML model $f(\mathbf{X}_i)$ and the human $h(\mathbf{X}_i)$ predictions are available. Recalling that $n$ is the size of the test set $\mathcal{D}_n$, the following holds.

\begin{proposition}
\label{pro:ITED}
    Let Scenario \ref{scenario1} hold. Then, for each $i\in[1,n]$ such that 
    $G_i=1$:
    \[\tau_{i} = T_i(1)-T_i(0) =\indicator\{h(\mathbf{X}_i)=Y_i\} - \indicator\{f(\mathbf{X}_i)=Y_i\}.\]
\end{proposition}

All the proofs are in Appendix \ref{sec:appendixProof}.
Because we can compute the most granular causal effect on deferred instances, $\tau_{i}$, we can also retrieve less granular quantities (see Figure \ref{fig:atdIdentification}). For instance, if we average the $\tau_{i}$ over the population of deferred units, we obtain the \textit{average treatment effect on the 
deferred}, $\tatd$, which is the deferring systems' equivalent of $\tau_{\mathtt{ATT}}$ in \eqref{eq:causalestimands}. We identify this causal estimand as follows.
\begin{proposition}
\label{pro:ATD}
    Let Scenario \ref{scenario1} hold. Then: 
    \begin{align*}
        \tatd = \mathbb{E}[T(1)-T(0)|G=1]= \\
        \mathbb{E}\left[\indicator\{h(\mathbf{X})=Y\}\mid G=1\right] - \mathbb{E}\left[ \indicator\{f(\mathbf{X})=Y\}\mid G=1\right].
    \end{align*}
\end{proposition}
$\tatd$ allows to measure the (average) effect on accuracy due to the ``intervention" of deferring to a human the prediction for the deferred instances. 
Intuitively,
$\tatd$ estimates what would be the average increase in accuracy for deferred instances if the human predicts instead of the ML model. Hence, $\tatd$ motivates introducing a deferring system to stakeholders. 
Policymakers can use it to assess the impact of such systems. 
We can estimate $\tatd$ through a difference-in-means estimator of the form:
\begin{align*}
    \htatd = \frac{1}{n_1}\sum_{i\in[1,n]:g(\mathbf{x}_i)=1} \left[t_i(1) - t_i(0)\right] = \\
    \frac{1}{n_1}\sum_{i\in[1,n]:g(\mathbf{x}_i)=1} \left[\indicator\{h(\mathbf{x}_i)=y_i\} - \indicator\{f(\mathbf{x}_i)=y_i\}\right],
\end{align*}
where $n_1 := |\{i\in[1,n]:g(\mathbf{x}_i)=1\}|$ is the number of deferred instances in the test set.

The next proposition highlights that $(i)$ $\htatd$ identifies a causal quantity of interest, and $(ii)$ $\htatd$ differs from the simple difference, typically used in the literature, between the accuracy of the deferring system and the one of the ML model $\hat{\tau}_{\Delta} := \Acc_\vartheta- \Acc_f$.

\begin{proposition}
\label{pro:Inconsistency}
Let Scenario 1 hold. Then:
\\
$(i)$ as $n\to\infty$,
$\htatd \cprob \mathbb{E}[T(1) -T(0)\mid G=1]=\tatd$
$(ii)$ for each $n >0$,
$\hat{\tau}_{\Delta} = \frac{n_1}{n}\htatd.$
\end{proposition}
A consequence of Proposition \ref{pro:Inconsistency} is that $\hat{\tau}_{\Delta}$ is an inconsistent estimator for the causal effect $\tatd$ unless all units are deferred. Another consequence is that we can re-weight $\hat{\tau}_{\Delta}$ by $n/n_1$ to obtain a consistent estimator. We discuss this further in Appendix \ref{sec:furtherDis}.

To conclude, we highlight that under Scenario \ref{scenario1}, any aggregated metrics of the individual causal effects on the deferred can be estimated. For instance, we could estimate the \textit{conditional average treatment effects on the deferred} $\tcatd(\tilde{\mathbf{x}}) = \mathbb{E}[T(1) -T(0)|G=1, \mathbf{X}=\tilde{\mathbf{x}}]$ by further conditioning on a specific set of features $\mathbf{X}=\tilde{\mathbf{x}}$. Similarly to $\tatd$, $\tcatd$ can be estimated by considering the difference in means estimator $\htcatd$:
\[\htcatd(\tilde{\mathbf{x}}) =
\frac{1}{n_{1,\tilde{x}}}\sum_{i\in[1,n]:g(\mathbf{x}_i)=1, \mathbf{x}_i = \tilde{\mathbf{x}}} \left[t_i(1) - t_i(0)\right],
\] 
where $n_{1,\tilde{x}}:=| \{ i\in[1,n]: (g(\mathbf{x}_i) =1) \land (\mathbf{x}_i = \tilde{\mathbf{x}})\} |$.

\subsection{Scenario 2: deferring systems as an RD design}\label{sec:s2}

If we cannot access the ML model predictions for the deferred instances, we can exploit the fact that \textit{deferring systems can be interpreted as an RD design} to compute a local version of $\tau_{\mathtt{ATD}}$. In this scenario, $\tau_{\mathtt{RD}}$ answers to the question: for a fixed coverage value $c$ and the corresponding reject score threshold $\ok_c$, what would be the increase in accuracy if we let the human expert predict instead of the ML model for the instances with reject score close to $\ok_c$? Such an interpretation is possible if Assumption \ref{ass:continuity} holds. If we have reasons to believe that small changes to the reject-score threshold $\ok_c$ do not abruptly change the expected predictive accuracy of the human expert and of the ML model, then Assumption \ref{ass:continuity} is satisfied. 
\begin{proposition}
\label{thm:ltdidentifiability}
    Let Scenario \ref{scenario2} hold and let Assumption \ref{ass:continuity} be satisfied for the deferring system. Then:
    \[\lim_{k\to \ok_c^+}\mathbb{E}[T\mid K=k]-\lim_{k\to \ok_c^-}\mathbb{E}[T\mid K=k] = \trdd,\]
    where
    $\trdd := \mathbb{E}[T(1) - T(0) \mid K = \ok_c]$.
\end{proposition}
Proposition~\ref{thm:ltdidentifiability} allows to evaluate the causal effect of deferring to a human in a decision flow even if we do not have access to ML predictions for the deferred instances. Indeed, $\trdd$ readily quantifies the gain in predictive accuracy of having the human expert predicting in place of the ML model at the cutoff (see Figure \ref{fig:rddIdentification}).
The local nature of $\trdd$ motivates the use of local non-parametric polynomial kernel regression\footnote{Local polynomial kernel regressions \citep{fan1996LocalPolynomialModelling} fit a $p$-th order polynomial locally at $k$ via weighted least squares, where the weight of each instance is determined by the shape of the kernel and is non-increasing in the distance between $k$ and $k(\mathbf{x}_i)$.} to estimate $\mathbb{E}[T\mid K=\ok_c]$ from the left and the right of the cutoff, thus obtaining an estimator $\htrdd$ of $\trdd$. 
The $\trdd$ can also be computed under Scenario \ref{scenario1}. In Appendix~\ref{sec:appendixCaveats}, we discuss additional caveats, including how to set the optimal coverage, how to check if Assumption \ref{ass:continuity} holds, and the uncertainty due to the ML model estimation.

\section{EXPERIMENTAL EVALUATION}\label{sec:exp}

In this section, we address three questions: 

\textbf{Q1}: \textit{For Scenario \ref{scenario1}, what is the causal effect on predictive accuracy of introducing a deferring system?} 

\textbf{Q2}: \textit{What other causal effects can be computed under Scenario \ref{scenario1}?}

\textbf{Q3}: \textit{For Scenario \ref{scenario2}, what is the causal effect on predictive accuracy of introducing a deferring system?}

We consider both synthetic and real data, detailing all results in Appendix \ref{sec:appendixExtraRes}.
The experimental software is available at \url{https://github.com/andrepugni/PODS}. 
We report hardware specifications and carbon footprint in Appendix \ref{sec:appendixHardware}.

\subsection{Experimental settings}
\label{sec:experimentSettings}

\paragraph{Data.}
We generate synthetic data using the procedure from \citet{DBLP:conf/aistats/MozannarLWSDS23}. Such a procedure generates samples containing (\textit{i}) instances for which the human expert performs better than the ML model, and (\textit{ii}) instances for which the ML model is better than the human expert.
Regarding the real data, 
we consider four datasets used in the LtD literature:
\texttt{cifar10h} \citep{DBLP:journals/nature/BattledayPG2020}, a hard-labelled version of
\texttt{galaxyzoo} \citep{galaxy-zoo-the-galaxy-challenge}, 
\texttt{hatespeech} \citep{DBLP:conf/icwsm/DavidsonWMW17}, and 
\texttt{xray-airspace} \citep{DBLP:conf/cvpr/WangPLLBS17,majkowska2020chest}. 
We provide data characteristics and applied pre-processing in Appendix~\ref{sec:appendixData}.

\paragraph{Baselines.}
We consider several deferring systems, including: 
\textit{Selective Prediction} (\SP{}) \citep{DBLP:conf/nips/GeifmanE17}, \textit{Compare Confidence} (\CC{}) \citep{Raghu2019}, \textit{Differentiable Triage} (\DT{}) \citep{DBLP:conf/nips/OkatiDG21}, \textit{Cross-Entropy Surrogate} (\LCE) \citep{DBLP:conf/icml/MozannarS20}, \textit{One Vs All} (\OVA{}) \citep{DBLP:conf/icml/VermaN22},  \textit{Realizable Surrogate} (\RS{}) \citep{DBLP:conf/aistats/MozannarLWSDS23} and \textit{Asymmetric SoftMax} (\ASM{}) \citep{DBLP:conf/nips/CaoM0W023}.
We provide details for the baselines and the hyper-parameter choice in Appendices~\ref{sec:appendixBaselines}, \ref{sec:appHyper}.

\paragraph{General setup.}
\label{sec:experimentsGS}
For all the experiments, we consider the following steps: 
$(i)$ we randomly split the dataset in training, validation, and test set, according to a $70\%,10\%,20\%$ proportion; 
$(ii)$ we train the deferring system on the training set; 
$(iii)$ we estimate different cutoff values $\ok_c$ over the validation set, considering each of the following target coverages~$c\in\{.10, .20, .30, .40, .50, .60, .70, .80, .90\}$; $(iv)$ for each cutoff value $\ok_c$, we estimate on the test set the deferring system accuracy as well as $\tatd$ and $\trdd$. We consider a single training, validation, and test split to adhere to real-world applications where $(i)$ the deferring system is considered as given (we do not need to validate its predictive accuracy with multiple splits), and $(ii)$ the goal is to estimate the causal effect on that specific test sample (hence, we do not have multiple splits).
The estimated $\htatd$ and $\htcatd$ are computed through a difference in means estimator and $\htrdd$ is obtained using the default implementation provided in the \texttt{rdrobust} package \citep{calonico2017RdrobustSoftwareRegressiondiscontinuity}, i.e.,~a local linear kernel regression with optimal bandwidth \citep{calonico2020optimal}. 
To assess the statistical significance of the results, we report the 95\% confidence intervals and the corresponding $p$-values ($pv$) associated with $\htatd$ and $\htrdd$ when testing the null hypotheses\footnote{We correct for multiple testing through the Bonferroni correction for all the tested hypotheses (five datasets, ten coverages under Scenario 1 and nine under Scenario 2, seven methods): to achieve significance at a family-wise error rate of $\alpha=.05$, the $p$-value must be smaller than $\num{7.52e-5}$.} of $\tatd=0$ and $\trdd=0$, respectively.

\subsection{Experimental results}
\label{sec:experiments}

Figure \ref{fig:experiments} shows the experimental results.
For each dataset, we analyze the best deferring system in terms of accuracy, as shown in the first row of Figure~\ref{fig:S1res}. We provide the experimental results for the other baselines in Appendix~\ref{sec:appendixExtraRes}.

\label{sec:expQ1}

\begin{figure*}[t]
\begin{subfigure}[t]{.75\textwidth}
    \includegraphics[scale=.2]{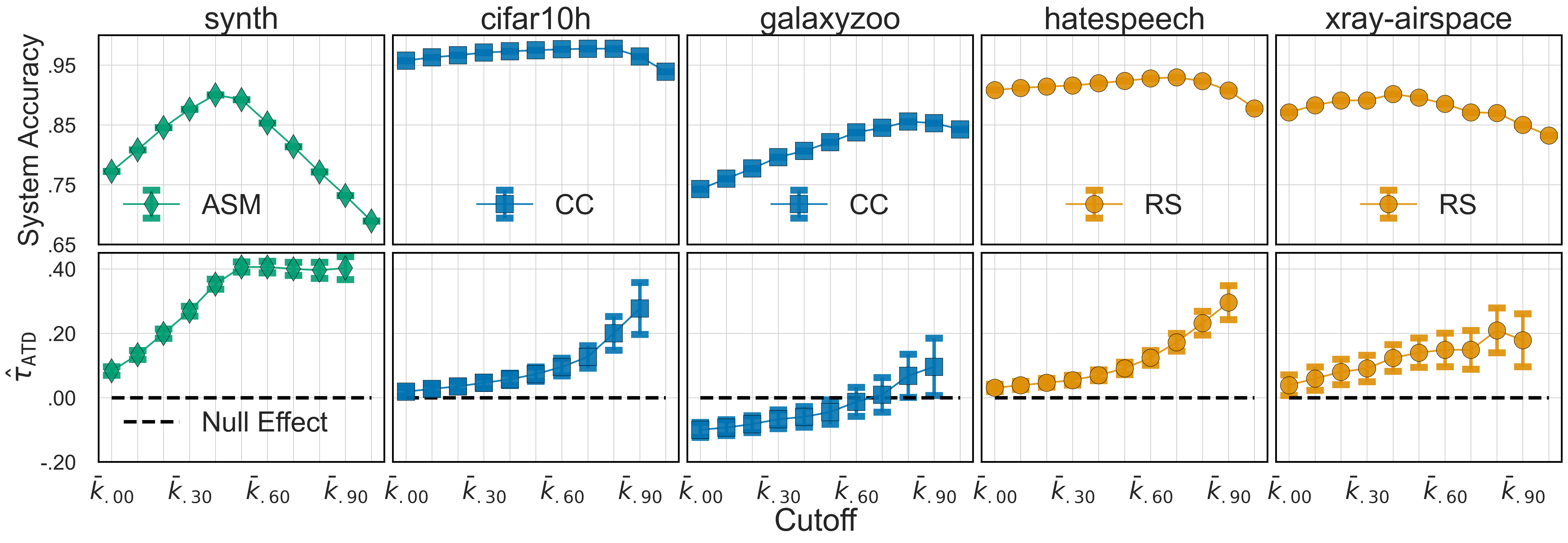}
    \caption{Best baseline's accuracy (top row) and estimated $\htatd$ (bottom row) when varying $\ok_c$.}
    \label{fig:S1res}
\end{subfigure}
\hfill
\begin{subfigure}[t]{.23\textwidth}
\includegraphics[scale=.20]{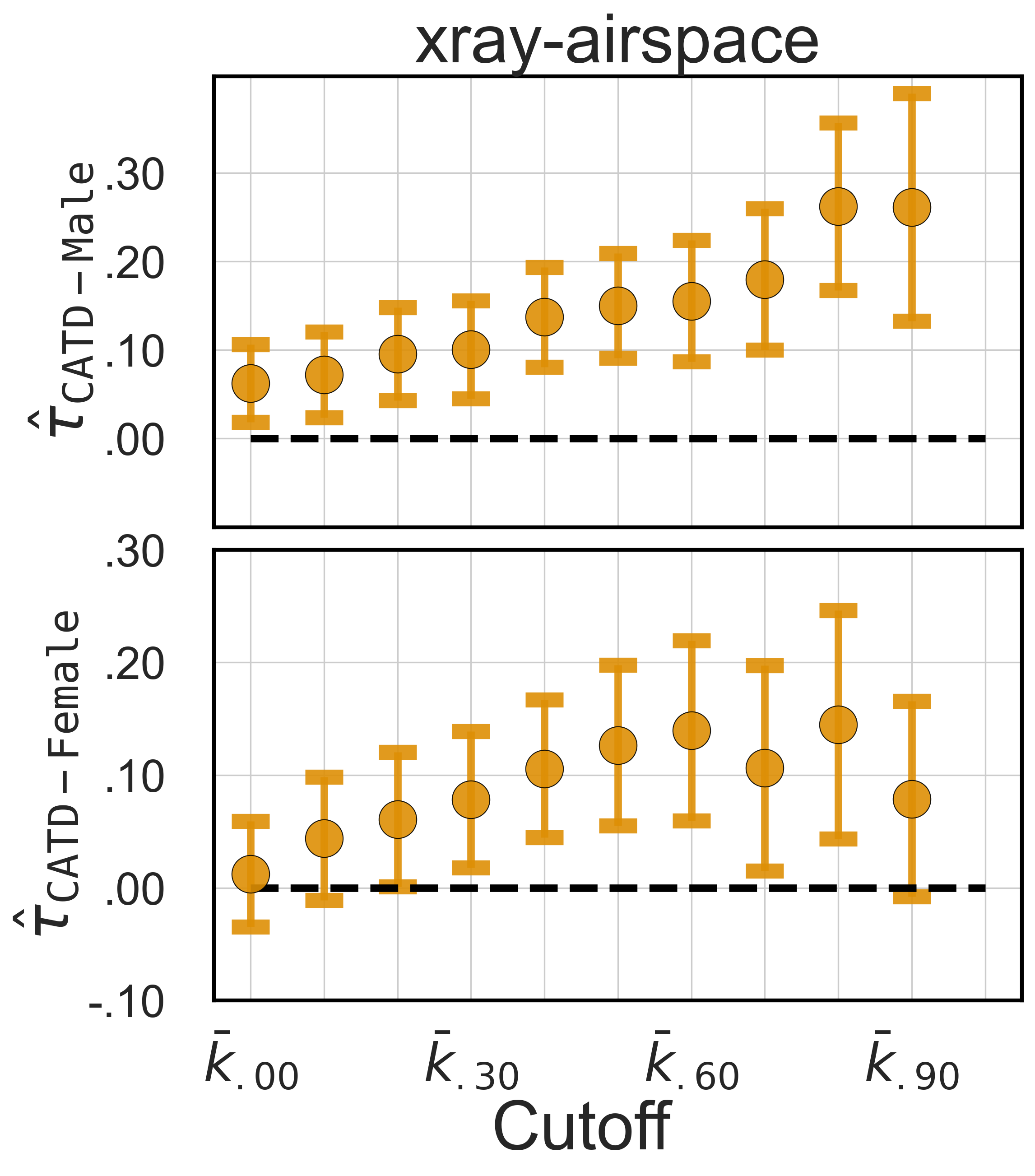}
        \caption{$\htcatd$ by gender.}
        
    \label{fig:S1CATE}
\end{subfigure}
\begin{subfigure}[t]{.24\textwidth}
\centering
\includegraphics[scale=.2]{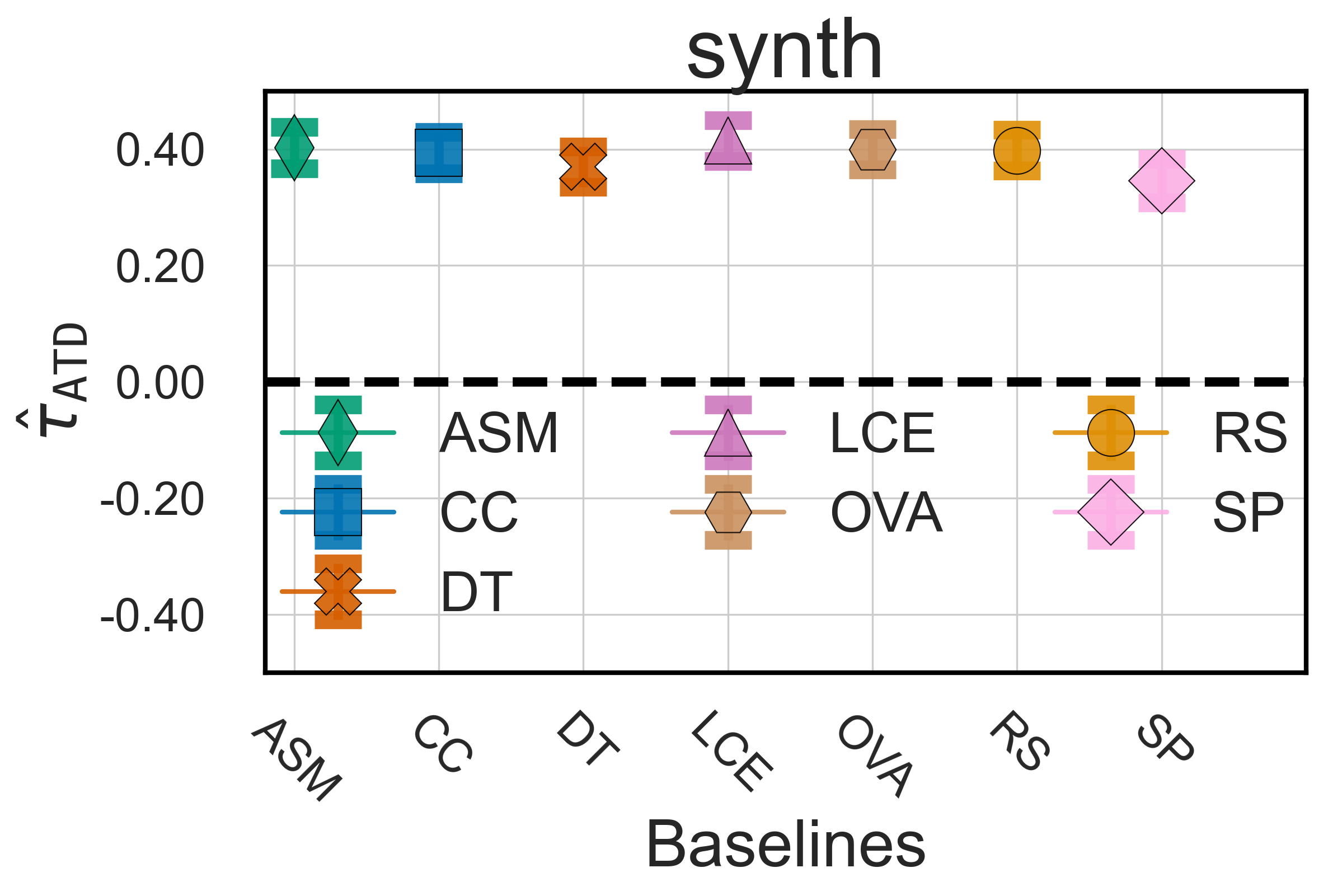}
        \caption{$\htatd$ at $c=.90$ for \texttt{synth}.}
        \label{fig:S1comparison}
\end{subfigure}
\hfill
\begin{subfigure}[t]{.75\textwidth}
    \includegraphics[scale=.20]{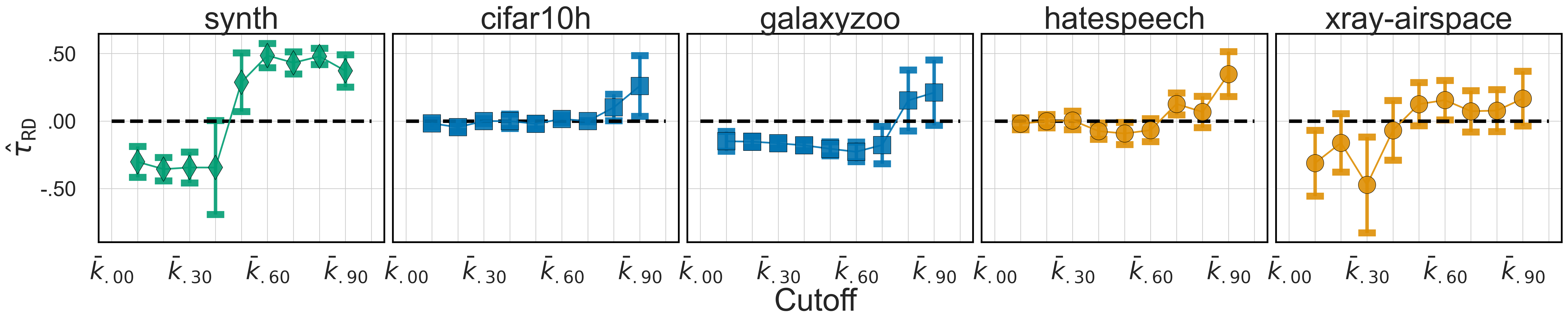}
    \caption{Estimated $\htrdd$ when varying cutoff  $\ok_c$ for the best baseline.}
    \label{fig:S2res}
\end{subfigure}
\caption{Experimental results: Figure \ref{fig:S1res} reports system accuracy and $\htatd$ (Scenario 1); Figure \ref{fig:S1CATE} reports estimated $\htcatd$ when conditioning on the gender of the patient on the \texttt{xray-airspace} dataset; Figure \ref{fig:S1comparison} compares $\htatd$ over multiple baselines on \texttt{synth}; Figure \ref{fig:S2res} reports $\htrdd$ (Scenario 2).}
\label{fig:experiments}
\end{figure*}

\paragraph{Q1: causal effects under Scenario 1.}

The second row of Figure \ref{fig:S1res} shows $\htatd$'s and their confidence intervals when varying the cutoff $\ok_c$. The black horizontal line denotes the null effect. 

For \texttt{synth}, the plot confirms the effectiveness of \ASM{}, with an increasing trend in the estimated causal effects, ranging from~$\approx.095$~$(pv\approx\num{6.98e-26})$ at zero coverage to $\approx.417$ ($pv\approx\num{7.32e-58}$) at $c=.90$. 
We see similar patterns for real data: regarding \texttt{cifar10h}, the $\htatd$'s are positive and increasing with $\ok_c$: the values range from $\approx.019$ at $c=0$  to $\approx .265$ $(pv\approx\num{2.83e-10})$ at $c=.90$. 
Also, for \texttt{hatespeech}, the causal effects monotonically increase with the coverage, with $\htatd$ ranging from $\approx.019$ $(pv\approx\num{9.54e-18})$ at $c=0$ up to $\approx.295$ $(pv\approx\num{7.40e-28})$ at $c=.90$. 
We observe an overall positive effect for the \texttt{xray-airspace} dataset as well, with the highest $\htatd$ achieved at $c=.80$ ($\approx.209$, $pv\approx\num{4.29e-9}$). All the estimates for those datasets significantly differ from zero, supporting the effectiveness of deferring.
The \texttt{galaxyzoo} dataset is an exception, with $\htatd$ taking negative values for the coverages below $c=.70$. Interestingly, despite an increase in the system accuracy for $c=.80$ and $c=.90$, the $\htatd$ are not statistically different from zero, with the $\htatd$ of $\approx.068$ $(pv\approx\num{4.69e-2})$ and $\approx.096$ $(pv\approx\num{3.56e-2})$ respectively). 
Hence, introducing a deferring system would not improve over a fully automated setting.

We point out that this causal effect cannot be quantified by examining the accuracy of the deferring system (see Proposition \ref{pro:Inconsistency}\footnote{Notice that $\hat{\tau}_{\Delta}$ in Proposition \ref{pro:Inconsistency} turns out to be the accuracy of Figure \ref{fig:S1res} minus a constant value, precisely the accuracy at $c = 1.0$.}), as normally done by previous works. In fact, the shape of the top and bottom plots in Figure \ref{fig:S1res} clearly differ.

To conclude, Figure \ref{fig:S1comparison}) compares the estimated $\htatd$'s for a fixed target coverage ($c = .90$) over multiple deferring systems.

\paragraph{Q2:\,conditional causal effects under Scenario\,1.} 
When designing a deferring system, we aim at not introducing new forms of bias or unfairness w.r.t. protected-by-law groups \citep{Ruggieri2023}. 
Since under Scenario 1, we can estimate the individual \textit{causal} effects for the deferred units, we can also study the causal effect on predictive accuracy of deferring for instances belonging to different social groups. In particular,~if $\tcatd$ is negative for the deferred instances belonging to a protected group, we can conclude that introducing the deferring system has negatively affected the predictive performance for such instances. Such an effect is causal, i.e., it can be attributed to the adoption of the deferring strategy.

In Figure \ref{fig:S1CATE}, we report such causal effects for the \texttt{xray-airspace} dataset. We consider gender as the protected attribute and compute the heterogeneous causal effect $\htcatd$ for male and female patients. As shown in the top plot, when looking at male patients, the estimated $\htcatd$ grows with the cutoff, and statistically significant effects can be observed for $c>.30$, reaching a maximum of $\approx.262$ ($pv-\num{5.83e-8}$) at $c=.80$. On the other hand, when looking at female patients (bottom plot), we see smaller positive effects of introducing a deferring system, and such effects are not statistically different from zero. This implies that introducing a deferring system benefits male patients but not female patients. We plan to further explore this fairness angle in our proposed causal evaluation framework in future work, since causal claims are essential e.g.,~from a legal perspective \cite{Bathaee2018}.

\paragraph{Q3: causal effects under Scenario 2.}
Under Scenario \ref{scenario2}, we expect the local causal effects $\trdd$ to be negative for instances where the ML model locally performs better than the human; positive for instances where the human locally outperforms the ML model; and not significant for those where the ML model and the human locally perform equally. In practice, we should see the estimated $\htrdd$ close to zero for the cutoffs that maximize the overall deferring system accuracy.

Figure \ref{fig:S2res} reports the estimated $\htrdd$'s. For \texttt{synth}, we see significant negative effects for $c\leq.30$, with the lowest $\htrdd$ of $\approx-.356$ ($pv\approx\num{9.60e-16}$) at $c=.20$; non-significant effects for $c$ between $.40$ and $.50$ and significant positive effects for $c\geq.60$, with a $\htrdd$ peaking to $\approx.485$ ($pv\num{3.07e-26}$) at $c=.60$. This aligns with the expected optimal behaviour, as the system accuracy is maximal between $c=.40$ and $c=.50$.

When looking at real datasets, we see many statistically non-significant effects (see Tables \ref{tab: synth-results}--\ref{tab: xray-airspace-results}~in the Appendix). This might be due to two factors: $(i)$ the test size is small, limiting the statistical power of the local polynomial regressions; $(ii)$ there are small differences in system accuracy at the variation of the cutoff (see, e.g., \texttt{cifar10h} and \texttt{hatespeech} in Figure \ref{fig:S1res}), hence locally to the deferring boundary the difference between the ML model and the human predictor is expected to be small. A few exceptions can be noticed. For \texttt{galaxyzoo}, the ML model performs better than the human expert on this task, translating into a few negative and statistically significant $\htrdd$ coefficients. For \texttt{hatespeech}, we see a statistically significant effect $\approx.348~(pv\approx \num{4.1e-5})$ for \RS{} at $\ok_c=.90$. Thus, the causal effect of deferring to the human expert on instances around the cutoff $\ok_{.90}$ is positive, supporting the effectiveness of the deferring strategy. 

Finally, when comparing Figure \ref{fig:S2res} to the top row of Figure \ref{fig:S1res}, we notice that $\htrdd$ is negative whenever the system accuracy is increasing, it is positive when the system accuracy drops, and it is near zero when the system accuracy flattens. This happens because $\htrdd$ is the opposite of the local change in accuracy when we slightly increase the cutoff and, thus, defer less.

\section{CONCLUSIONS}
\label{sec:conclusions}

We tackled the evaluation of the predictive performance of deferring systems from a causal perspective. Our link with the potential outcomes and with the RD design frameworks directly allows for identifying the causal effects of the deferring strategy. 
Experiments on synthetic and real datasets provided practical guidance on how to reason about the causal estimation problem.

\paragraph{Limitations and broader impact.}
\label{sec:limitations}
We assumed access to the reject scores $k(\mathbf{x})$ of instances. This is not strictly required under Scenario 1, for which only the information on whether an instance is deferred or not is required. Conversely, identifying $\trdd$ under Scenario~2 requires at least knowing the ranking induced by the reject score.
Moreover, Assumption \ref{ass:continuity} must hold to identify $\trdd$. However, such an assumption is not directly verifiable and can only be falsified. We discuss this further in Appendix~\ref{sec:appendixCaveats} and show how to falsify Assumption \ref{ass:continuity} in Appendix \ref{sec:appendixValidate}.
From a broader perspective, our evaluation framework helps to better quantify the impact of deferring systems and deploy safer ones, especially in high-risk contexts.

\paragraph{Future work.}
We plan to extend our approach $(i)$ to deferring strategies that consider multiple human experts~\citep{DBLP:conf/nips/MaoMM023,DBLP:conf/aistats/VermaBN23}; $(ii)$ to other forms of human-AI teaming, e.g., through conformal prediction~\citep{detoni2024towards} $(iii)$; to support interpretable and fair mechanisms~\citep{DBLP:conf/pkdd/LendersPPCPG24}, and $(iv)$ to account for the influence that deferring has on human behaviour, e.g.,~in strategic classification~\citep{DBLP:conf/innovations/HardtMPW16}. Finally, $(v)$ exploiting the connection between learning to defer and the policy evaluation literature \citep{DBLP:conf/icml/DudikLL11,DBLP:conf/nips/BennettK19}, we could move from policy evaluation towards (online) learning policies that maximize the causal effect of deferring.
\section*{Acknowledgments}
A. Pugnana and S. Ruggieri have received funding by PNRR - M4C2 - Investimento 1.3, Partenariato Esteso PE00000013 - ``FAIR - Future Artificial Intelligence Research" - Spoke 1 ``Human-centered AI", funded by the European Commission under the NextGeneration EU programme.
This work was also funded by the European Union under Grant Agreement no. 101120763 - TANGO. Views and opinions expressed are however those of the author(s) only and do not necessarily reflect those of the European Union or the European Health and Digital Executive Agency (HaDEA). Neither the European Union nor the granting authority can be held responsible for them. 

\bibliographystyle{apalike}
\bibliography{biblio}


\section*{Checklist}
 \begin{enumerate}

 \item For all models and algorithms presented, check if you include:
 \begin{enumerate}
   \item A clear description of the mathematical setting, assumptions, algorithm, and/or model. [Yes] See Table \ref{tab:exp_details}.
   \item An analysis of the properties and complexity (time, space, sample size) of any algorithm. [Yes] We include the time and carbon footprint required in Section~\ref{sec:appendixHardware}, while for sample size, we refer to Table~\ref{tab:exp_details}.
   \item (Optional) Anonymized source code, with specification of all dependencies, including external libraries. [Yes] The code can be found at \url{https://github.com/andrepugni/PODS}.
 \end{enumerate}

 \item For any theoretical claim, check if you include:
 \begin{enumerate}
   \item Statements of the full set of assumptions of all theoretical results. [Yes] All the assumptions are made in the main text.
   \item Complete proofs of all theoretical results. [Yes] Proofs are in Appendix \ref{sec:appendixProof} for space reasons.
   \item Clear explanations of any assumptions. [Yes] We detail the two scenarios in the main paper. See Section \ref{sec:methodology}.
 \end{enumerate}

 \item For all figures and tables that present empirical results, check if you include:
 \begin{enumerate}
   \item The code, data, and instructions needed to reproduce the main experimental results (either in the supplemental material or as a URL). [Yes] The code can be found at \url{https://github.com/andrepugni/PODS}. 
   \item All the training details (e.g., data splits, hyperparameters, how they were chosen). [Yes] See Section \ref{sec:experimentSettings} in the main paper and Sections\ref{sec:appHyper} and \ref{sec:appendixData} in the Appendix.
     \item A clear definition of the specific measure or statistics and error bars (e.g., with respect to the random seed after running experiments multiple times). [Yes] We detail these measures in Section \ref{sec:experimentSettings}.
     \item A description of the computing infrastructure used. (e.g., type of GPUs, internal cluster, or cloud provider). [Yes] See \ref{sec:appendixHardware}.
 \end{enumerate}

 \item If you are using existing assets (e.g., code, data, models) or curating/releasing new assets, check if you include:
 \begin{enumerate}
   \item Citations of the creator If your work uses existing assets. [Yes]
   \item The license information of the assets, if applicable. [Yes]
   \item New assets either in the supplemental material or as a URL, if applicable. [Yes] The code can be found at \url{https://github.com/andrepugni/PODS}.
   \item Information about consent from data providers/curators. [Not Applicable]
   \item Discussion of sensible content if applicable, e.g., personally identifiable information or offensive content. [Not Applicable]
 \end{enumerate}

 \item If you used crowdsourcing or conducted research with human subjects, check if you include:
 \begin{enumerate}
   \item The full text of instructions given to participants and screenshots. [Not Applicable]
   \item Descriptions of potential participant risks, with links to Institutional Review Board (IRB) approvals if applicable. [Not Applicable]
   \item The estimated hourly wage paid to participants and the total amount spent on participant compensation. [Not Applicable]
 \end{enumerate}

 \end{enumerate}

\appendix
\clearpage
\onecolumn

\section{EXTENDED RELATED WORK}
\label{sec:appRelatedWork}

\paragraph{Abstaining systems.}

The idea to allow ML models to abstain (a.k.a. to reject) from predicting dates back to the 1970s, with the seminal work by \citet{DBLP:journals/tit/Chow70}.

In the literature, two kinds of rejections have been considered: novelty rejection and ambiguity rejection.
The former provides methods that abstain when the instances are far away from the training data distribution; the latter abstains on instances close to the decision boundary of the classifier~\citep{hendrickx_machine_2024}. 

Novelty rejection is highly sought if a shift between the training and the test set distributions can occur \citep{DBLP:conf/sdm/PlasMVD23}.
Multiple approaches have been proposed for building novelty rejectors. For instance, one can estimate the marginal density of the training distribution and reject an instance if its probability is below a certain threshold \citep{DBLP:conf/icml/NalisnickMTGL19,DBLP:conf/ijcai/WangY20a}.
Another approach relies on a one-class classification model that predicts as novel the instances falling out of the region learnt from the training set~\citep{coenen2020probability}.
Alternatively, one can provide a score representing the novelty of an instance and abstain when such a score is above a certain level~\citep{DBLP:conf/iclr/LiangLS18,kuhne2021securing,perini2023unsupervised,van2023novel}.

Regarding ambiguity rejection, two main approaches have emerged in the literature: Learning to Reject (LtR) \citep{DBLP:journals/tit/Chow70} and Selective Prediction (SP) \citep{DBLP:journals/jmlr/El-YanivW10}.

LtR - based on the original work by ~\citet{DBLP:journals/tit/Chow70} - aims at learning a pair (classifier, rejector) such that the rejector determines when the classifier predicts, limiting the predictions to the region where the classifier is likely correct \citep{cortes2023theory}.
The LtR methods learn the trade-off between abstention and prediction through a parameter $a$, representing the cost of rejection~\citep{Herbei06,DBLP:conf/nips/CortesDM16,DBLP:journals/prl/Tortorella05,DBLP:conf/isbi/CondessaBCOK13}.

On the other hand, SP methods rely on confidence functions, identifying instances where the classifier is more prone to make mistakes \citep{DBLP:journals/jmlr/El-YanivW10}. Confidence values allow one to trade off coverage $c$ for selective risk, namely the risk over those instances for which a prediction is provided.
Such a trade-off can be used to frame the learning problem in two ways: either we maximize coverage given a minimal target risk we want to ensure (bounded-improvement problem), or we minimize the selective risk given a target coverage (bounded-abstention problem) \citep{DBLP:journals/jmlr/FrancPV23}.

From a practical perspective, both model-agnostic methods (e.g.,~\citep{denis2020consistency,PugnanaRuggieri2023a,PugnanaRuggieri2023b}) and model-specific ones (e.g.,~using Deep Neural Network architectures \citep{DBLP:conf/nips/GeifmanE17,DBLP:conf/icml/GeifmanE19,DBLP:conf/nips/CorbiereTBCP19,DBLP:conf/nips/Huang0020,feng2023towards}) have been proposed to solve the selective prediction task. For an extensive characterization of deep-neural-network (DNN)-based approaches, we refer to \cite{Benchmark2024}, where the authors empirically compare existing DNN-based approaches, categorizing them depending on how they abstain.

An in-depth theoretical analysis for both LtR and SP can be found in \cite{DBLP:journals/jmlr/FrancPV23}, where the authors show that both frameworks share similar optimal strategies, and a recent survey covering abstaining systems can be found in \cite{hendrickx_machine_2024}, where the authors provide an overall taxonomy of existing approaches.

\citet{DBLP:conf/nips/ChoeGR23} consider the evaluation of abstaining classifiers, in a scenario where the predictions are missing or unaccessible for the rejected instances. They define the \textit{counterfactual score} as the expected accuracy of the classifier had it not been given the option to abstain. A double-ML approach~\citep{chernozhukov2018DoubleDebiasedMachine} is used to estimate the counterfactual score.
In order to exploit results of inference 
under missing data, the paper assumes a stochastic abstention policy, which is impractical/unethical in the context of deferring systems: instances are not deferred to a human expert at random.

\paragraph{Deferring systems.} Learning to Defer (LtD) - as framed by~\citet{DBLP:conf/nips/MadrasPZ18} - is a generalization of LtR, where rather than incurring a rejection cost, the system can defer instances to human expert(s).
Compared with LtR and SP, one of the main differences is that the expert's predictions might be wrong under the LtD framework. 

From a theoretical perspective, \citet{DBLP:conf/aaai/DeKGG20} show that the problem of learning to defer when choosing a ridge regression as a base predictor is NP-hard. By reformulating the problem using submodular functions, they devise a greedy algorithm with some theoretical guarantees. Similar results also hold for the classification setting when considering margin-based classifiers, as shown in \citep{DBLP:conf/aaai/DeOZR21}.
\citet{DBLP:conf/nips/OkatiDG21} formally characterises the scenarios where a predictive model can take advantage of including humans in the loop. They show that standard ML models trained to predict over all the instances may be suboptimal when it comes to LtD, proposing a deterministic threshold rule to determine when the ML model or the human has to predict.

\noindent
Due to the difficulties in directly optimizing (\ref{eq:loss}), several approaches provide surrogate losses to learn predictors that can defer to experts:
\citet{DBLP:conf/icml/MozannarS20} propose a method that jointly learns both the rejector and the ML predictor with some generalization bounds;
\citet{DBLP:conf/icml/VermaN22} and \citet{DBLP:conf/nips/CaoM0W023} extend the work of \cite{DBLP:conf/icml/MozannarS20} by providing consistent surrogate loss with better calibration properties;
\citet{DBLP:conf/aistats/MozannarLWSDS23} provide a Mixed Integer Linear Programming formulation to solve the problem in the linear setting and a novel surrogate loss that is realizable-consistent and also Bayes consistent, as shown in \citet{DBLP:conf/nips/MaoM024b};
\citet{DBLP:conf/aistats/LiuCZF024} propose new surrogate losses that are not prone to underfitting; \citet{DBLP:conf/icml/WeiC024} consider another surrogate loss that takes into account the dependency between the ML model predictions and the human ones.

A few works also consider extensions to staged learning, where the ML model is already given and not jointly trained, e.g.,~\citet{DBLP:conf/icml/CharusaieMSS22} and \citet{DBLP:conf/nips/MaoMM023}.
To conclude, extensions to multi-experts can be found in \citet{DBLP:conf/aistats/VermaBN23} and \citet{DBLP:conf/nips/CaoM0W023}.

\paragraph{Regression discontinuity.}

The RD design first appeared in \cite{thistlethwaite1960RegressiondiscontinuityAnalysisAlternative} to study the motivational effect of public recognition on the likelihood of obtaining a scholarship. Forty years later, \cite{hahn2001IdentificationEstimationTreatment} proposed a thorough formalization of this methodology, which shows the identification of the average treatment effect on the treated via smoothness of the potential outcomes. An alternative framework for identification has been presented in \cite{lee2008RandomizedExperimentsNonrandom} and \cite{cattaneo2015RandomizationInferenceRegression}, where the authors carefully propose conditions under which, at least near the cutoff, the RD design can be interpreted as an RCT. Early reviews are \cite{imbens1994IdentificationEstimationLocal} and \cite{lee2010RegressionDiscontinuityDesigns}, whereas a more recent one contrasting the two approaches mentioned above is \cite{cattaneo2022RegressionDiscontinuityDesigns}. 

For estimation purposes, the local nature of $\tau_{\mathtt{RD}}$ motivates the use of local non-parametric polynomial kernel regressions estimators from the left and from the right of the cutoff (for a review see \cite{fan1996LocalPolynomialModelling}). Particular attention in the literature has been devoted to how to optimally choose the smoothing bandwidth used in local polynomial estimation (see \cite{imbens1994IdentificationEstimationLocal,calonico2014RobustNonparametricConfidence, kolesar2018InferenceRegressionDiscontinuity}) and how to correct for the smoothing bias and conduct valid inference accordingly \citep{calonico2018EffectBiasEstimation, calonico2022CoverageErrorOptimal}. For a practical introduction to RD and more information on how to choose the other tuning parameters (kernel shape and degree of the polynomial), see \cite{cattaneo2019PracticalIntroductionRegression}. 

\setcounter{proposition}{0}
\section{PROOFS}
\label{sec:appendixProof}
We report here the proofs for propositions from Section \ref{sec:methodology}.
For the reader's convenience, we restate the claims.
\begin{proposition}
    Let Scenario \ref{scenario1} hold. Then, for each $i\in[1,n]$ such that 
    $G_i=1$:
    \[\tau_{i} = T_i(1)-T_i(0) =\indicator\{h(\mathbf{X}_i)=Y_i\} - \indicator\{f(\mathbf{X}_i)=Y_i\}.\]
\end{proposition}
\begin{proof}
Identification is immediate because both potential outcomes are observed for unit $i\in[1,n]$ that has also been evaluated by a human ($G_i=1$). In particular,  $T_i(0) = \indicator\{f(\mathbf{X}_i)=Y_i\}$ is observable because we are under Scenario~\ref{scenario1} and $T_i(1)=\indicator\{h(\mathbf{X}_i)=Y_i\})$ since $G_i=1$.
\end{proof}

\begin{proposition}
Let Scenario \ref{scenario1} hold. Then:
\[ \tatd = \mathbb{E}[T(1)-T(0)|G=1]= \mathbb{E}\left[\indicator\{h(\mathbf{X})=Y\}\mid G=1\right] - \mathbb{E}\left[ \indicator\{f(\mathbf{X})=Y\}\mid G=1\right].\]
\end{proposition}
\begin{proof}
    We have that:
    \begin{align*}
    \tatd &= 
    \mathbb{E}\left[T(1) \mid G=1\right] - \mathbb{E}\left[T(0)\mid G=1\right] \\  
        &=\mathbb{E}\left[\indicator\{h(\mathbf{X})=Y\}\mid G=1\right] - \mathbb{E}\left[\indicator\{f(\mathbf{X})=Y\}\mid G=1\right],
    \end{align*}
    where the last two quantities are observable under Scenario~\ref{scenario1}.
\end{proof}

\begin{proposition}
Let Scenario 1 hold. Then:
\\
$(i)$ as $n\to\infty$,
$\htatd \cprob \mathbb{E}[T(1) -T(0)\mid G=1]=\tatd $\\
$(ii)$ for each $n >0$,
$\hat{\tau}_{\Delta} = \frac{n_1}{n}\htatd.$
\end{proposition}

\begin{proof}
For $(i)$, let $g_i:\indicator\{g(\mathbf{x}_i)=1\}$ and so $n_1:=\sum_{i=1}^ng_i$. Then:
    \begin{align*}
        \htatd&=\frac{1}{n_1}\left(\sum_{i:g(\mathbf{x}_i)=1}\left(\indicator\{h(\mathbf{x}_i)=y_i\}-\indicator\{f(\mathbf{x}_i)=y_i\}\right)\right)\\
        &=\left(\frac{n_1}{n}\right)^{-1} \frac{1}{n}\left(\sum_{i=1}^ng_i\left(\indicator\{h(\mathbf{x}_i)=y_i\}-\indicator\{f(\mathbf{x}_i)=y_i\}\right)\right).
    \end{align*}
    The weak law of large numbers immediately gives us:
    \[\frac{n_1}{n} \cprob \mathbb{P}[G=1], \qquad \frac{1}{n}\left(\sum_{i=1}^ng_i\left(\indicator\{h(\mathbf{x}_i)=y_i\}-\indicator\{f(\mathbf{x}_i)=y_i\}\right)\right)\cprob \mathbb{E}[G(T(1)-T(0))].\]
    Then, using the fact that $\mathbb{E}[G(T(1)-T(0))] =\mathbb{P}[G=1] \mathbb{E}[T(1)-T(0)\mid G=1],$ we can conclude that $\htatd\cprob \tatd$.\\
For $(ii)$, by definition, we have that:
    \[\Acc_\vartheta := \frac{1}{n}\sum_{i=1}^n\indicator\{\vartheta(\mathbf{x}_i)=y_i\} = \frac{1}{n}\left[\sum_{i:g(\mathbf{x}_i)=0}\indicator\{f(\mathbf{x}_i)=y_i\} + \sum_{i:g(\mathbf{x}_i)=1}\indicator\{h(\mathbf{x}_i)=y_i\}\right],\]
    and
    \[\Acc_f := \frac{1}{n}\sum_{i=1}^n\indicator\{f(\mathbf{x}_i)=y_i\} = \frac{1}{n}\left[\sum_{i:g(\mathbf{x}_i)=0}\indicator\{f(\mathbf{x}_i)=y_i\} + \sum_{i:g(\mathbf{x}_i)=1}\indicator\{f(\mathbf{x}_i)=y_i\}\right].\]
    Therefore, when we consider the difference between the two accuracies we have that:
    \begin{align*}
            \hat{\tau}_{\Delta} :&= \Acc_\vartheta- \Acc_f \\
            &= \frac{1}{n}\left(\sum_{i:g(\mathbf{x}_i)=1}\left(\indicator\{h(\mathbf{x}_i)=y_i\}-\indicator\{f(\mathbf{x}_i)=y_i\}\right)\right)  \\
            &=\frac{n_1}{n}\frac{1}{n_1}\left(\sum_{i:g(\mathbf{x}_i)=1}\left(\indicator\{h(\mathbf{x}_i)=y_i\}-\indicator\{f(\mathbf{x}_i)=y_i\}\right)\right) \\
            &= \frac{n_1}{n}\htatd.
    \end{align*}
\end{proof}

\begin{proposition}
    Let Scenario \ref{scenario2} hold and let Assumption \ref{ass:continuity} be satisfied for the deferring system. Then:
    \[\lim_{k\to \ok_c^+}\mathbb{E}[T\mid K=k]-\lim_{k\to \ok_c^-}\mathbb{E}[T\mid K=k] = \trdd,\]
    where
    $\trdd := \mathbb{E}[T(1) - T(0) \mid K = \ok_c]$.
\end{proposition}
\begin{proof}
This is an instance of Theorem~\ref{thm:rdidentifiability}.
\end{proof}

\section{EXTENSIONS AND LIMITATIONS}
\label{sec:appendixCaveats}
Here are a few caveats regarding the proposed framework.
\paragraph{Further discussion on Proposition 3 and identifiability}
\label{sec:furtherDis}

If we have access to accuracy estimates for $\Acc_{\vartheta}$ and $\Acc_f$, we can just re-weight $\hat{\tau}_{\Delta}$ by $n/n_1$ to obtain a consistent estimator of $\tatd$. This depicts a setting in between Scenario 1 and Scenario 2. Indeed, having access to $n_1$, $n$, and accuracy estimates for the ML model and the deferring system allows us to obtain an estimate for $\tatd$ without requiring the full knowledge of individual ML model predictions for the deferred instances. We further note that this reweighting is nothing else than a Horvitz-Thompson type of reweighting \citep{horvitz1952generalization} where the observations are weighted by the inverse of the sampling probability. Here the ``sampling probability" denotes the probability that an instance is assigned to a human.

\paragraph{Extending to global effects} Following \citet{DBLP:conf/aistats/KallusZ18} and  \citet{DBLP:conf/nips/KallusU20}, assuming some regularity conditions, one could try to approximate the deferral strategy (read a deterministic policy) by some other stochastic policies. Such an approach seems promising in helping us extend the estimation to global effects when we cannot access both human and ML predictions. However, we stress that we do not need to approximate the deterministic deferral strategy if Scenario 1 holds, and the estimation would be straightforward for the causal effects on the deferred. Similarly, if the regularity assumptions do not hold, one can always use the Scenario 2 setting to estimate a local effect.

\paragraph{Is Assumption \ref{ass:continuity} met?}
Assumption \ref{ass:continuity} is required to be able to identify the causal effect under Scenario \ref{scenario2}. Formally, such an assumption requires continuity of the expected predictive accuracy at coverage level $\ok_c$ for \textit{both} the ML model and the human. In practical terms, there is no reason to believe that the accuracy of the ML model would abruptly change in a neighborhood of $\ok_c$.  However, continuity of $\mathbb{E}[T(1)\mid K=k]$ around $\ok_c$ could be falsified, e.g.,~if the expert would put extra effort into predicting deferred instances compared to the non-deferred ones. Concerning this aspect, \citet{DBLP:conf/aaai/BondiKSCBCPD22} show that the role of communicating the deferral status can indeed impact human performance. On the one hand, they observe improvements in human accuracy for those instances where the ML model is correct if the deferral choice is communicated to the human predictors. On the other hand, they do not observe a statistically significant effect in those instances in which the ML model makes mistakes. However, the deferring system considered by \citet{DBLP:conf/aaai/BondiKSCBCPD22} does not take into account where the human performs better than the ML model. Conversely, \citet{DBLP:conf/iui/HemmerWSVVS23} show that if the ML model explicitly considers where the human outperforms the model, the results improve even without communicating the deferral decision.
Therefore, we advise considering behavioural aspects when developing and evaluating a deferring system.

\paragraph{Is Assumption \ref{ass:continuity} testable?} 
\label{sec:assContinuityChecks}
Unfortunately, Assumption \ref{ass:continuity} is not directly testable because ML predictions ($T(0)$) and human predictions ($T(1)$) are not available when $K\geq \ok_c$ or $K<\ok_c$, respectively. However, it is possible (and suggested) to test the implications of Assumption \ref{ass:continuity}. In what follows, we briefly describe three different falsification tests. We refer the reader interested in a more detailed review and guide on these (and other) tests to \citep{lee2010RegressionDiscontinuityDesigns} and \cite[Chapter 5]{cattaneo2019PracticalIntroductionRegression}. 

\textit{Non-manipulation of the running variable}. One instance in which Assumption \ref{ass:continuity} does not hold is when units know the rule with which the running variable $K_i$ is computed and/or can manipulate its value. Accordingly, a feasible falsification test involves checking whether the empirical probability distribution of the running variable is smooth (i.e.,~it does not jump) at the cutoff. This test can be formally conducted by estimating the density of the running variable with histograms or kernel density estimators \citep{cattaneo2019PracticalIntroductionRegression,cattaneo2020SimpleLocalPolynomial}. For a practical example, see Section \ref{sec:appendixValidate} and Figures \ref{fig:synth estimated density best}-\ref{fig:xray-airspace estimated density best}.

\textit{No effect on predetermined features and placebo outcomes}.
In a similar spirit to what we described above, another falsification test involves comparing treated and control units near the cutoff to see if they share similar observable traits: if units can't manipulate their score, there should not be systematic differences between units close to the cutoff, aside from their treatment status. As such, units just above and below the cutoff should resemble each other in all aspects unaffected by the treatment. These aspects (or features) can be \textit{predetermined features}, variables realizing before treatment assignment, or \textit{placebo outcomes}, variables that should not have been influenced by the treatment. This test can be conducted by estimating an RD where the outcome variable is either a predetermined feature or a placebo outcome and checking that the null hypothesis of no effect is not rejected. An application of this falsification test can be found in Section \ref{sec: placebo tests} and in the fourth row of Figure \ref{fig:placebo-test}.

\textit{Placebo cutoffs}. Another type of falsification test involves checking if statistically significant treatment effects can be estimated using artificial cutoff values. We stress that the evidence of no effect --hence smoothness of the expected potential outcomes-- away from the cutoff is neither sufficient nor necessary for Assumption \ref{ass:continuity} to hold, but the presence of discontinuities in other places might discredit such an assumption. To conduct such a test, one has to estimate the RD using a cutoff value that is different from the original one. Section \ref{sec: placebo tests} and the second and third rows of Figure \ref{fig:placebo-test} showcase this falsification test in our empirical applications.

\paragraph{Optimal coverage} We again stress that RD designs are local in nature. Hence, under Scenario \ref{scenario2}, without imposing additional strong parametric assumptions on the shape of $\mathbb{E}[T(d)|K=k],d\in\{0,1\}$, we cannot recover the average treatment effect for coverage levels other than $\ok_{c}$. However, suppose that we have access to different batches of data, where similar human experts and the same ML model take turns in predicting the ground truth $Y$, but the reject-score threshold was let vary in a countable (ordered) set $\overline{\mathcal{K}}\subseteq \mathcal{K}$. For instance, we can be in the presence of multiple human moderators that receive different amounts of content to moderate (i.e. $\ok$ is changing). Then, we can leverage results in the literature of RD designs with multiple non-cumulative cutoffs \citep{cattaneo2016InterpretingRegressionDiscontinuity,cattaneo2021ExtrapolatingTreatmentEffects} and identify 
\[\tau_{\mathtt{RD}}(k) := \mathbb{E}[T(1)-T(0)\mid K = k], \qquad k \in [\ok_{\mathtt{lb}},\ok_{\mathtt{ub}}],\]
where $\ok_{\mathtt{lb}}:=\min \overline{\mathcal{K}}$ and $\ok_{\mathtt{ub}}:=\max \overline{\mathcal{K}}$. More precisely, identification of $\tau_{\mathtt{ATD}}(k)$ requires continuity of $\mathbb{E}[T(d)\mid K=k],d\in\{0,1\}$ in $k$ for $k\in[\ok_{\mathtt{lb}},\ok_{\mathtt{ub}}]$ and a similar shape of $\mathbb{E}[T(0)\mid K_i=k]$ across datasets. An exercise in this spirit might be useful to choose the optimal level of coverage $\ok^\star:=\argmax_{k\in\overline{\mathcal{K}}}\tau_{\mathtt{ATD}}(k)$, where the optimality is defined in the sense of getting the largest increase in predictive accuracy out of having human experts guessing instead of the model. 

\paragraph{Should we increase or decrease the coverage?} Despite being local by construction, the RD can be used to learn about the gradient of the treatment effect at the cutoff. Indeed, \citet{dong2015IdentifyingEffectChanging} show that, under regularity conditions on how $\mathbb{E}[T(d)\mid K=k],d\in\{0,1\}$ changes around $\ok_c$, the RD setting can be used to learn how $\trdd$  would change if the reject-score threshold $\ok_c$ were marginally changed.
Therefore, this suggests that when the deferring system accuracy is maximal, $\trdd$ should be close to zero. The intuition is that we should be indifferent between predicting with the ML model or deferring to the human expert at the optimal value.

\paragraph{Uncertainty in the ML model estimation} The outcome variable for unit $i$ in the LtD framework is $T_i=\indicator\{\vartheta(\mathbf{X}_i)=Y_i\}$ and the running variable is the reject score $K_i=k(\mathbf{X}_i)$ (see Table \ref{tab: RD to LtD mapping}). However, in practice, we compute the outcome and the reject score via an estimated version of the model, i.e. $\hat{f}(\cdot).$ We explicitly do not consider this source of uncertainty because access is usually limited to an estimated final version of the model without possibly fitting it again (e.g.,~large language models \citep{DBLP:conf/nips/BrownMRSKDNSSAA20}). 
For this reason, with a slight abuse of notation, we always write $f(\cdot)$, $T_i$, and $K_i$ when we should write $\hat{f}(\cdot)$, $\hat{T}_i,$ and $\hat{K}_i$, respectively. If one is willing also to capture this uncertainty and model re-estimation is viable, off-the-shelf non-parametric bootstrap procedures are available. Moreover, \cite{dong2023WhenCanWe} show that under the condition that the noisy score correctly assigns samples to treatment and control groups, $\tau_{\mathtt{RD}}$ can be interpreted as the treatment effect when the \textit{noisy} score equals the cutoff. We argue that this latter measure is still the one of interest, particularly so when the model is taken as given because the reject score can only be computed using the estimated ML model.

\paragraph{Generalize the RD estimand}
As we already stressed several times, the RD estimand $\trdd$ is local in nature as it refers to instances exactly at the coverage threshold $\overline{\kappa}_c$. A natural question to ask is under which assumptions it would be possible to generalize $\trdd$ to other populations of instances. First, a treatment effect homogeneity assumption would suffice to claim that $\trdd=\tau_{\mathtt{ATE}}$, that is, the RD estimand coincides with the ATE for the whole population of interest. However, this is an extreme assumption. Several other works in the causal inference literature have recently proposed ways to generalize $\trdd$ to other populations without imposing treatment homogeneity. For example, \cite{angrist2015WannaGetAway} invoke a conditional independence assumption within a neighborhood of the cutoff to generalize $\trdd$ within such a neighborhood. \cite{dong2015IdentifyingEffectChanging} show how to identify the derivative of the treatment effect $\trdd$ at the cutoff under a local policy invariance assumption. \cite{cattaneo2021ExtrapolatingTreatmentEffects} leverage multiple cutoffs and a parallel-trend assumption to extrapolate the average treatment effect to score values falling between (at least) two cutoffs. The latter work is probably the most compelling to apply under Scenario \ref{scenario2}. For example, suppose we have access to two different samples where deferral occurred according to two different reject score thresholds $\overline{\kappa}_{c_1}$ and $\overline{\kappa}_{c_2}$, respectively. Then, using the framework proposed by \cite{cattaneo2021ExtrapolatingTreatmentEffects}, it would be possible to identify the ATD for all units with a value of the reject score between $\overline{\kappa}_{c_1}$ and $\overline{\kappa}_{c_2}$.

\section{EXTENDED EXPERIMENTAL EVALUATION}
\
\subsection{Additional details}
\begin{table}[t]
    \centering
    \caption{Datasets and baselines details (epochs - \texttt{ep.}, learning rate - \texttt{lr}, optimizer - \texttt{op.}).}
\resizebox{\textwidth}{!}{
\begin{tabular}{ccccccc}
\textbf{DATASET} & \textbf{n} & \boldmath{}\textbf{$|\mathcal{Y}|$}\unboldmath{} & \textbf{HUMAN} & \textbf{MODEL} & \textbf{PRE-TRAINED} 
& \textbf{HYPER-PARAMETERS} \\
\midrule
\texttt{synth} & 50k   & 2     & synthetic & linear &    no   
& $\texttt{ep.}=50;\texttt{lr}=1e-2;\-\texttt{op.}=\texttt{Adam}$ \\
\texttt{cifar10h} & 10k   & 10    & separate annotator & WideResNet  & yes   
& $\texttt{ep.}=150;\texttt{lr}=1e-3;\-\texttt{op.}=\texttt{AdamW}$ \\
\texttt{galaxyzoo} & 10k   & 2     & random annotator & ResNet50 & yes   
& $\texttt{ep.}=50;\texttt{lr}=1e-3;\-\texttt{op.}=\texttt{Adam}$ \\
\texttt{hatespeech} & 25k   & 3     & random annotator & FNN on SBERT embeddings  & yes   
& $\texttt{ep.}=100;\texttt{lr}=1e-2;\-\texttt{op.}=\texttt{Adam}$ \\
\texttt{xray-airspace} & 4.4k    & 2     & random annotator & DenseNet121 & yes   
& $\texttt{ep.}=3;\texttt{lr}=1e-3;\-\texttt{op.}=\texttt{AdamW}$ \\
\end{tabular}%
}
\label{tab:exp_details}
\end{table}

\subsubsection{Data}
We use the data from \citet{DBLP:conf/aistats/MozannarLWSDS23} and \citet{DBLP:conf/nips/OkatiDG21}.

Concerning the synthetic dataset \texttt{synth}, we generate the data using the method described in \citet{DBLP:conf/aistats/MozannarLWSDS23}\footnote{\label{MozannarRepo}See the GitHub repository  \url{https://github.com/clinicalml/human_ai_deferral}}: given a parameter $d$, $\mathbf{x}\in\mathbb{R}^d$ is sampled from a mixture of $d$ equally weighted Gaussians, each one with uniformly random mean and variance. 
To obtain the target variable~$Y$, the procedure generates two random half-spaces, one referring to the optimal policy function $g^*:\mathcal{X}\to\{0,1\}$ and one representing an optimal ML model $f^*:\mathcal{X}\to\mathcal{Y}$. The fraction of instances for which $g^*(\mathbf{x})=0$ is randomly chosen to be between $.20$ and $.80$.
For all those instances on the side where $g^*(\mathbf{x})=0$, the target variable $Y$ is changed to be consistent with the optimal ML model $f^*(\mathbf{x})$ with probability $1-p_{ML}$ and otherwise uniform. 
Conversely, when $g^*(\mathbf{X}) = 1$, the labels are uniformly sampled. The human expert $h(\mathbf{x})$ is then set to make mistakes at a rate of $p_{h0}$ when $g^*(\mathbf{x})=0$ and at a rate of $p_{h1}$ when $g^*(\mathbf{x})=1$. In our experiments, we set  $d=10$, $p_{h0} = .10$, $p_{h1} = .10$, $p_{ML}=.40$, as done in the unrealizable setting by \citet{DBLP:conf/aistats/MozannarLWSDS23}.

Regarding real data, in \texttt{cifar10h} \citep{DBLP:journals/nature/BattledayPG2020}, the task is to annotate images belonging to 10 different categories. Here, the human prediction is provided by a separate human annotator.

In \texttt{galaxyzoo} \citep{galaxy-zoo-the-galaxy-challenge}, the main task is identifying whether the image contains a non-smooth galaxy. Thus, $Y=0$ if the image contains a smooth galaxy and $Y=1$ otherwise. Since we have 30 annotators for each image, we consider the majority of the annotators as the target variable $Y$, while the human expert prediction $h(\mathbf{X})$ is sampled randomly from the 30 annotators. 

For \texttt{hatespeech} \citep{DBLP:conf/icwsm/DavidsonWMW17}, the goal is to detect whether the text contains offensive or hate-speech language. The human predictor is sampled randomly as in \citet{DBLP:conf/aistats/MozannarLWSDS23}.

Finally, \texttt{xray-airspace} \citep{DBLP:conf/cvpr/WangPLLBS17,majkowska2020chest} contains both chest X-rays with human predictions and chest X-rays without human predictions. For each image, the target variable $Y$ encodes the presence of an airspace opacity. The human predictions are randomly sampled from multiple experts, as done by \citet{DBLP:conf/aistats/MozannarLWSDS23}.

\label{sec:appendixData}
\subsubsection{Baselines}
We include in our experiments seven baselines for which the code was publicly available. For all the baselines but \ASM{}, we consider the implementation provided by \citet{DBLP:conf/aistats/MozannarLWSDS23}\footref{MozannarRepo}. 
We use ~\citet{DBLP:conf/nips/CaoM0W023}'s code for \ASM{}, as provided in their supplementary material.

\label{sec:appendixBaselines}
\noindent
\textit{Selective Prediction} (\SP{}): \citet{DBLP:conf/nips/GeifmanE17} present a neural network classifier with a reject option. The reject score is defined considering the maximum of the final softmax values, i.e.,~$k(\mathbf{x})=\max_y s_y(\mathbf{x})$, where $s_y(\mathbf{x})$ is the final layer softmax value for class $y$.
We stress that \SP{} does not take into account the human expert's ability but determines deferral only based on those cases where the ML model is uncertain.

\noindent
\textit{Compare Confidence} (\CC{}): \citet{Raghu2019} extend \SP{} by learning (independently) another model - called the expert model - that uses as a target variable whether the human expert is correct. Then, deferral is determined by comparing the reject score of the classifier and the expert model.  

\noindent
\textit{Differentiable Triage} (\DT{}): \citet{DBLP:conf/nips/OkatiDG21} consider a two-stage approach, where at each epoch \textit{(i)} the classifier is trained only on those points where the classifier loss is lower than the human loss, \textit{(ii)} another ML model - called the rejector - is fitted to predict who has a lower loss between the classifier and the human. At the end of the training procedure, deferral is decided based on the estimated probability of the human expert having a lower loss than the classifier.

\noindent
\textit{Cross-Entropy Surrogate} (\LCE): \citet{DBLP:conf/icml/MozannarS20} propose a consistent surrogate loss of (\ref{eq:loss}), when $l$ is the 0-1 loss. The surrogate loss is then used to train the deferring system, which employs a neural network with an additional head to represent deferral;

\noindent
\textit{One Vs All} (\OVA{}): \citet{DBLP:conf/icml/VermaN22}
propose a different consistent surrogate loss, which improves the final calibration of the deferring system;

\noindent
\textit{Realizable Surrogate} (\RS{}): \citet{DBLP:conf/aistats/MozannarLWSDS23} extend the approaches based on surrogate losses by considering a consistent and realizable-consistent surrogate of (\ref{eq:loss}), when $l_M$ and $l_H$ are the 0-1 loss. As for \LCE{}, also \RS{} considers a neural network with an additional head representing deferral.

\noindent
\textit{Asymmetric SoftMax} (\ASM{}): \citet{DBLP:conf/nips/CaoM0W023} extend both \LCE{} and \OVA{} by providing a surrogate loss that ensures a better calibration of the reject score.

\subsubsection{Hyperparameters}\label{sec:appHyper}
For \texttt{synth}, we considered a simple linear feedforward architecture. For each baseline, we trained the model for $50$ epochs with $\texttt{lr}=1e-2$ and \texttt{Adam}~\citep{DBLP:journals/corr/KingmaB14} as the optimizer. Batch size was set to $1,024$.

For real data, all the methods were trained following the settings in either \citet{DBLP:conf/aistats/MozannarLWSDS23} (\texttt{cifar10h}, \texttt{hatespeech}, \texttt{xray-airspace}) or \citet{DBLP:conf/nips/OkatiDG21} (\texttt{galaxyzoo}).

In particular, for \texttt{cifar10h}, we trained a base WideResNet \citep{DBLP:conf/bmvc/ZagoruykoK16} on the original \texttt{cifar10} dataset for 200 epochs using cross-entropy loss, learning rate equals to $.001$ and \texttt{AdamW}~\citep{DBLP:conf/iclr/LoshchilovH19} as an optimizer. For each baseline, we fine-tuned the base WideResnet on \texttt{cifar10h} for 150 epochs, using a learning rate of $.001$ and \texttt{AdamW} as an optimizer.

For \texttt{hatespeech}, we considered pre-trained embeddings of SBERT and we fine-tuned a feed-forward neural network \citep{DBLP:conf/emnlp/ReimersG19} for 100 epochs, setting the learning rate to $.01$ and Adam as optimizer.

For \texttt{xray-airspace}, we first fine-tuned a pre-trained DenseNet121 \citep{DBLP:conf/cvpr/HuangLMW17} for 10 epochs on the x-rays that do not contain human predictions, setting $\texttt{lr}=\num{1e-4}$ and \texttt{AdamW} as the optimizer. For each baseline, we further fine-tuned the obtained model, training it for $3$ epochs on \texttt{xray-airspace} with a learning rate equal to $\num{1e-3}$ and \texttt{AdamW} as the optimizer.

Finally, for \texttt{galaxyzoo}, we consider a pre-trained ResNet50 \citep{DBLP:conf/cvpr/HeZRS16} and train each baseline for $50$ epochs, using \texttt{Adam} as the optimizer and a learning rate of $\num{1e-3}$.

The batch size was set to $128$ for all the real datasets.

\subsubsection{Hardware and carbon footprint}
\label{sec:appendixHardware}
Regarding computational resources, we split the workload over two machines: 
$(i)$ a 96 cores machine with Intel(R) Xeon(R) Gold 6342 CPU @ 2.80GHz and two NVIDIA RTX A6000, OS Ubuntu 20.04.4; 
$(ii)$ a 224 cores machine with Intel(R) Xeon(R) Platinum 8480+ CPU and eight NVIDIA A100-SXM4-80GB, OS Ubuntu 22.04.4 LTS.

We track all our runs using the \texttt{Python} package \texttt{codecarbon} \citep{benoit_courty_2024_11171501}. This allows us to consider the total time required by all our experimentation (including failed and repeated experiments) and its environmental costs. 
Overall, the cumulated time of all our runs amounts to $\approx5$ days.
This translates into an overall $\text{CO}_2$ consumption of roughly $\approx25.2$~Kg~Eq.$\text{CO}_2$, which equals a car drive of $\approx60$ miles.

\subsection{Detailed results for \textbf{Q1}, \textbf{Q2} and \textbf{Q3}}
\label{sec:appendixExtraRes}

Tables \ref{tab: synth-results}-\ref{tab: xray-airspace-results} report the detailed results of experiments on the synthetic and the real datasets. Tables include $\htatd$ (the Scenario \ref{scenario1} main estimate), $\htrdd$ (the Scenario \ref{scenario2} main estimate), and the deferring system accuracy at various target coverages for all the seven baselines. For $\htatd$ and $\htrdd$, we show in parentheses the associated p-value when testing the significance of the causal effect being different from zero. Significant (after Bonferroni correction of $\alpha=0.05$) p-values are shown in blue. 

\afterpage{
\begin{table}[t]
\caption{\texttt{synth} results, with the statistically significant ones at $\alpha=0.05$ in blue.}
\label{tab: synth-results}
\resizebox{\textwidth}{!}{
\begin{tabular}{c|c|ccccccc}
 & $c$ & \ASM{} & \CC{} & \DT{} & \LCE{} & \OVA{} & \RS{} & \SP{} \\
\midrule
\multirow{10}{*}{\rotatebox[origin=c]{90}{$\htatd$}} & .00 & $.083~(\textcolor{blue}{\num{8.20e-35}})$ & $.135~(\textcolor{blue}{\num{8.55e-89}})$ & $.107~(\textcolor{blue}{\num{2.55e-56}})$ & $\mathbf{.143~(\textcolor{blue}{\num{1.04e-100}})}$ & $.138~(\textcolor{blue}{\num{9.57e-93}})$ & $.102~(\textcolor{blue}{\num{2.31e-51}})$ & $.132~(\textcolor{blue}{\num{4.40e-85}})$ \\
 & .10 & $.133~(\textcolor{blue}{\num{2.12e-78}})$ & $.194~(\textcolor{blue}{\num{1.42e-165}})$ & $.153~(\textcolor{blue}{\num{2.32e-103}})$ & $\mathbf{.205~(\textcolor{blue}{\num{1.16e-186}})}$ & $.199~(\textcolor{blue}{\num{1.88e-174}})$ & $.157~(\textcolor{blue}{\num{8.86e-110}})$ & $.188~(\textcolor{blue}{\num{1.94e-155}})$ \\
 & .20 & $.199~(\textcolor{blue}{\num{2.88e-157}})$ & $.254~(\textcolor{blue}{\num{6.03e-265}})$ & $.208~(\textcolor{blue}{\num{1.63e-173}})$ & $\mathbf{.261~(\textcolor{blue}{\num{4.03e-285}})}$ & $.259~(\textcolor{blue}{\num{1.49e-276}})$ & $.218~(\textcolor{blue}{\num{3.05e-193}})$ & $.230~(\textcolor{blue}{\num{5.03e-215}})$ \\
 & .30 & $.269~(\textcolor{blue}{\num{1.48e-266}})$ & $.310~(\textcolor{blue}{\num{.00e+00}})$ & $.269~(\textcolor{blue}{\num{2.67e-264}})$ & $.314~(\textcolor{blue}{\num{.00e+00}})$ & $\mathbf{.318~(\textcolor{blue}{\num{.00e+00}})}$ & $.280~(\textcolor{blue}{\num{2.94e-291}})$ & $.259~(\textcolor{blue}{\num{2.08e-245}})$ \\
 & .40 & $.353~(\textcolor{blue}{\num{.00e+00}})$ & $.350~(\textcolor{blue}{\num{.00e+00}})$ & $.332~(\textcolor{blue}{\num{.00e+00}})$ & $.343~(\textcolor{blue}{\num{.00e+00}})$ & $\mathbf{.359~(\textcolor{blue}{\num{.00e+00}})}$ & $.347~(\textcolor{blue}{\num{.00e+00}})$ & $.289~(\textcolor{blue}{\num{5.03e-269}})$ \\
 & .50 & $\mathbf{.406~(\textcolor{blue}{\num{.00e+00}})}$ & $.373~(\textcolor{blue}{\num{.00e+00}})$ & $.371~(\textcolor{blue}{\num{.00e+00}})$ & $.372~(\textcolor{blue}{\num{.00e+00}})$ & $.392~(\textcolor{blue}{\num{.00e+00}})$ & $.389~(\textcolor{blue}{\num{.00e+00}})$ & $.300~(\textcolor{blue}{\num{1.39e-244}})$ \\
 & .60 & $.406~(\textcolor{blue}{\num{.00e+00}})$ & $.387~(\textcolor{blue}{\num{.00e+00}})$ & $.393~(\textcolor{blue}{\num{.00e+00}})$ & $.390~(\textcolor{blue}{\num{.00e+00}})$ & $.409~(\textcolor{blue}{\num{.00e+00}})$ & $\mathbf{.409~(\textcolor{blue}{\num{.00e+00}})}$ & $.326~(\textcolor{blue}{\num{1.07e-238}})$ \\
 & .70 & $.400~(\textcolor{blue}{\num{.00e+00}})$ & $.403~(\textcolor{blue}{\num{.00e+00}})$ & $.399~(\textcolor{blue}{\num{.00e+00}})$ & $.404~(\textcolor{blue}{\num{.00e+00}})$ & $.402~(\textcolor{blue}{\num{.00e+00}})$ & $\mathbf{.406~(\textcolor{blue}{\num{.00e+00}})}$ & $.335~(\textcolor{blue}{\num{9.06e-196}})$ \\
 & .80 & $.396~(\textcolor{blue}{\num{2.91e-212}})$ & $.404~(\textcolor{blue}{\num{4.90e-225}})$ & $.391~(\textcolor{blue}{\num{3.41e-203}})$ & $\mathbf{.416~(\textcolor{blue}{\num{3.31e-244}})}$ & $.414~(\textcolor{blue}{\num{8.57e-234}})$ & $.404~(\textcolor{blue}{\num{2.03e-220}})$ & $.340~(\textcolor{blue}{\num{3.38e-137}})$ \\
 & .90 & $.402~(\textcolor{blue}{\num{3.01e-109}})$ & $.393~(\textcolor{blue}{\num{4.80e-104}})$ & $.370~(\textcolor{blue}{\num{2.87e-96}})$ & $\mathbf{.414~(\textcolor{blue}{\num{7.30e-119}})}$ & $.400~(\textcolor{blue}{\num{6.82e-111}})$ & $.398~(\textcolor{blue}{\num{2.38e-108}})$ & $.345~(\textcolor{blue}{\num{9.38e-73}})$ \\

\midrule
\multirow{9}{*}{\rotatebox[origin=c]{90}{$\htrdd$}} & .10 & $-.302~(\textcolor{blue}{\num{2.54e-7}})$ & $-.293~(\textcolor{blue}{\num{5.63e-11}})$ & $\mathbf{-.244~(\textcolor{blue}{\num{3.59e-6}})}$ & $-.258~(\textcolor{blue}{\num{4.43e-6}})$ & $-.402~(\textcolor{blue}{\num{1.94e-7}})$ & $-.313~(\textcolor{blue}{\num{5.94e-8}})$ & $-.344~(\textcolor{blue}{\num{1.61e-8}})$ \\
 & .20 & $-.356~(\textcolor{blue}{\num{9.60e-16}})$ & $-.231~(\textcolor{blue}{\num{2.25e-8}})$ & $-.216~(\textcolor{blue}{\num{2.89e-7}})$ & $-.240~(\textcolor{blue}{\num{3.01e-8}})$ & $-.221~(\textcolor{blue}{\num{1.29e-6}})$ & $-.342~(\textcolor{blue}{\num{2.10e-10}})$ & $\mathbf{-.098~(\num{7.13e-2})}$ \\
 & .30 & $-.344~(\textcolor{blue}{\num{3.60e-9}})$ & $-.048~(\num{2.65e-1})$ & $-.145~(\num{1.09e-4})$ & $.016~(\num{6.97e-1})$ & $-.036~(\num{4.53e-1})$ & $-.237~(\textcolor{blue}{\num{1.18e-5}})$ & $\mathbf{.086~(\num{4.41e-2})}$ \\
 & .40 & $-.345~(\num{5.24e-2})$ & $.179~(\textcolor{blue}{\num{1.95e-6}})$ & $.028~(\num{4.16e-1})$ & $\mathbf{.185~(\textcolor{blue}{\num{8.04e-7}})}$ & $.178~(\num{1.25e-3})$ & $.006~(\num{8.93e-1})$ & $.129~(\num{5.45e-4})$ \\
 & .50 & $\mathbf{.288~(\num{9.25e-3})}$ & $.282~(\textcolor{blue}{\num{1.37e-13}})$ & $.229~(\textcolor{blue}{\num{1.98e-11}})$ & $.243~(\textcolor{blue}{\num{1.47e-10}})$ & $.205~(\textcolor{blue}{\num{3.36e-5}})$ & $.230~(\textcolor{blue}{\num{1.56e-10}})$ & $.159~(\num{4.26e-4})$ \\
 & .60 & $\mathbf{.485~(\textcolor{blue}{\num{3.07e-26}})}$ & $.333~(\textcolor{blue}{\num{3.12e-20}})$ & $.299~(\textcolor{blue}{\num{5.88e-11}})$ & $.317~(\textcolor{blue}{\num{1.12e-18}})$ & $.415~(\textcolor{blue}{\num{6.81e-29}})$ & $.443~(\textcolor{blue}{\num{6.63e-30}})$ & $.209~(\textcolor{blue}{\num{2.22e-7}})$ \\
 & .70 & $.431~(\textcolor{blue}{\num{7.20e-25}})$ & $.334~(\textcolor{blue}{\num{6.93e-21}})$ & $.395~(\textcolor{blue}{\num{1.30e-22}})$ & $.366~(\textcolor{blue}{\num{8.37e-27}})$ & $\mathbf{.473~(\textcolor{blue}{\num{1.19e-34}})}$ & $.433~(\textcolor{blue}{\num{1.04e-19}})$ & $.333~(\textcolor{blue}{\num{3.88e-11}})$ \\
 & .80 & $\mathbf{.479~(\textcolor{blue}{\num{1.97e-55}})}$ & $.406~(\textcolor{blue}{\num{3.47e-24}})$ & $.427~(\textcolor{blue}{\num{2.49e-22}})$ & $.404~(\textcolor{blue}{\num{4.36e-30}})$ & $.384~(\textcolor{blue}{\num{9.31e-27}})$ & $.353~(\textcolor{blue}{\num{4.14e-17}})$ & $.282~(\num{1.21e-4})$ \\
 & .90 & $.372~(\textcolor{blue}{\num{1.17e-9}})$ & $.275~(\textcolor{blue}{\num{3.48e-5}})$ & $.450~(\textcolor{blue}{\num{9.72e-19}})$ & $.398~(\textcolor{blue}{\num{4.18e-19}})$ & $\mathbf{.468~(\textcolor{blue}{\num{2.65e-17}})}$ & $.432~(\textcolor{blue}{\num{1.08e-11}})$ & $.422~(\textcolor{blue}{\num{9.15e-10}})$ \\

\midrule
\multirow{11}{*}{\rotatebox[origin=c]{90}{Accuracy}} & .00 & $\mathbf{.772}$ & $\mathbf{.772}$ & $\mathbf{.772}$ & $\mathbf{.772}$ & $\mathbf{.772}$ & $\mathbf{.772}$ & $\mathbf{.772}$ \\
 & .10 & $.808$ & $.810$ & $.802$ & $.811$ & $\mathbf{.811}$ & $.811$ & $.808$ \\
 & .20 & $\mathbf{.846}$ & $.838$ & $.831$ & $.837$ & $.839$ & $.843$ & $.825$ \\
 & .30 & $\mathbf{.876}$ & $.853$ & $.852$ & $.848$ & $.855$ & $.865$ & $.823$ \\
 & .40 & $\mathbf{.900}$ & $.848$ & $.863$ & $.837$ & $.850$ & $.877$ & $.815$ \\
 & .50 & $\mathbf{.892}$ & $.826$ & $.851$ & $.816$ & $.832$ & $.867$ & $.790$ \\
 & .60 & $\mathbf{.853}$ & $.794$ & $.824$ & $.787$ & $.801$ & $.837$ & $.770$ \\
 & .70 & $\mathbf{.814}$ & $.759$ & $.787$ & $.750$ & $.757$ & $.797$ & $.742$ \\
 & .80 & $\mathbf{.771}$ & $.720$ & $.746$ & $.716$ & $.719$ & $.752$ & $.709$ \\
 & .90 & $\mathbf{.732}$ & $.677$ & $.705$ & $.671$ & $.676$ & $.712$ & $.677$ \\
 & 1.00 & $\mathbf{.689}$ & $.637$ & $.666$ & $.629$ & $.634$ & $.671$ & $.640$ \\
\end{tabular}
}
\end{table}

\begin{table}[t]
\caption{\texttt{cifar10h} results, with the statistically significant ones at $\alpha=0.05$ in blue.}
\label{tab: cifar10h-results}
\resizebox{\textwidth}{!}{
\begin{tabular}{c|c|ccccccc}
 & $c$ & \ASM{} & \CC{} & \DT{} & \LCE{} & \OVA{} & \RS{} & \SP{} \\
\midrule
\multirow{10}{*}{\rotatebox[origin=c]{90}{$\htatd$}} & .00 & $.028~(\textcolor{blue}{\num{4.53e-5}})$ & $.018~(\num{5.34e-3})$ & $\mathbf{.049~(\textcolor{blue}{\num{1.38e-10}})}$ & $.018~(\num{4.83e-3})$ & $.021~(\num{2.39e-3})$ & $.018~(\num{5.34e-3})$ & $.021~(\num{1.36e-3})$ \\
 & .10 & $.033~(\textcolor{blue}{\num{2.13e-5}})$ & $.027~(\num{1.85e-4})$ & $\mathbf{.038~(\textcolor{blue}{\num{4.61e-7}})}$ & $.021~(\num{2.80e-3})$ & $.024~(\num{1.35e-3})$ & $.021~(\num{4.10e-3})$ & $.024~(\num{1.01e-3})$ \\
 & .20 & $\mathbf{.039~(\textcolor{blue}{\num{6.41e-6}})}$ & $.035~(\textcolor{blue}{\num{1.15e-5}})$ & $.035~(\textcolor{blue}{\num{1.53e-6}})$ & $.027~(\num{6.12e-4})$ & $.027~(\num{1.00e-3})$ & $.025~(\num{2.36e-3})$ & $.029~(\num{5.41e-4})$ \\
 & .30 & $\mathbf{.048~(\textcolor{blue}{\num{7.00e-7}})}$ & $.046~(\textcolor{blue}{\num{1.08e-7}})$ & $.016~(\num{1.79e-2})$ & $.033~(\num{2.18e-4})$ & $.035~(\num{1.98e-4})$ & $.032~(\num{5.10e-4})$ & $.037~(\textcolor{blue}{\num{6.72e-5}})$ \\
 & .40 & $\mathbf{.062~(\textcolor{blue}{\num{3.25e-8}})}$ & $.057~(\textcolor{blue}{\num{5.69e-9}})$ & $.007~(\num{2.58e-1})$ & $.041~(\textcolor{blue}{\num{7.50e-5}})$ & $.045~(\textcolor{blue}{\num{3.05e-5}})$ & $.041~(\num{1.25e-4})$ & $.043~(\textcolor{blue}{\num{4.53e-5}})$ \\
 & .50 & $\mathbf{.086~(\textcolor{blue}{\num{4.56e-11}})}$ & $.073~(\textcolor{blue}{\num{2.26e-10}})$ & $-.004~(\num{4.50e-1})$ & $.059~(\textcolor{blue}{\num{1.43e-6}})$ & $.056~(\textcolor{blue}{\num{1.31e-5}})$ & $.051~(\textcolor{blue}{\num{2.55e-5}})$ & $.058~(\textcolor{blue}{\num{3.50e-6}})$ \\
 & .60 & $\mathbf{.108~(\textcolor{blue}{\num{2.62e-12}})}$ & $.095~(\textcolor{blue}{\num{1.12e-11}})$ & $-.003~(\num{6.38e-1})$ & $.073~(\textcolor{blue}{\num{7.50e-7}})$ & $.079~(\textcolor{blue}{\num{5.51e-7}})$ & $.080~(\textcolor{blue}{\num{6.06e-8}})$ & $.079~(\textcolor{blue}{\num{5.44e-8}})$ \\
 & .70 & $\mathbf{.132~(\textcolor{blue}{\num{1.64e-13}})}$ & $.127~(\textcolor{blue}{\num{1.20e-12}})$ & $-.007~(\num{1.57e-1})$ & $.098~(\textcolor{blue}{\num{2.91e-8}})$ & $.116~(\textcolor{blue}{\num{7.68e-9}})$ & $.120~(\textcolor{blue}{\num{1.18e-10}})$ & $.122~(\textcolor{blue}{\num{8.34e-11}})$ \\
 & .80 & $.178~(\textcolor{blue}{\num{1.71e-14}})$ & $.200~(\textcolor{blue}{\num{7.25e-14}})$ & $-.005~(\num{3.17e-1})$ & $.126~(\textcolor{blue}{\num{2.25e-7}})$ & $.167~(\textcolor{blue}{\num{2.71e-10}})$ & $.189~(\textcolor{blue}{\num{4.93e-12}})$ & $\mathbf{.207~(\textcolor{blue}{\num{7.78e-15}})}$ \\
 & .90 & $.306~(\textcolor{blue}{\num{2.38e-14}})$ & $.277~(\textcolor{blue}{\num{1.53e-11}})$ & $-.005~(\num{3.17e-1})$ & $.193~(\textcolor{blue}{\num{2.14e-7}})$ & $.254~(\textcolor{blue}{\num{9.42e-11}})$ & $.286~(\textcolor{blue}{\num{2.39e-12}})$ & $\mathbf{.320~(\textcolor{blue}{\num{2.35e-13}})}$ \\

\midrule
\multirow{9}{*}{\rotatebox[origin=c]{90}{$\htrdd$}} & .10 & $-.843~(\num{9.54e-2})$ & $-.017~(\num{1.89e-1})$ & $\mathbf{.114~(\num{1.90e-1})}$ & $-.029~(\num{6.86e-2})$ & $-.005~(\num{5.60e-1})$ & $-.017~(\num{1.69e-3})$ & $-.015~(\num{8.08e-3})$ \\
 & .20 & $-.022~(\num{2.12e-1})$ & $-.046~(\num{3.61e-3})$ & $\mathbf{.122~(\num{2.26e-2})}$ & $-.011~(\num{4.58e-1})$ & $-.027~(\num{2.90e-2})$ & $-.022~(\num{9.53e-4})$ & $-.025~(\num{1.27e-4})$ \\
 & .30 & $-.043~(\textcolor{blue}{\num{1.73e-5}})$ & $-.000~(\num{9.92e-1})$ & $\mathbf{.213~(\num{4.70e-4})}$ & $-.005~(\num{7.28e-1})$ & $-.039~(\num{8.70e-2})$ & $-.029~(\num{3.00e-4})$ & $-.024~(\num{3.51e-4})$ \\
 & .40 & $-.039~(\num{1.65e-3})$ & $-.001~(\num{9.75e-1})$ & $\mathbf{.054~(\num{3.26e-1})}$ & $-.016~(\num{4.46e-1})$ & $-.009~(\num{5.19e-1})$ & $-.038~(\num{2.60e-4})$ & $-.037~(\num{1.01e-4})$ \\
 & .50 & $.006~(\num{8.04e-1})$ & $-.020~(\num{2.94e-1})$ & $\mathbf{.063~(\num{1.73e-1})}$ & $-.026~(\num{1.85e-1})$ & $-.020~(\num{3.68e-1})$ & $-.052~(\num{3.28e-4})$ & $-.048~(\num{6.66e-4})$ \\
 & .60 & $-.146~(\num{2.84e-2})$ & $\mathbf{.016~(\num{2.30e-1})}$ & $-.037~(\num{3.74e-1})$ & $-.019~(\num{6.48e-1})$ & $-.008~(\num{8.33e-1})$ & $-.048~(\num{5.91e-3})$ & $-.052~(\num{7.96e-3})$ \\
 & .70 & $.070~(\num{6.07e-1})$ & $-.001~(\num{7.48e-1})$ & $-.023~(\num{5.67e-1})$ & $.052~(\num{3.03e-1})$ & $\mathbf{.080~(\num{1.94e-1})}$ & $-.055~(\num{4.83e-2})$ & $-.059~(\num{4.96e-1})$ \\
 & .80 & $-.035~(\num{6.22e-1})$ & $\mathbf{.101~(\num{4.51e-2})}$ & $-.041~(\num{2.44e-1})$ & $-.014~(\num{8.57e-1})$ & $-.027~(\num{7.02e-1})$ & $.057~(\num{5.20e-1})$ & $.000~(\num{9.99e-1})$ \\
 & .90 & $.102~(\num{5.04e-1})$ & $\mathbf{.260~(\num{2.37e-2})}$ & $-.030~(\num{2.35e-1})$ & $.175~(\num{1.89e-1})$ & $.143~(\num{2.15e-1})$ & $.059~(\num{7.25e-1})$ & $.142~(\num{5.07e-1})$ \\

\midrule
\multirow{11}{*}{\rotatebox[origin=c]{90}{Accuracy}} & .00 & $\mathbf{.958}$ & $\mathbf{.958}$ & $\mathbf{.958}$ & $\mathbf{.958}$ & $\mathbf{.958}$ & $\mathbf{.958}$ & $\mathbf{.958}$ \\
 & .10 & $.959$ & $\mathbf{.963}$ & $.943$ & $.959$ & $.959$ & $.958$ & $.958$ \\
 & .20 & $.960$ & $\mathbf{.967}$ & $.936$ & $.961$ & $.959$ & $.959$ & $.959$ \\
 & .30 & $.963$ & $\mathbf{.971}$ & $.919$ & $.963$ & $.962$ & $.962$ & $.962$ \\
 & .40 & $.966$ & $\mathbf{.973}$ & $.913$ & $.964$ & $.964$ & $.964$ & $.963$ \\
 & .50 & $.971$ & $\mathbf{.975}$ & $.907$ & $.968$ & $.965$ & $.966$ & $.966$ \\
 & .60 & $.971$ & $\mathbf{.977}$ & $.908$ & $.967$ & $.968$ & $.972$ & $.970$ \\
 & .70 & $.970$ & $\mathbf{.978}$ & $.907$ & $.968$ & $.970$ & $.976$ & $.974$ \\
 & .80 & $.967$ & $\mathbf{.978}$ & $.908$ & $.962$ & $.970$ & $.975$ & $.977$ \\
 & .90 & $.958$ & $.965$ & $.908$ & $.955$ & $.962$ & $\mathbf{.967}$ & $.964$ \\
 & 1.00 & $.929$ & $\mathbf{.939}$ & $.909$ & $\mathbf{.939}$ & $.937$ & $\mathbf{.939}$ & $.936$ \\
\end{tabular}
}
\end{table}

\begin{table}[t]
\caption{\texttt{galaxyzoo} results, with the statistically significant ones at $\alpha=0.05$ in blue.}
\label{tab: galaxyzoo-results}
\resizebox{\textwidth}{!}{
\begin{tabular}{c|c|ccccccc}
 & $c$ & \ASM{} & \CC{} & \DT{} & \LCE{} & \OVA{} & \RS{} & \SP{} \\
\midrule
\multirow{10}{*}{\rotatebox[origin=c]{90}{$\htatd$}} & .00 & $-.103~(\textcolor{blue}{\num{5.85e-18}})$ & $-.100~(\textcolor{blue}{\num{2.93e-17}})$ & $-.093~(\textcolor{blue}{\num{1.11e-14}})$ & $\mathbf{-.082~(\textcolor{blue}{\num{3.00e-11}})}$ & $-.098~(\textcolor{blue}{\num{1.86e-16}})$ & $-.099~(\textcolor{blue}{\num{6.51e-17}})$ & $-.097~(\textcolor{blue}{\num{2.76e-15}})$ \\
 & .10 & $-.088~(\textcolor{blue}{\num{5.56e-12}})$ & $-.093~(\textcolor{blue}{\num{1.78e-13}})$ & $-.089~(\textcolor{blue}{\num{1.29e-12}})$ & $\mathbf{-.081~(\textcolor{blue}{\num{3.57e-10}})}$ & $-.086~(\textcolor{blue}{\num{3.04e-11}})$ & $-.098~(\textcolor{blue}{\num{4.10e-14}})$ & $-.096~(\textcolor{blue}{\num{2.76e-15}})$ \\
 & .20 & $-.076~(\textcolor{blue}{\num{1.33e-8}})$ & $-.082~(\textcolor{blue}{\num{8.01e-10}})$ & $-.094~(\textcolor{blue}{\num{1.74e-12}})$ & $-.074~(\textcolor{blue}{\num{6.20e-8}})$ & $\mathbf{-.072~(\textcolor{blue}{\num{2.08e-7}})}$ & $-.092~(\textcolor{blue}{\num{1.15e-10}})$ & $-.096~(\textcolor{blue}{\num{2.76e-15}})$ \\
 & .30 & $-.069~(\textcolor{blue}{\num{1.00e-6}})$ & $-.067~(\textcolor{blue}{\num{5.55e-6}})$ & $-.084~(\textcolor{blue}{\num{1.63e-9}})$ & $-.072~(\textcolor{blue}{\num{6.48e-7}})$ & $\mathbf{-.060~(\num{9.11e-5})}$ & $-.078~(\textcolor{blue}{\num{8.63e-7}})$ & $-.096~(\textcolor{blue}{\num{2.76e-15}})$ \\
 & .40 & $-.063~(\textcolor{blue}{\num{5.94e-5}})$ & $-.060~(\num{2.76e-4})$ & $-.082~(\textcolor{blue}{\num{2.45e-8}})$ & $\mathbf{-.051~(\num{1.23e-3})}$ & $-.057~(\num{5.83e-4})$ & $-.072~(\textcolor{blue}{\num{2.50e-5}})$ & $-.056~(\num{2.03e-3})$ \\
 & .50 & $-.060~(\num{4.07e-4})$ & $-.045~(\num{2.18e-2})$ & $-.090~(\textcolor{blue}{\num{1.11e-8}})$ & $-.039~(\num{3.18e-2})$ & $-.046~(\num{1.61e-2})$ & $-.039~(\num{4.31e-2})$ & $\mathbf{-.039~(\num{6.10e-2})}$ \\
 & .60 & $-.051~(\num{7.92e-3})$ & $-.013~(\num{5.68e-1})$ & $-.100~(\textcolor{blue}{\num{4.30e-8}})$ & $\mathbf{-.006~(\num{7.61e-1})}$ & $-.025~(\num{2.54e-1})$ & $-.006~(\num{7.74e-1})$ & $-.006~(\num{7.88e-1})$ \\
 & .70 & $-.036~(\num{1.16e-1})$ & $.009~(\num{7.49e-1})$ & $-.111~(\textcolor{blue}{\num{7.80e-8}})$ & $.013~(\num{5.85e-1})$ & $-.014~(\num{5.86e-1})$ & $.014~(\num{6.04e-1})$ & $\mathbf{.019~(\num{5.06e-1})}$ \\
 & .80 & $.017~(\num{5.84e-1})$ & $.068~(\num{4.66e-2})$ & $-.105~(\textcolor{blue}{\num{1.61e-5}})$ & $.028~(\num{3.64e-1})$ & $.007~(\num{8.12e-1})$ & $.060~(\num{7.60e-2})$ & $\mathbf{.079~(\num{2.79e-2})}$ \\
 & .90 & $.072~(\num{1.23e-1})$ & $.096~(\num{3.52e-2})$ & $-.119~(\textcolor{blue}{\num{5.32e-6}})$ & $.102~(\num{1.88e-2})$ & $.024~(\num{6.01e-1})$ & $.089~(\num{7.55e-2})$ & $\mathbf{.105~(\num{3.11e-2})}$ \\

\midrule
\multirow{9}{*}{\rotatebox[origin=c]{90}{$\htrdd$}} & .10 & $-.177~(\num{2.37e-1})$ & $-.150~(\num{7.54e-5})$ & $\mathbf{.496~(\num{2.64e-2})}$ & $-.192~(\num{7.48e-2})$ & $-.262~(\textcolor{blue}{\num{1.68e-19}})$ & $-.162~(\textcolor{blue}{\num{1.41e-14}})$ & $--$ \\
 & .20 & $-.216~(\num{1.18e-1})$ & $-.153~(\textcolor{blue}{\num{3.66e-13}})$ & $\mathbf{.084~(\num{6.61e-1})}$ & $-.193~(\num{3.60e-3})$ & $-.110~(\num{1.84e-1})$ & $-.277~(\textcolor{blue}{\num{7.73e-26}})$ & $--$ \\
 & .30 & $-.187~(\num{3.00e-1})$ & $-.165~(\textcolor{blue}{\num{2.90e-20}})$ & $-.511~(\num{1.05e-1})$ & $\mathbf{-.086~(\num{2.47e-1})}$ & $-.152~(\num{2.89e-2})$ & $-.277~(\textcolor{blue}{\num{2.63e-28}})$ & $--$ \\
 & .40 & $-.040~(\num{7.16e-1})$ & $-.180~(\textcolor{blue}{\num{3.18e-20}})$ & $\mathbf{.097~(\num{4.50e-1})}$ & $-.082~(\num{8.66e-2})$ & $-.180~(\num{4.49e-2})$ & $-.215~(\num{1.02e-2})$ & $--$ \\
 & .50 & $-.050~(\num{5.46e-1})$ & $-.206~(\textcolor{blue}{\num{6.27e-15}})$ & $-.023~(\num{8.38e-1})$ & $-.261~(\textcolor{blue}{\num{1.17e-5}})$ & $\mathbf{.060~(\num{5.84e-1})}$ & $-.250~(\num{4.18e-4})$ & $-.220~(\num{1.06e-1})$ \\
 & .60 & $\mathbf{.012~(\num{8.70e-1})}$ & $-.228~(\textcolor{blue}{\num{6.41e-10}})$ & $-.067~(\num{5.97e-1})$ & $-.106~(\num{5.36e-2})$ & $-.358~(\num{8.70e-2})$ & $-.287~(\num{2.55e-2})$ & $-.290~(\num{2.44e-2})$ \\
 & .70 & $-.123~(\num{1.66e-2})$ & $-.178~(\num{1.16e-2})$ & $-.312~(\num{1.24e-2})$ & $\mathbf{-.018~(\num{7.69e-1})}$ & $-.381~(\num{2.83e-2})$ & $-.152~(\num{5.62e-1})$ & $-.036~(\num{8.97e-1})$ \\
 & .80 & $-.013~(\num{8.67e-1})$ & $.153~(\num{1.84e-1})$ & $--$ & $-.089~(\num{2.41e-1})$ & $-.295~(\num{7.49e-2})$ & $\mathbf{.213~(\num{3.18e-1})}$ & $-.353~(\num{1.44e-1})$ \\
 & .90 & $-.066~(\num{7.63e-1})$ & $\mathbf{.211~(\num{8.75e-2})}$ & $--$ & $.158~(\num{8.07e-2})$ & $-.010~(\num{9.59e-1})$ & $--$ & $-.184~(\num{6.52e-1})$ \\

\midrule
\multirow{11}{*}{\rotatebox[origin=c]{90}{Accuracy}} & .00 & $\mathbf{.743}$ & $\mathbf{.743}$ & $\mathbf{.743}$ & $\mathbf{.743}$ & $\mathbf{.743}$ & $\mathbf{.743}$ & $\mathbf{.743}$ \\
 & .10 & $\mathbf{.767}$ & $.760$ & $.755$ & $.751$ & $.764$ & $.753$ & $.743$ \\
 & .20 & $\mathbf{.784}$ & $.777$ & $.761$ & $.764$ & $.783$ & $.768$ & $.743$ \\
 & .30 & $.796$ & $.796$ & $.776$ & $.772$ & $\mathbf{.800}$ & $.788$ & $.743$ \\
 & .40 & $\mathbf{.808}$ & $.806$ & $.786$ & $.793$ & $.807$ & $.797$ & $.806$ \\
 & .50 & $.815$ & $.821$ & $.791$ & $.805$ & $.819$ & $\mathbf{.822}$ & $.821$ \\
 & .60 & $.825$ & $.838$ & $.797$ & $.822$ & $.831$ & $\mathbf{.839}$ & $.837$ \\
 & .70 & $.834$ & $\mathbf{.845}$ & $.803$ & $.829$ & $.837$ & $\mathbf{.845}$ & $.845$ \\
 & .80 & $.849$ & $\mathbf{.856}$ & $.816$ & $.830$ & $.842$ & $.853$ & $.854$ \\
 & .90 & $\mathbf{.853}$ & $\mathbf{.853}$ & $.817$ & $.834$ & $.843$ & $.849$ & $.850$ \\
 & 1.00 & $\mathbf{.846}$ & $.843$ & $.836$ & $.825$ & $.841$ & $.841$ & $.839$ \\
\end{tabular}
}
\end{table}

\begin{table}[t]
\caption{\texttt{hatespeech} results, with the statistically significant ones at $\alpha=0.05$ in blue.}
\label{tab: hatespeech-results}
\resizebox{\textwidth}{!}{
\begin{tabular}{c|c|ccccccc}
 & $c$ & \ASM{} & \CC{} & \DT{} & \LCE{} & \OVA{} & \RS{} & \SP{} \\
\midrule
\multirow{10}{*}{\rotatebox[origin=c]{90}{$\htatd$}} & .00 & $.067~(\textcolor{blue}{\num{1.09e-26}})$ & $.033~(\textcolor{blue}{\num{8.77e-9}})$ & $\mathbf{.097~(\textcolor{blue}{\num{1.38e-49}})}$ & $.037~(\textcolor{blue}{\num{2.13e-10}})$ & $.031~(\textcolor{blue}{\num{8.31e-8}})$ & $.031~(\textcolor{blue}{\num{9.51e-8}})$ & $.030~(\textcolor{blue}{\num{1.67e-7}})$ \\
 & .10 & $.078~(\textcolor{blue}{\num{7.62e-30}})$ & $.037~(\textcolor{blue}{\num{1.57e-9}})$ & $\mathbf{.089~(\textcolor{blue}{\num{3.99e-41}})}$ & $.042~(\textcolor{blue}{\num{4.38e-11}})$ & $.038~(\textcolor{blue}{\num{1.80e-9}})$ & $.038~(\textcolor{blue}{\num{1.26e-9}})$ & $.036~(\textcolor{blue}{\num{1.37e-8}})$ \\
 & .20 & $\mathbf{.092~(\textcolor{blue}{\num{9.70e-35}})}$ & $.047~(\textcolor{blue}{\num{1.06e-12}})$ & $.078~(\textcolor{blue}{\num{3.03e-30}})$ & $.047~(\textcolor{blue}{\num{9.93e-12}})$ & $.044~(\textcolor{blue}{\num{1.26e-10}})$ & $.046~(\textcolor{blue}{\num{3.65e-11}})$ & $.047~(\textcolor{blue}{\num{3.23e-11}})$ \\
 & .30 & $\mathbf{.107~(\textcolor{blue}{\num{6.96e-38}})}$ & $.058~(\textcolor{blue}{\num{7.32e-16}})$ & $.074~(\textcolor{blue}{\num{3.63e-26}})$ & $.060~(\textcolor{blue}{\num{1.25e-15}})$ & $.059~(\textcolor{blue}{\num{2.97e-15}})$ & $.054~(\textcolor{blue}{\num{1.17e-12}})$ & $.058~(\textcolor{blue}{\num{1.00e-13}})$ \\
 & .40 & $\mathbf{.129~(\textcolor{blue}{\num{4.28e-45}})}$ & $.071~(\textcolor{blue}{\num{6.29e-18}})$ & $.062~(\textcolor{blue}{\num{7.28e-18}})$ & $.076~(\textcolor{blue}{\num{2.94e-20}})$ & $.075~(\textcolor{blue}{\num{2.33e-19}})$ & $.070~(\textcolor{blue}{\num{7.55e-16}})$ & $.076~(\textcolor{blue}{\num{6.24e-18}})$ \\
 & .50 & $\mathbf{.159~(\textcolor{blue}{\num{2.89e-51}})}$ & $.088~(\textcolor{blue}{\num{3.59e-21}})$ & $.052~(\textcolor{blue}{\num{7.78e-12}})$ & $.094~(\textcolor{blue}{\num{2.03e-23}})$ & $.094~(\textcolor{blue}{\num{5.16e-22}})$ & $.091~(\textcolor{blue}{\num{1.45e-20}})$ & $.096~(\textcolor{blue}{\num{8.19e-21}})$ \\
 & .60 & $\mathbf{.202~(\textcolor{blue}{\num{5.31e-59}})}$ & $.117~(\textcolor{blue}{\num{8.60e-25}})$ & $.055~(\textcolor{blue}{\num{5.01e-11}})$ & $.111~(\textcolor{blue}{\num{1.72e-24}})$ & $.113~(\textcolor{blue}{\num{3.40e-23}})$ & $.124~(\textcolor{blue}{\num{3.03e-26}})$ & $.117~(\textcolor{blue}{\num{3.30e-22}})$ \\
 & .70 & $\mathbf{.281~(\textcolor{blue}{\num{9.86e-73}})}$ & $.160~(\textcolor{blue}{\num{1.29e-29}})$ & $.052~(\textcolor{blue}{\num{1.76e-8}})$ & $.134~(\textcolor{blue}{\num{4.30e-24}})$ & $.138~(\textcolor{blue}{\num{1.86e-23}})$ & $.172~(\textcolor{blue}{\num{3.85e-34}})$ & $.165~(\textcolor{blue}{\num{8.30e-30}})$ \\
 & .80 & $\mathbf{.359~(\textcolor{blue}{\num{1.85e-73}})}$ & $.212~(\textcolor{blue}{\num{5.52e-33}})$ & $.049~(\textcolor{blue}{\num{1.37e-5}})$ & $.185~(\textcolor{blue}{\num{4.23e-27}})$ & $.165~(\textcolor{blue}{\num{1.38e-20}})$ & $.232~(\textcolor{blue}{\num{4.13e-35}})$ & $.201~(\textcolor{blue}{\num{1.64e-26}})$ \\
 & .90 & $\mathbf{.468~(\textcolor{blue}{\num{1.71e-63}})}$ & $.318~(\textcolor{blue}{\num{2.10e-33}})$ & $.034~(\num{1.78e-2})$ & $.222~(\textcolor{blue}{\num{5.48e-20}})$ & $.213~(\textcolor{blue}{\num{8.30e-17}})$ & $.295~(\textcolor{blue}{\num{6.56e-28}})$ & $.255~(\textcolor{blue}{\num{1.15e-19}})$ \\

\midrule
\multirow{9}{*}{\rotatebox[origin=c]{90}{$\htrdd$}} & .10 & $-.129~(\num{2.04e-2})$ & $-.037~(\num{3.12e-1})$ & $\mathbf{.122~(\num{9.14e-2})}$ & $-.044~(\num{1.64e-1})$ & $.068~(\num{2.39e-1})$ & $-.020~(\num{3.79e-1})$ & $-.026~(\num{2.93e-1})$ \\
 & .20 & $-.020~(\num{5.38e-1})$ & $-.063~(\num{1.27e-2})$ & $\mathbf{.123~(\num{2.79e-2})}$ & $-.041~(\num{1.23e-1})$ & $.001~(\num{9.90e-1})$ & $-.000~(\num{9.89e-1})$ & $-.046~(\num{2.71e-2})$ \\
 & .30 & $.014~(\num{7.85e-1})$ & $.008~(\num{6.90e-1})$ & $\mathbf{.057~(\num{2.78e-1})}$ & $-.061~(\num{2.66e-2})$ & $-.035~(\num{3.35e-1})$ & $.005~(\num{8.97e-1})$ & $-.049~(\num{6.71e-2})$ \\
 & .40 & $-.006~(\num{8.85e-1})$ & $-.031~(\num{4.91e-2})$ & $\mathbf{.121~(\num{2.29e-3})}$ & $-.057~(\num{2.34e-2})$ & $-.021~(\num{6.03e-1})$ & $-.075~(\num{1.45e-2})$ & $-.064~(\num{2.53e-2})$ \\
 & .50 & $-.038~(\num{2.78e-1})$ & $.003~(\num{8.80e-1})$ & $\mathbf{.133~(\num{7.73e-3})}$ & $.024~(\num{3.98e-1})$ & $-.004~(\num{9.19e-1})$ & $-.092~(\num{2.70e-2})$ & $-.037~(\num{3.80e-1})$ \\
 & .60 & $\mathbf{.044~(\num{3.19e-1})}$ & $-.035~(\num{7.25e-2})$ & $-.041~(\num{4.18e-1})$ & $.026~(\num{3.91e-1})$ & $-.021~(\num{6.29e-1})$ & $-.068~(\num{1.20e-1})$ & $-.006~(\num{9.15e-1})$ \\
 & .70 & $-.113~(\num{1.66e-1})$ & $.071~(\num{8.66e-2})$ & $.098~(\num{7.61e-3})$ & $-.009~(\num{8.13e-1})$ & $.125~(\num{2.68e-3})$ & $\mathbf{.127~(\num{2.08e-3})}$ & $-.045~(\num{4.37e-1})$ \\
 & .80 & $\mathbf{.151~(\num{4.38e-2})}$ & $.035~(\num{5.19e-1})$ & $-.030~(\num{5.77e-1})$ & $.097~(\num{1.18e-2})$ & $.028~(\num{6.77e-1})$ & $.067~(\num{2.58e-1})$ & $.066~(\num{4.38e-1})$ \\
 & .90 & $.315~(\num{8.58e-5})$ & $.218~(\num{2.66e-2})$ & $.054~(\num{3.27e-1})$ & $.267~(\num{7.86e-5})$ & $.153~(\num{1.32e-1})$ & $\mathbf{.348~(\textcolor{blue}{\num{4.10e-5}})}$ & $.112~(\num{2.68e-1})$ \\

\midrule
\multirow{11}{*}{\rotatebox[origin=c]{90}{Accuracy}} & .00 & $\mathbf{.908}$ & $\mathbf{.908}$ & $\mathbf{.908}$ & $\mathbf{.908}$ & $\mathbf{.908}$ & $\mathbf{.908}$ & $\mathbf{.908}$ \\
 & .10 & $.911$ & $.908$ & $.893$ & $.909$ & $.911$ & $\mathbf{.912}$ & $.910$ \\
 & .20 & $.914$ & $.912$ & $.874$ & $.908$ & $.912$ & $.914$ & $\mathbf{.915}$ \\
 & .30 & $.915$ & $.915$ & $.863$ & $.913$ & $\mathbf{.918}$ & $.916$ & $\mathbf{.918}$ \\
 & .40 & $.919$ & $.916$ & $.849$ & $.918$ & $.922$ & $.920$ & $\mathbf{.924}$ \\
 & .50 & $.921$ & $.918$ & $.837$ & $.919$ & $.923$ & $.924$ & $\mathbf{.926}$ \\
 & .60 & $.922$ & $.920$ & $.833$ & $.917$ & $.922$ & $\mathbf{.928}$ & $.925$ \\
 & .70 & $.922$ & $.921$ & $.827$ & $.912$ & $.918$ & $\mathbf{.930}$ & $.928$ \\
 & .80 & $.910$ & $.919$ & $.821$ & $.908$ & $.910$ & $\mathbf{.923}$ & $.919$ \\
 & .90 & $.883$ & $.907$ & $.815$ & $.894$ & $.899$ & $\mathbf{.908}$ & $.904$ \\
 & 1.00 & $.841$ & $.875$ & $.812$ & $.871$ & $.877$ & $.878$ & $\mathbf{.878}$ \\
\end{tabular}
}
\end{table}

\begin{table}[t]
\caption{\texttt{xray-airspace} results, with the statistically significant ones at $\alpha=0.05$ in blue.}
\label{tab: xray-airspace-results}
\resizebox{\textwidth}{!}{
\begin{tabular}{c|c|ccccccc}
 & $c$ & \ASM{} & \CC{} & \DT{} & \LCE{} & \OVA{} & \RS{} & \SP{} \\
\midrule
\multirow{10}{*}{\rotatebox[origin=c]{90}{$\htatd$}} & .00 & $\mathbf{.359~(\textcolor{blue}{\num{2.27e-59}})}$ & $.029~(\num{6.71e-2})$ & $.189~(\textcolor{blue}{\num{1.46e-21}})$ & $.033~(\num{4.08e-2})$ & $.033~(\num{3.76e-2})$ & $.039~(\num{1.78e-2})$ & $.023~(\num{1.31e-1})$ \\
 & .10 & $\mathbf{.416~(\textcolor{blue}{\num{5.20e-70}})}$ & $.046~(\num{1.18e-2})$ & $.196~(\textcolor{blue}{\num{2.76e-22}})$ & $.053~(\num{3.79e-3})$ & $.055~(\num{2.26e-3})$ & $.059~(\num{1.30e-3})$ & $.034~(\num{5.55e-2})$ \\
 & .20 & $\mathbf{.482~(\textcolor{blue}{\num{1.73e-86}})}$ & $.063~(\num{1.45e-3})$ & $.199~(\textcolor{blue}{\num{6.39e-22}})$ & $.068~(\num{5.38e-4})$ & $.072~(\num{1.96e-4})$ & $.080~(\textcolor{blue}{\num{7.05e-5}})$ & $.038~(\num{4.97e-2})$ \\
 & .30 & $\mathbf{.585~(\textcolor{blue}{\num{2.91e-115}})}$ & $.079~(\num{3.09e-4})$ & $.209~(\textcolor{blue}{\num{2.64e-22}})$ & $.085~(\textcolor{blue}{\num{5.32e-5}})$ & $.082~(\textcolor{blue}{\num{3.32e-5}})$ & $.091~(\textcolor{blue}{\num{1.32e-5}})$ & $.043~(\num{4.27e-2})$ \\
 & .40 & $\mathbf{.673~(\textcolor{blue}{\num{3.57e-144}})}$ & $.092~(\num{1.08e-4})$ & $.239~(\textcolor{blue}{\num{4.79e-26}})$ & $.092~(\textcolor{blue}{\num{5.99e-5}})$ & $.108~(\textcolor{blue}{\num{1.08e-7}})$ & $.123~(\textcolor{blue}{\num{5.10e-9}})$ & $.068~(\num{4.51e-3})$ \\
 & .50 & $\mathbf{.730~(\textcolor{blue}{\num{9.64e-161}})}$ & $.104~(\num{7.99e-5})$ & $.242~(\textcolor{blue}{\num{6.45e-23}})$ & $.106~(\textcolor{blue}{\num{7.34e-5}})$ & $.115~(\textcolor{blue}{\num{1.56e-7}})$ & $.140~(\textcolor{blue}{\num{1.82e-9}})$ & $.085~(\num{2.03e-3})$ \\
 & .60 & $\mathbf{.800~(\textcolor{blue}{\num{5.17e-193}})}$ & $.125~(\textcolor{blue}{\num{3.50e-5}})$ & $.254~(\textcolor{blue}{\num{5.80e-20}})$ & $.120~(\textcolor{blue}{\num{5.09e-5}})$ & $.159~(\textcolor{blue}{\num{5.70e-11}})$ & $.149~(\textcolor{blue}{\num{2.12e-8}})$ & $.130~(\textcolor{blue}{\num{2.40e-5}})$ \\
 & .70 & $\mathbf{.858~(\textcolor{blue}{\num{1.58e-202}})}$ & $.183~(\textcolor{blue}{\num{1.90e-8}})$ & $.191~(\textcolor{blue}{\num{1.97e-11}})$ & $.160~(\textcolor{blue}{\num{7.21e-6}})$ & $.173~(\textcolor{blue}{\num{8.75e-10}})$ & $.149~(\textcolor{blue}{\num{1.24e-6}})$ & $.158~(\textcolor{blue}{\num{7.87e-6}})$ \\
 & .80 & $\mathbf{.887~(\textcolor{blue}{\num{1.93e-146}})}$ & $.271~(\textcolor{blue}{\num{5.94e-11}})$ & $.032~(\num{1.00e-1})$ & $.203~(\textcolor{blue}{\num{2.04e-6}})$ & $.181~(\textcolor{blue}{\num{1.99e-8}})$ & $.209~(\textcolor{blue}{\num{3.82e-9}})$ & $.218~(\textcolor{blue}{\num{3.98e-7}})$ \\
 & .90 & $\mathbf{.887~(\textcolor{blue}{\num{1.93e-146}})}$ & $.352~(\textcolor{blue}{\num{2.28e-9}})$ & $.033~(\num{1.54e-1})$ & $.320~(\textcolor{blue}{\num{1.61e-8}})$ & $.194~(\num{1.30e-4})$ & $.179~(\textcolor{blue}{\num{2.16e-5}})$ & $.253~(\textcolor{blue}{\num{7.28e-5}})$ \\

\midrule
\multirow{9}{*}{\rotatebox[origin=c]{90}{$\htrdd$}} & .10 & $-.272~(\num{1.29e-1})$ & $-.215~(\num{4.69e-1})$ & $\mathbf{.059~(\num{7.38e-1})}$ & $-.190~(\num{9.40e-2})$ & $-.123~(\num{4.68e-1})$ & $-.312~(\num{1.22e-2})$ & $-.034~(\num{3.36e-1})$ \\
 & .20 & $-.558~(\num{1.22e-2})$ & $-.013~(\num{7.70e-1})$ & $\mathbf{.109~(\num{5.37e-1})}$ & $-.012~(\num{9.34e-1})$ & $-.162~(\num{1.19e-1})$ & $-.161~(\num{1.43e-1})$ & $.042~(\num{6.52e-1})$ \\
 & .30 & $.073~(\num{6.71e-1})$ & $-.058~(\num{5.48e-1})$ & $-.067~(\num{6.38e-1})$ & $.103~(\num{1.90e-1})$ & $\mathbf{.726~(\num{3.63e-4})}$ & $-.473~(\num{8.95e-3})$ & $-.273~(\num{4.35e-3})$ \\
 & .40 & $\mathbf{.322~(\num{1.03e-1})}$ & $.027~(\num{7.69e-1})$ & $-.022~(\num{8.44e-1})$ & $.092~(\num{1.06e-1})$ & $-.153~(\num{2.13e-1})$ & $-.069~(\num{5.42e-1})$ & $.055~(\num{5.15e-1})$ \\
 & .50 & $\mathbf{.354~(\num{1.38e-1})}$ & $.058~(\num{6.20e-1})$ & $.144~(\num{2.38e-1})$ & $.010~(\num{8.83e-1})$ & $.157~(\num{1.16e-1})$ & $.125~(\num{1.25e-1})$ & $-.018~(\num{8.06e-1})$ \\
 & .60 & $\mathbf{.465~(\num{2.12e-1})}$ & $.093~(\num{4.63e-1})$ & $.041~(\num{7.67e-1})$ & $.014~(\num{8.47e-1})$ & $.090~(\num{2.72e-1})$ & $.155~(\num{3.70e-2})$ & $-.171~(\num{1.60e-1})$ \\
 & .70 & $.357~(\num{3.82e-1})$ & $-.180~(\num{1.35e-1})$ & $\mathbf{.638~(\textcolor{blue}{\num{5.77e-5}})}$ & $-.064~(\num{4.29e-1})$ & $-.044~(\num{6.15e-1})$ & $.073~(\num{3.52e-1})$ & $-.114~(\num{5.64e-1})$ \\
 & .80 & $--$ & $.137~(\num{3.60e-1})$ & $.072~(\num{6.93e-1})$ & $-.152~(\num{3.21e-1})$ & $.176~(\num{3.33e-2})$ & $.077~(\num{3.28e-1})$ & $\mathbf{.278~(\num{1.36e-1})}$ \\
 & .90 & $--$ & $.065~(\num{7.91e-1})$ & $-.222~(\num{1.94e-1})$ & $-.023~(\num{9.01e-1})$ & $.032~(\num{8.44e-1})$ & $\mathbf{.167~(\num{1.06e-1})}$ & $.008~(\num{9.54e-1})$ \\

\midrule
\multirow{9}{*}{\rotatebox[origin=c]{90}{$\htcatd-{\texttt{Male}}$}} & .00 & $\mathbf{.402~(\textcolor{blue}{\num{2.24e-42}})}$ & $.054~(\num{1.48e-2})$ & $.199~(\textcolor{blue}{\num{2.14e-14}})$ & $.062~(\num{5.22e-3})$ & $.058~(\num{7.55e-3})$ & $.062~(\num{5.22e-3})$ & $.058~(\num{7.55e-3})$ \\
 & .10 & $\mathbf{.451~(\textcolor{blue}{\num{1.13e-47}})}$ & $.063~(\num{1.06e-2})$ & $.208~(\textcolor{blue}{\num{4.56e-15}})$ & $.070~(\num{4.74e-3})$ & $.070~(\num{3.56e-3})$ & $.072~(\num{3.59e-3})$ & $.066~(\num{7.50e-3})$ \\
 & .20 & $\mathbf{.507~(\textcolor{blue}{\num{1.25e-55}})}$ & $.080~(\num{2.85e-3})$ & $.213~(\textcolor{blue}{\num{2.99e-15}})$ & $.094~(\num{3.78e-4})$ & $.086~(\num{9.02e-4})$ & $.095~(\num{3.77e-4})$ & $.071~(\num{8.96e-3})$ \\
 & .30 & $\mathbf{.594~(\textcolor{blue}{\num{9.86e-72}})}$ & $.094~(\num{2.20e-3})$ & $.227~(\textcolor{blue}{\num{1.20e-15}})$ & $.109~(\num{8.23e-5})$ & $.090~(\num{7.99e-4})$ & $.100~(\num{3.82e-4})$ & $.078~(\num{1.07e-2})$ \\
 & .40 & $\mathbf{.682~(\textcolor{blue}{\num{2.53e-90}})}$ & $.110~(\num{1.24e-3})$ & $.245~(\textcolor{blue}{\num{3.65e-16}})$ & $.116~(\num{1.40e-4})$ & $.130~(\textcolor{blue}{\num{2.07e-6}})$ & $.137~(\textcolor{blue}{\num{1.85e-6}})$ & $.116~(\num{8.76e-4})$ \\
 & .50 & $\mathbf{.719~(\textcolor{blue}{\num{3.04e-92}})}$ & $.134~(\num{4.93e-4})$ & $.249~(\textcolor{blue}{\num{8.22e-15}})$ & $.125~(\num{2.80e-4})$ & $.130~(\textcolor{blue}{\num{6.63e-6}})$ & $.150~(\textcolor{blue}{\num{7.08e-7}})$ & $.130~(\num{1.45e-3})$ \\
 & .60 & $\mathbf{.810~(\textcolor{blue}{\num{3.08e-127}})}$ & $.151~(\num{7.98e-4})$ & $.268~(\textcolor{blue}{\num{8.49e-14}})$ & $.140~(\num{4.69e-4})$ & $.170~(\textcolor{blue}{\num{4.09e-8}})$ & $.155~(\textcolor{blue}{\num{9.19e-6}})$ & $.185~(\textcolor{blue}{\num{4.44e-5}})$ \\
 & .70 & $\mathbf{.853~(\textcolor{blue}{\num{2.02e-119}})}$ & $.234~(\textcolor{blue}{\num{1.01e-6}})$ & $.176~(\textcolor{blue}{\num{5.73e-7}})$ & $.188~(\num{1.90e-4})$ & $.186~(\textcolor{blue}{\num{1.32e-7}})$ & $.180~(\textcolor{blue}{\num{1.00e-5}})$ & $.227~(\textcolor{blue}{\num{9.27e-6}})$ \\
 & .80 & $\mathbf{.863~(\textcolor{blue}{\num{4.02e-69}})}$ & $.337~(\textcolor{blue}{\num{1.23e-8}})$ & $.048~(\num{4.21e-2})$ & $.223~(\num{2.90e-4})$ & $.175~(\textcolor{blue}{\num{4.30e-5}})$ & $.262~(\textcolor{blue}{\num{5.84e-8}})$ & $.245~(\textcolor{blue}{\num{4.46e-5}})$ \\
 & .90 & $\mathbf{.863~(\textcolor{blue}{\num{4.02e-69}})}$ & $.436~(\textcolor{blue}{\num{7.48e-8}})$ & $.049~(\num{1.54e-1})$ & $.407~(\textcolor{blue}{\num{1.32e-7}})$ & $.156~(\num{1.66e-2})$ & $.261~(\textcolor{blue}{\num{6.86e-5}})$ & $.318~(\num{1.71e-4})$ \\

\midrule
\multirow{9}{*}{\rotatebox[origin=c]{90}{$\htcatd-{\texttt{Female}}$}} & .00 & $\mathbf{.311~(\textcolor{blue}{\num{4.87e-21}})}$ & $.002~(\num{9.16e-1})$ & $.178~(\textcolor{blue}{\num{3.92e-9}})$ & $-.000~(\num{1.00e+00})$ & $.005~(\num{8.29e-1})$ & $.012~(\num{6.04e-1})$ & $-.015~(\num{5.03e-1})$ \\
 & .10 & $\mathbf{.376~(\textcolor{blue}{\num{5.51e-26}})}$ & $.026~(\num{3.59e-1})$ & $.184~(\textcolor{blue}{\num{2.37e-9}})$ & $.031~(\num{2.45e-1})$ & $.035~(\num{1.91e-1})$ & $.044~(\num{1.17e-1})$ & $-.003~(\num{9.08e-1})$ \\
 & .20 & $\mathbf{.453~(\textcolor{blue}{\num{6.39e-34}})}$ & $.042~(\num{1.51e-1})$ & $.185~(\textcolor{blue}{\num{4.81e-9}})$ & $.035~(\num{2.25e-1})$ & $.054~(\num{6.16e-2})$ & $.061~(\num{4.54e-2})$ & $.000~(\num{1.00e+00})$ \\
 & .30 & $\mathbf{.573~(\textcolor{blue}{\num{3.46e-46}})}$ & $.062~(\num{4.76e-2})$ & $.190~(\textcolor{blue}{\num{5.24e-9}})$ & $.053~(\num{9.47e-2})$ & $.070~(\num{1.41e-2})$ & $.079~(\num{1.09e-2})$ & $.004~(\num{9.00e-1})$ \\
 & .40 & $\mathbf{.662~(\textcolor{blue}{\num{5.16e-57}})}$ & $.073~(\num{2.82e-2})$ & $.232~(\textcolor{blue}{\num{7.55e-12}})$ & $.059~(\num{8.81e-2})$ & $.078~(\num{8.58e-3})$ & $.106~(\num{6.72e-4})$ & $.017~(\num{6.00e-1})$ \\
 & .50 & $\mathbf{.744~(\textcolor{blue}{\num{3.95e-70}})}$ & $.073~(\num{4.15e-2})$ & $.236~(\textcolor{blue}{\num{3.07e-10}})$ & $.077~(\num{6.82e-2})$ & $.096~(\num{4.35e-3})$ & $.127~(\num{5.12e-4})$ & $.036~(\num{3.17e-1})$ \\
 & .60 & $\mathbf{.787~(\textcolor{blue}{\num{4.52e-71}})}$ & $.099~(\num{1.43e-2})$ & $.239~(\textcolor{blue}{\num{3.18e-8}})$ & $.094~(\num{3.50e-2})$ & $.143~(\num{2.21e-4})$ & $.140~(\num{6.22e-4})$ & $.071~(\num{8.38e-2})$ \\
 & .70 & $\mathbf{.867~(\textcolor{blue}{\num{1.86e-84}})}$ & $.132~(\num{2.66e-3})$ & $.214~(\textcolor{blue}{\num{8.23e-6}})$ & $.127~(\num{1.17e-2})$ & $.153~(\num{1.03e-3})$ & $.106~(\num{2.22e-2})$ & $.080~(\num{9.19e-2})$ \\
 & .80 & $\mathbf{.929~(\textcolor{blue}{\num{1.09e-118}})}$ & $.198~(\num{5.07e-4})$ & $.000~(\num{1.00e+00})$ & $.180~(\num{2.24e-3})$ & $.188~(\num{1.35e-4})$ & $.145~(\num{5.04e-3})$ & $.184~(\num{2.64e-3})$ \\
 & .90 & $\mathbf{.929~(\textcolor{blue}{\num{1.09e-118}})}$ & $.260~(\num{2.16e-3})$ & $--$ & $.217~(\num{7.76e-3})$ & $.233~(\num{2.95e-3})$ & $.079~(\num{7.46e-2})$ & $.179~(\num{6.21e-2})$ \\

\midrule
\multirow{11}{*}{\rotatebox[origin=c]{90}{Accuracy}} & .00 & $\mathbf{.871}$ & $\mathbf{.871}$ & $\mathbf{.871}$ & $\mathbf{.871}$ & $\mathbf{.871}$ & $\mathbf{.871}$ & $\mathbf{.871}$ \\
 & .10 & $.879$ & $.880$ & $.864$ & $.883$ & $\mathbf{.884}$ & $.883$ & $.877$ \\
 & .20 & $\mathbf{.892}$ & $.889$ & $.850$ & $.889$ & $.891$ & $.891$ & $.877$ \\
 & .30 & $\mathbf{.897}$ & $.893$ & $.838$ & $.893$ & $.891$ & $.891$ & $.877$ \\
 & .40 & $.877$ & $.894$ & $.836$ & $.890$ & $.900$ & $\mathbf{.902}$ & $.886$ \\
 & .50 & $.842$ & $.892$ & $.814$ & $.886$ & $.892$ & $\mathbf{.896}$ & $.887$ \\
 & .60 & $.789$ & $.890$ & $.783$ & $.885$ & $\mathbf{.897}$ & $.885$ & $\mathbf{.897}$ \\
 & .70 & $.726$ & $\mathbf{.900}$ & $.730$ & $.887$ & $.887$ & $.871$ & $.892$ \\
 & .80 & $.632$ & $\mathbf{.899}$ & $.687$ & $.884$ & $.869$ & $.870$ & $.891$ \\
 & .90 & $.632$ & $\mathbf{.885}$ & $.685$ & $.876$ & $.852$ & $.850$ & $.872$ \\
 & 1.00 & $.512$ & $.842$ & $.682$ & $.838$ & $.838$ & $.832$ & $\mathbf{.848}$ \\
\end{tabular}
}
\end{table}

}

\subsection{How to validate estimates under Scenario \ref{scenario2}}
\label{sec:appendixValidate}
When considering Scenario \ref{scenario2}, we require Assumption \ref{ass:continuity} to hold. Here, we provide a few sanity checks that can be implemented to validate the estimates (see Appendix~\ref{sec:assContinuityChecks} for a detailed description).

\subsubsection{Placebo Tests} \label{sec: placebo tests}

\paragraph{Placebo cutoff test setup.}  We consider the same setup presented in Section \ref{sec:experimentsGS}, except for the following: 
\begin{enumerate}
    \item  we estimate two different cutoff values $\ok_{c,L}$ and $\ok_{c,H}$
    \begin{itemize}
        \item $\ok_{c,L}$ is obtained by considering the 75-th percentile on the instances with a reject score below $\ok_c$,
        \item $\ok_{c,H}$ is estimated by considering the 25-th percentile of the reject scores above $\ok_c$;
    \end{itemize}
    \item  we run a local kernel polynomial regression to estimate $\htrdd$, substituting the true cutoff value $\ok_c$ with both $\ok_{c,L}$ and $\ok_{c,H}$.

\end{enumerate}

\paragraph{Placebo outcome test setup.} We consider the same setup presented in Section \ref{sec:experimentsGS}, except for the following: 
\begin{enumerate}
    \item we sample a placebo outcome $\Tilde{T}$, such that $\Tilde{T}\sim\mathtt{Bernoulli}(p)$ and $p=.5$;
    \item we run the same local kernel polynomial regression used to estimate $\htrdd$ substituting the target variable $T$ with $\Tilde{T}$.
\end{enumerate}

\paragraph{Results}
\afterpage{
\begin{figure}
    \centering
    \includegraphics[scale=.195]{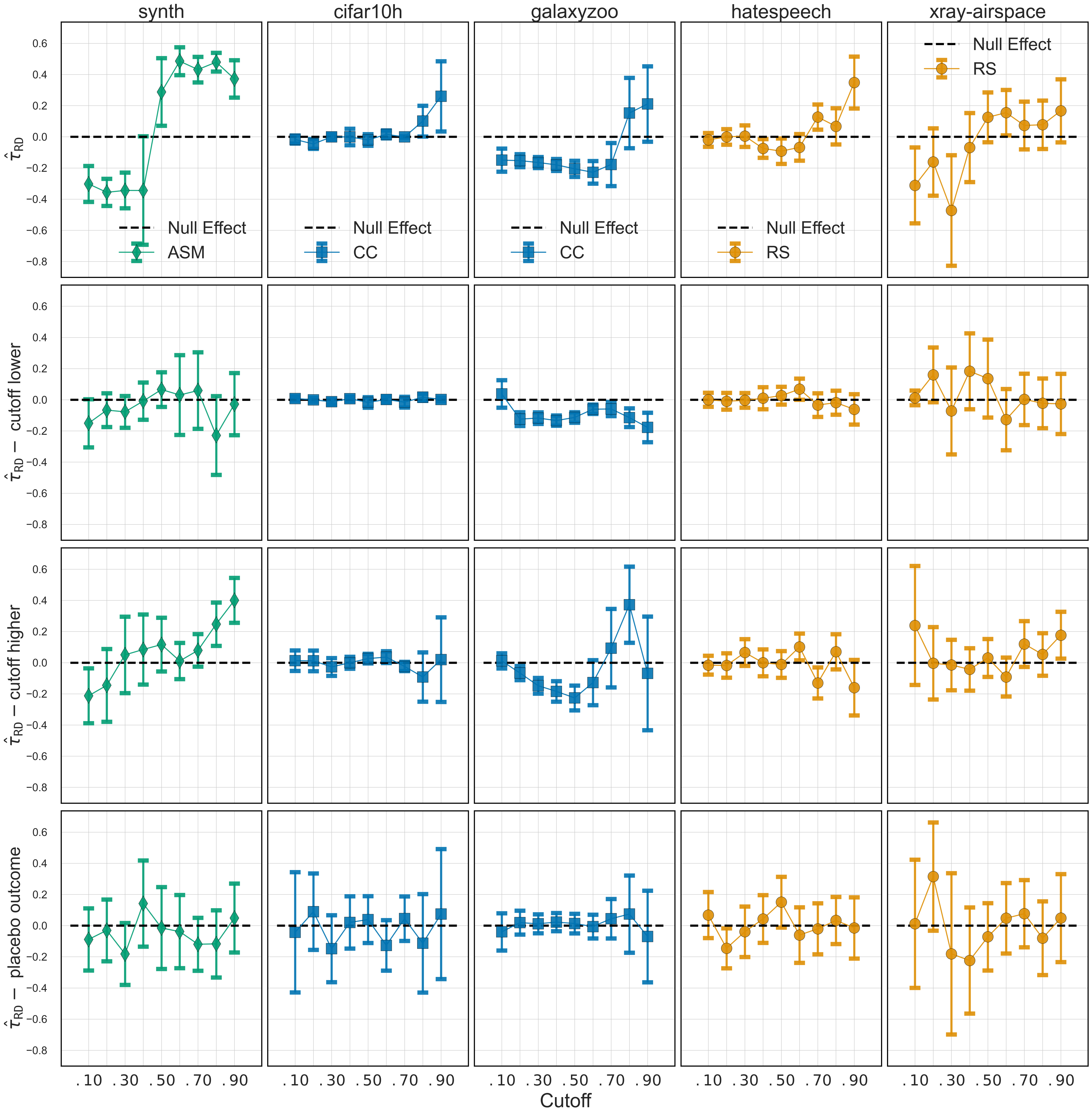}
    \caption{Estimated $\htrdd$ for the best baselines over the five datasets. The first row provides the estimates of $\htrdd$. The second row provides the estimates for the lower cutoff placebo test. The third row provides the estimates for the higher cutoff placebo test. The fourth row provides the estimates for the placebo outcome test.}
    \label{fig:placebo-test}
\end{figure}
}

Figure \ref{fig:placebo-test} provides the results for the placebo tests above. The first row shows the original estimates for $\htrdd$ under Scenario \ref{scenario2} for the best baselines.

Regarding \texttt{synth}, all the coefficients of the placebo tests are not significant, with the sole exception of high coverage values ($c\geq.80$)  when considering the highest cutoff (third row of Figure \ref{fig:placebo-test}).

When considering \texttt{cifar10h}, \texttt{hatespeech}, and \texttt{xray-airspace}, all the coefficients for all the placebo tests are not significant. This suggests that our estimated $\htrdd$ should be robust.

On the other hand, when looking at \texttt{galaxyzoo}, we see statistically significant estimates for the fake cutoff placebo estimates. This means we should take with a grain of salt our estimated $\htrdd$, as assumptions could be violated.

We can exploit other tests, such as the one detailed in the next subsection, to understand why the placebo tests fail.

\subsubsection{Density Estimation}\label{sec: density estimation}

\paragraph{Density estimation test setup.}

We consider the same setup presented in Section \ref{sec:experimentsGS}, and we validate $\htrdd$ estimates by estimating the empirical density of reject scores. This is done separately for instances below the cutoff and above the cutoff, using the $\texttt{R}$ package $\texttt{rddensity}$ \cite{DBLP:journals/jstatsoft/CattaneoJM22} with default parameters.  This allows us to: $(i)$ statistically assess the continuity of estimated reject score densities around the cutoff using a permutation test (null hypothesis stands for continuity at the cutoff); $(ii)$ visualize whether there is a discontinuity in the estimated reject score densities around the cutoffs.

\paragraph{Results}

Figures \ref{fig:synth estimated density best}-\ref{fig:xray-airspace estimated density best} provide the density estimation plots and the permutation tests p-values (high values mean we can't falsify the assumption) for the best baseline on all the datasets. If the running variable is not manipulable, we would expect to see no difference in the estimated densities from both sides of the cutoff. 

We can see that for the majority of cutoffs, the estimated densities are close, thus not falsifying Assumption \ref{ass:continuity}. 
The only exception is $\texttt{galaxyzoo}$ (Figure \ref{fig:galaxyzoo estimated density best}), where most tests reject the null hypothesis. The main reason for this behavior is that the reject score density peaks around one value, with little variation in the reject scores. Accordingly, small changes in the reject score imply abrupt changes in predictive accuracy. This sheds light on a potential criticality of the deferring system, i.e.,~the reject score estimation.
Moreover, this behavior also explains why most placebo cutoff tests failed for \texttt{galaxyzoo}.

Regarding \texttt{synth}, the null hypothesis is rejected only for $c\geq.70$. Once again, this is because the reject scores peak at $-1$ and $1$, potentially harming the quality of the estimates.
We see similar trends for \texttt{cifar10h} (Figure \ref{fig:cifar10h estimated density best}), where the p-values are significant only at $c=.10$, and \texttt{xray-airspace}, where values are significant for $c>.70$ and around the accuracy maximizing cutoff value, i.e.,~.$\ok_{.40}$ and $\ok_{.50}$.

Regarding \texttt{hatespeech}, we see no statistically significant p-values, suggesting the validity of $\htrdd$ across all cutoff values (see Figure \ref{fig:hatespeech estimated density best}).

\afterpage{
\begin{figure*}[t!]
    \begin{subfigure}[t]{.32\textwidth}
        \includegraphics[scale=.22]{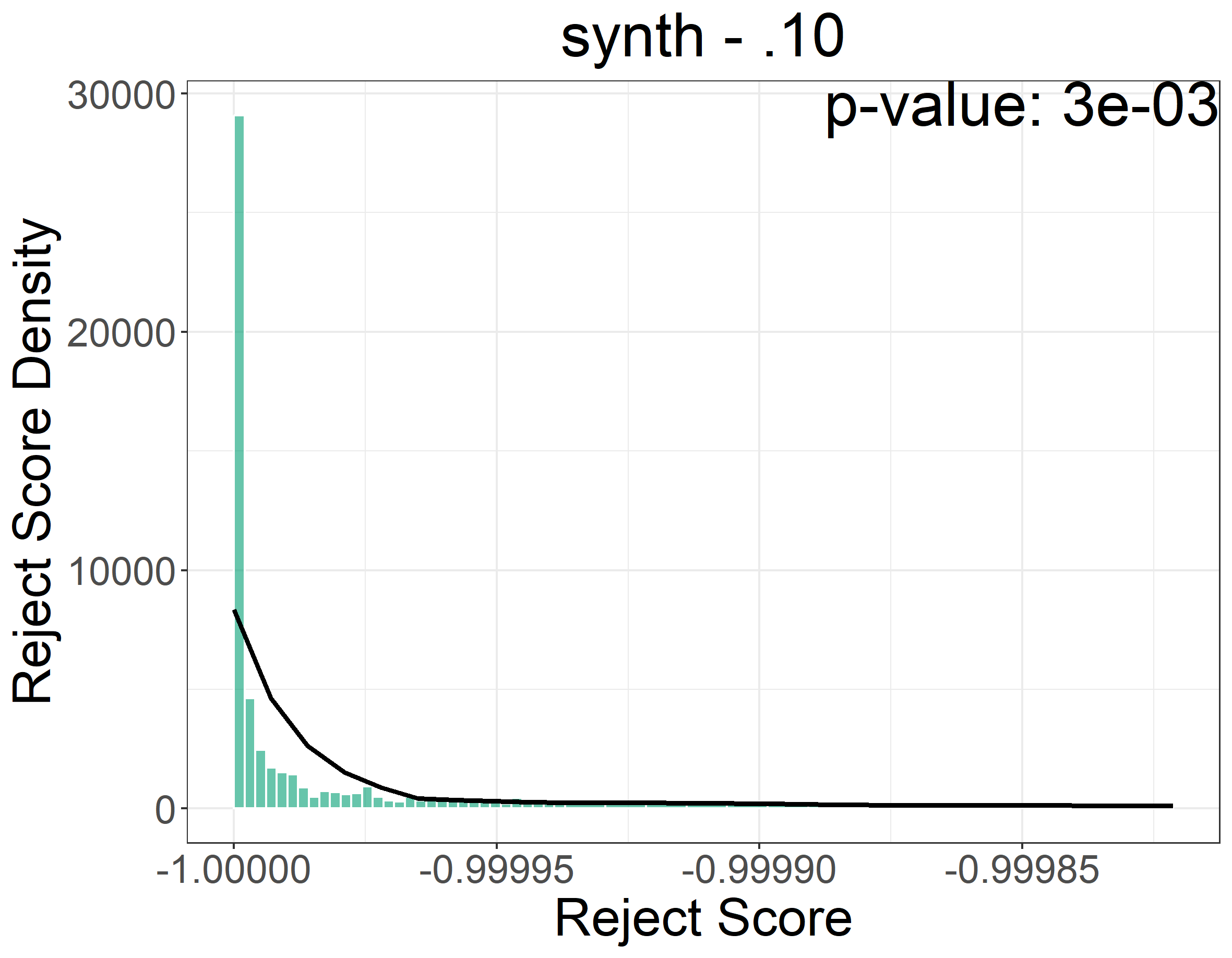}
        \caption{$\ok_{.10}$.}
        \label{fig:synth_best_density_cutoff.10}
    \end{subfigure}
    \hfill
    \begin{subfigure}[t]{.32\textwidth}
        \includegraphics[scale=.22]{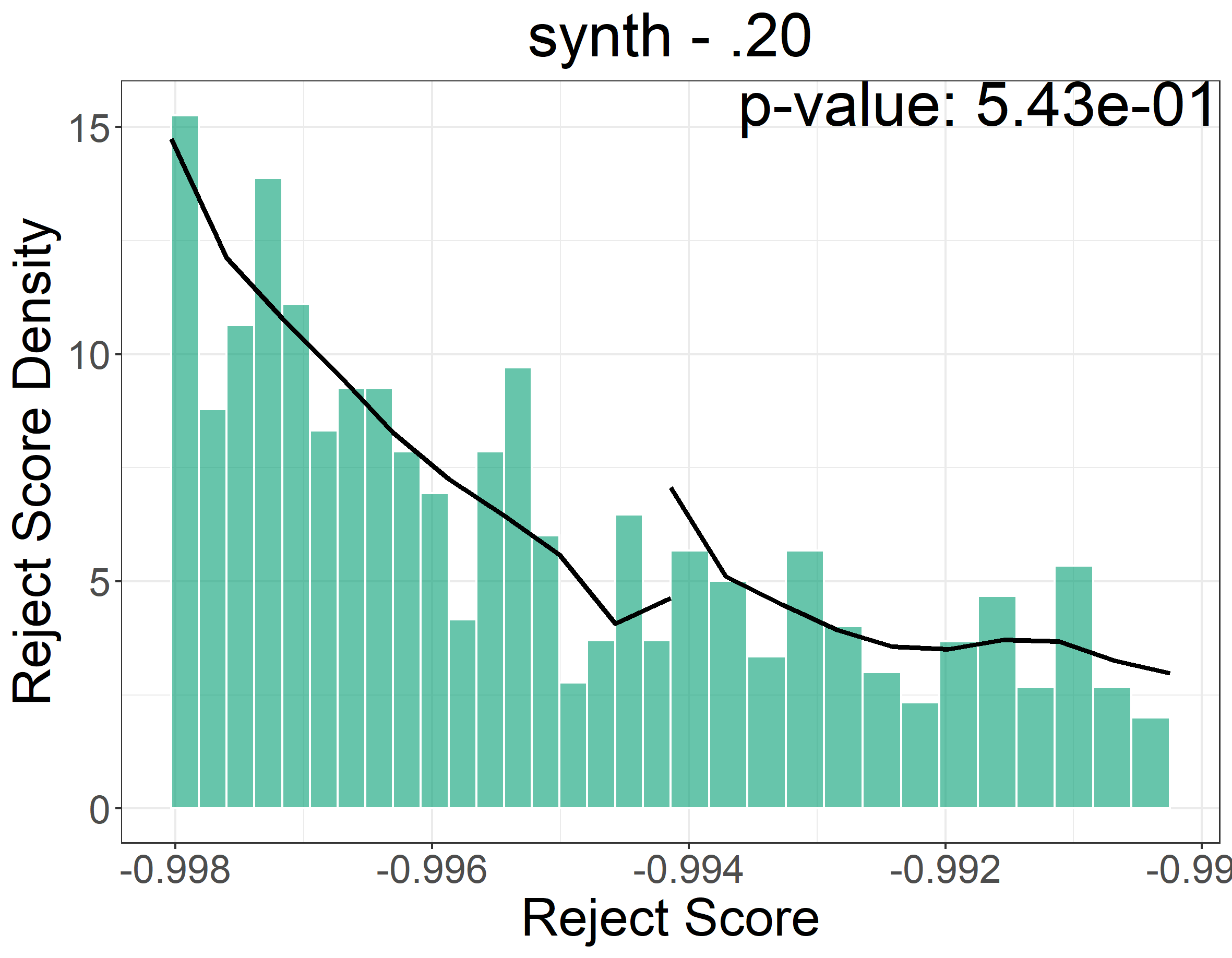}
        \caption{$\ok_{.20}$.}
        \label{fig:synth_best_density_cutoff.20}
    \end{subfigure}
    \hfill
    \begin{subfigure}[t]{.32\textwidth}
        \includegraphics[scale=.22]{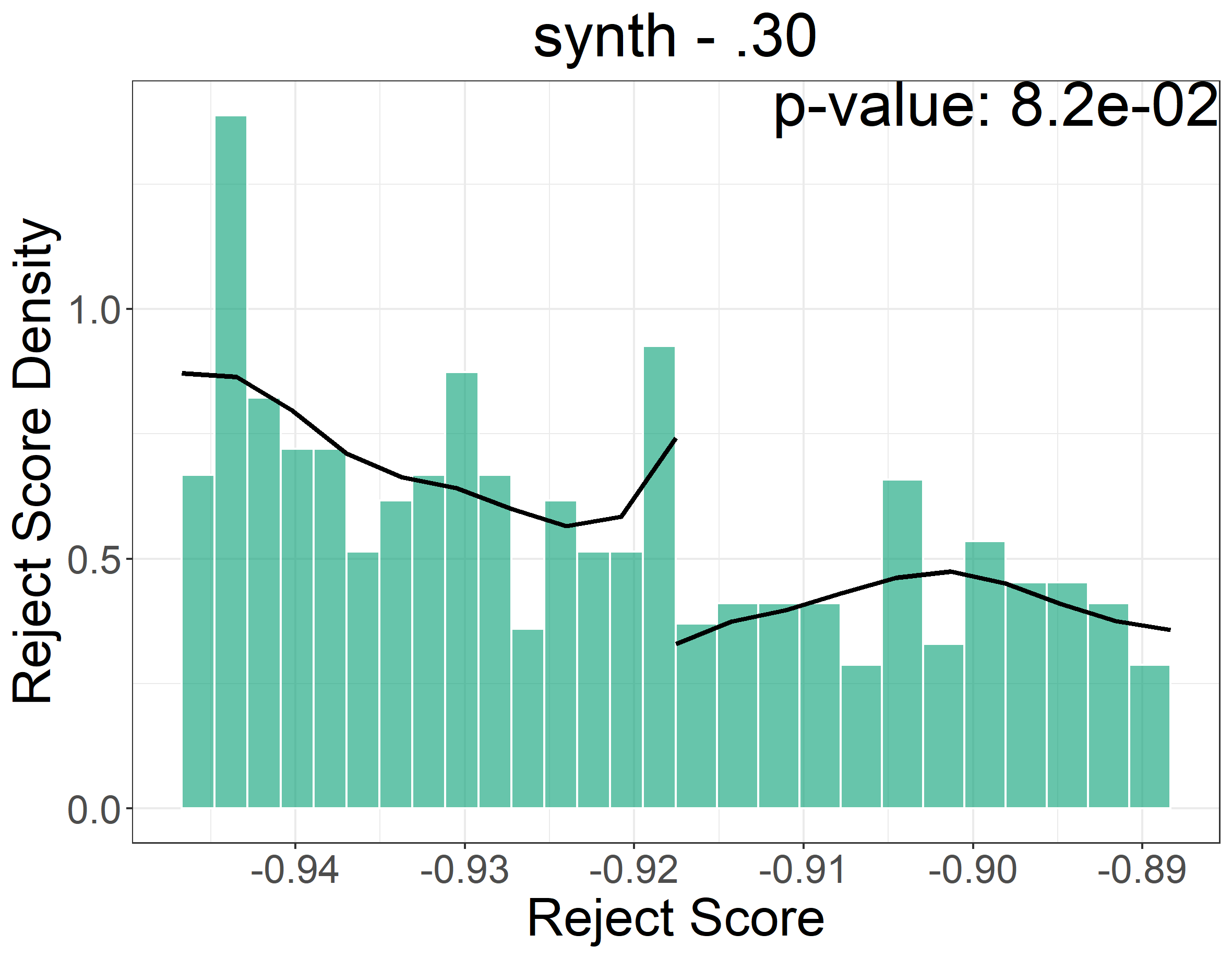}
        \caption{$\ok_{.30}$.}
        \label{fig:synth_best_density_cutoff.30}
    \end{subfigure}
    \hfill
    \begin{subfigure}[t]{.32\textwidth}
        \includegraphics[scale=.22]{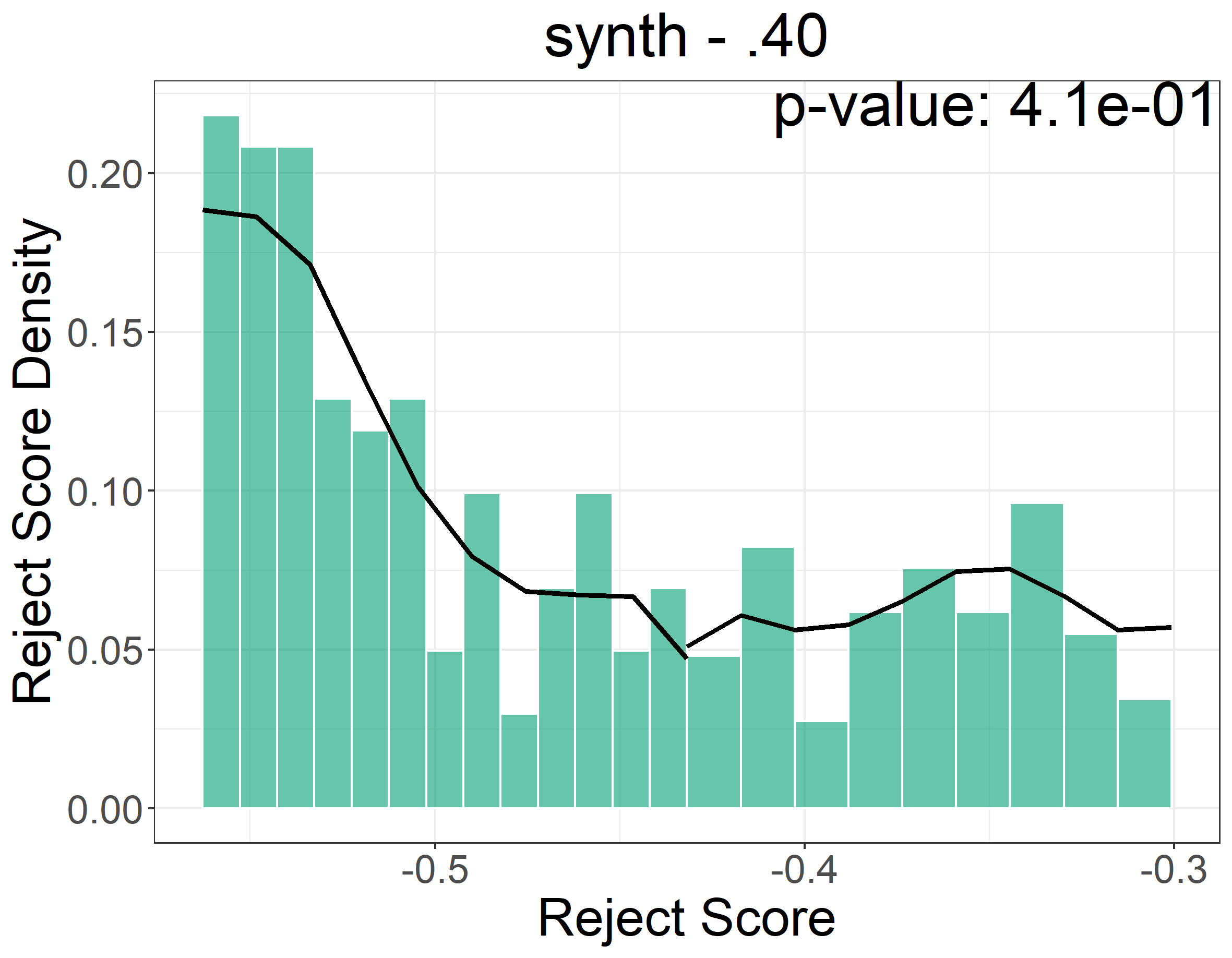}
        \caption{$\ok_{.40}$.}
        \label{fig:synth_best_density_cutoff.40}
    \end{subfigure}
    \hfill
    \begin{subfigure}[t]{.32\textwidth}
        \includegraphics[scale=.22]{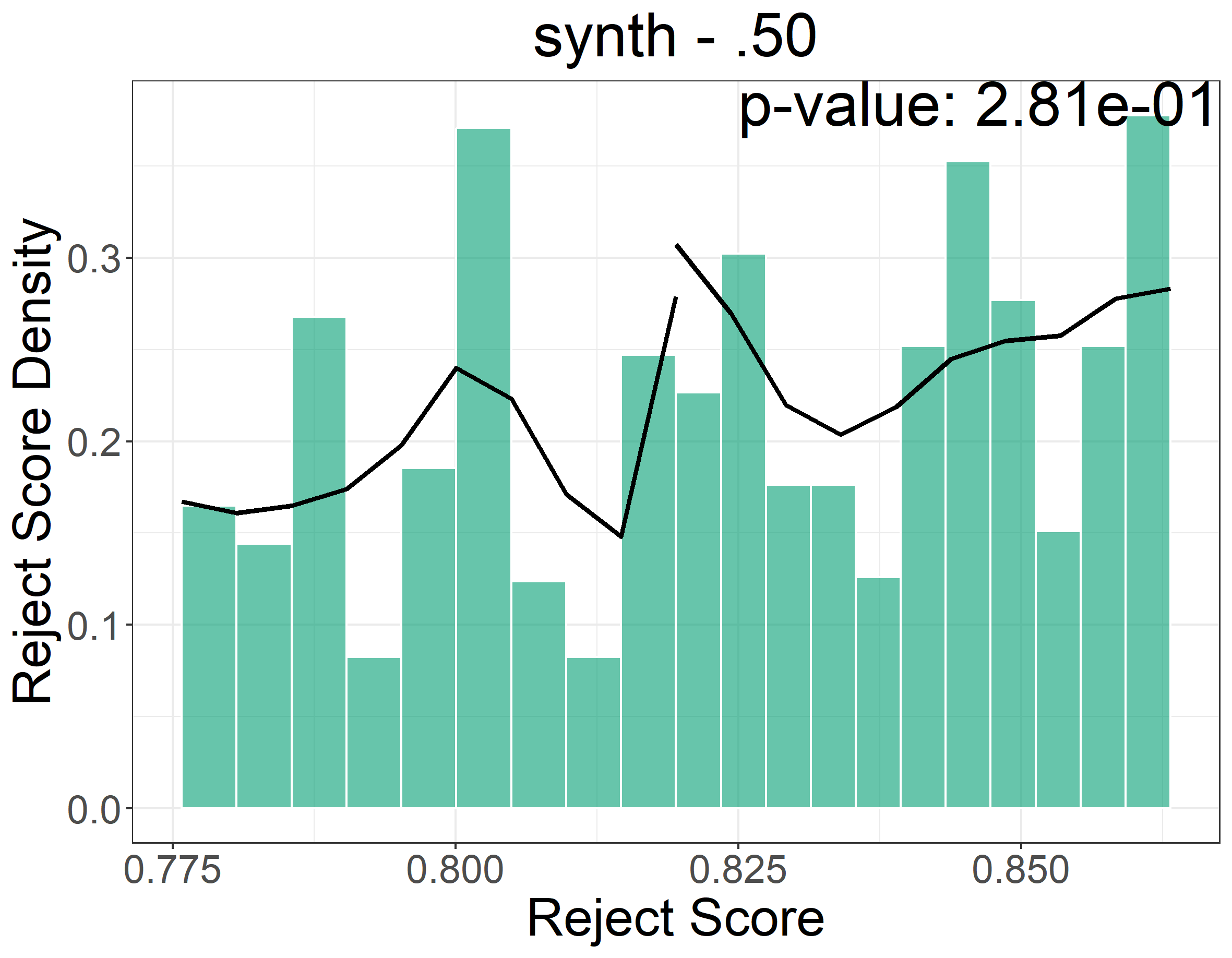}
        \caption{$\ok_{.50}$.}
        \label{fig:synth_best_density_cutoff.50}
    \end{subfigure}
    \hfill
    \begin{subfigure}[t]{.32\textwidth}
        \includegraphics[scale=.22]{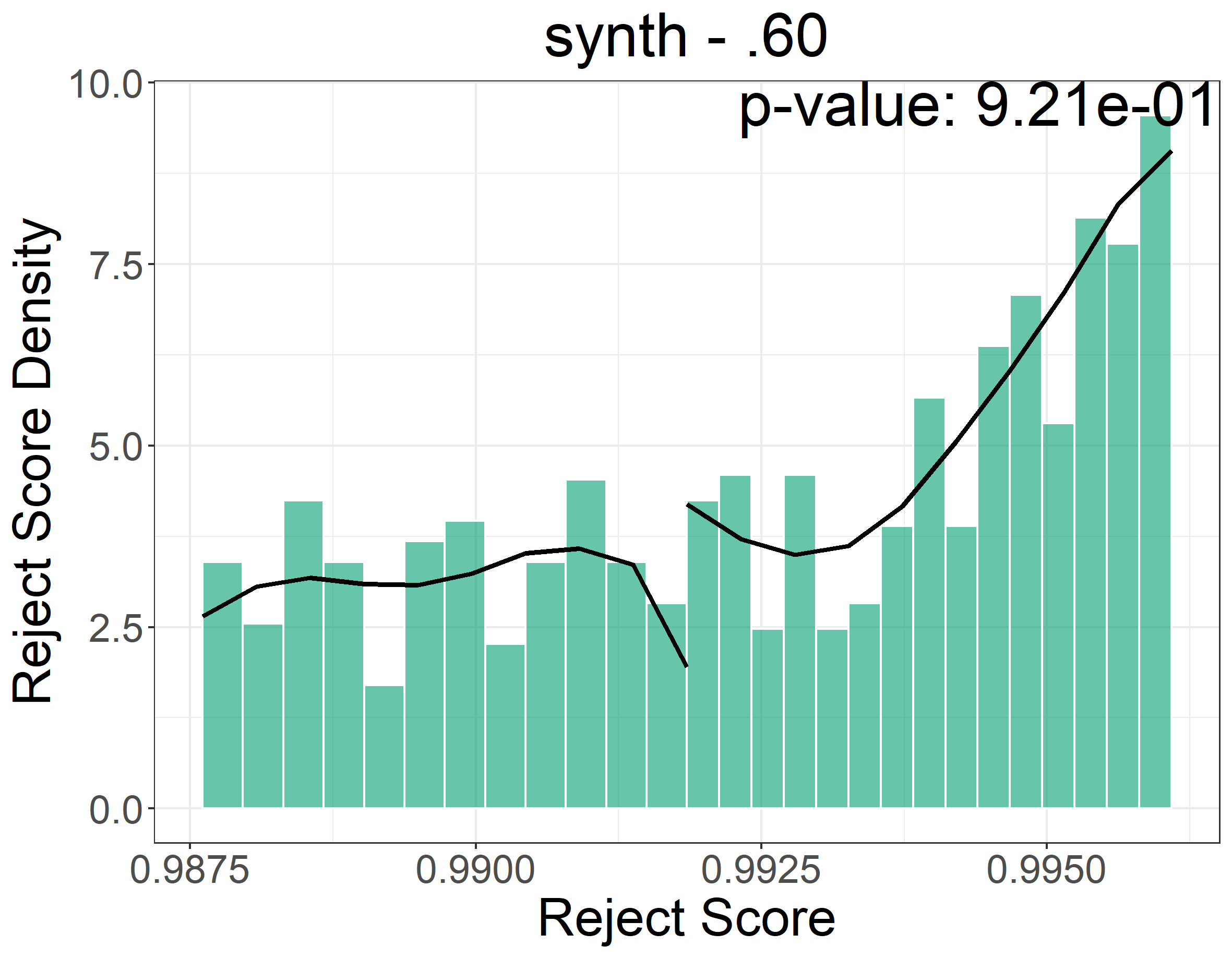}
        \caption{$\ok_{.60}$.}
        \label{fig:synth_best_density_cutoff.60}
    \end{subfigure}
    \hfill
    \begin{subfigure}[t]{.32\textwidth}
        \includegraphics[scale=.22]{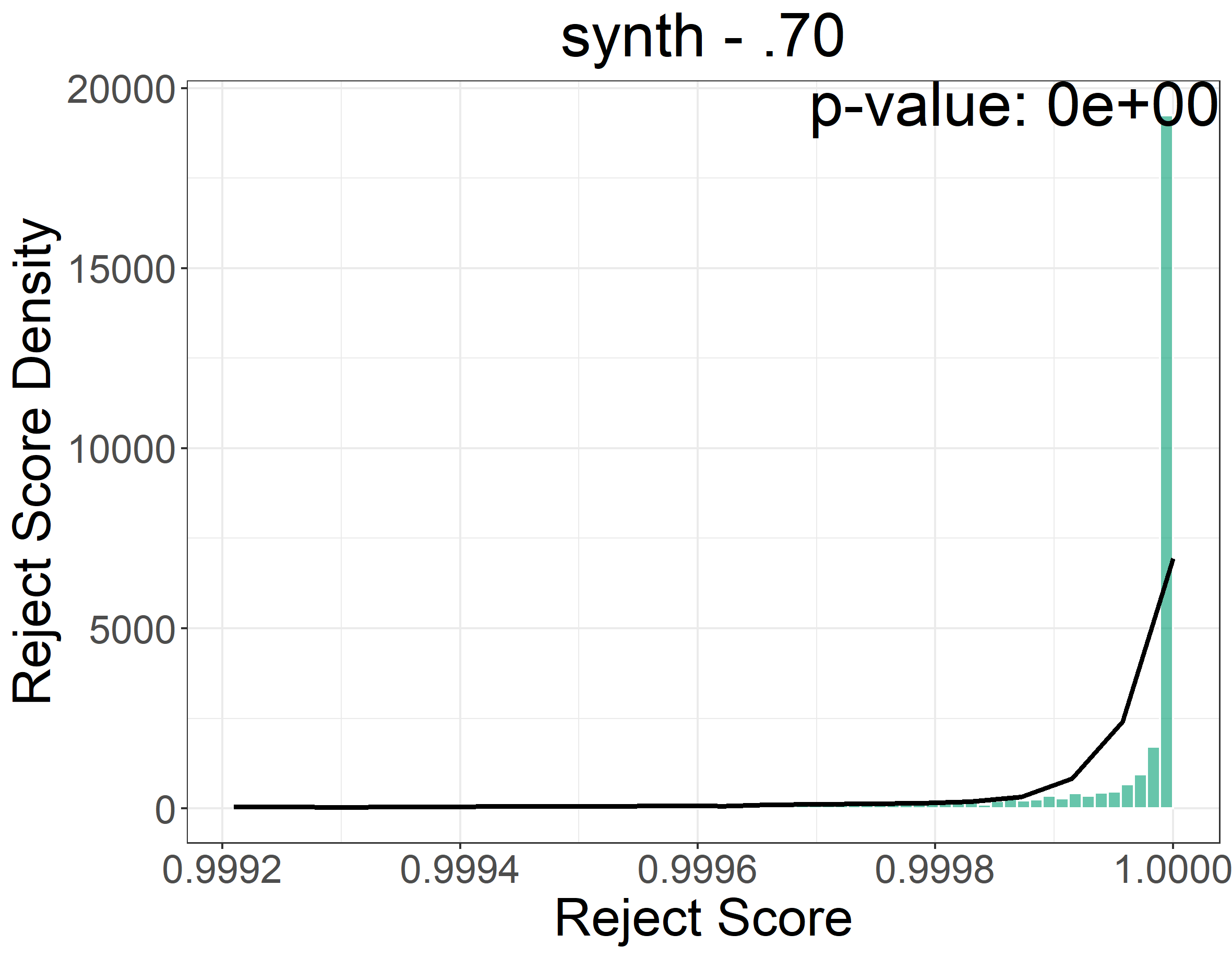}
        \caption{$\ok_{.70}$.}
        \label{fig:synth_best_density_cutoff.70}
    \end{subfigure}
    \hfill
    \begin{subfigure}[t]{.32\textwidth}
        \includegraphics[scale=.22]{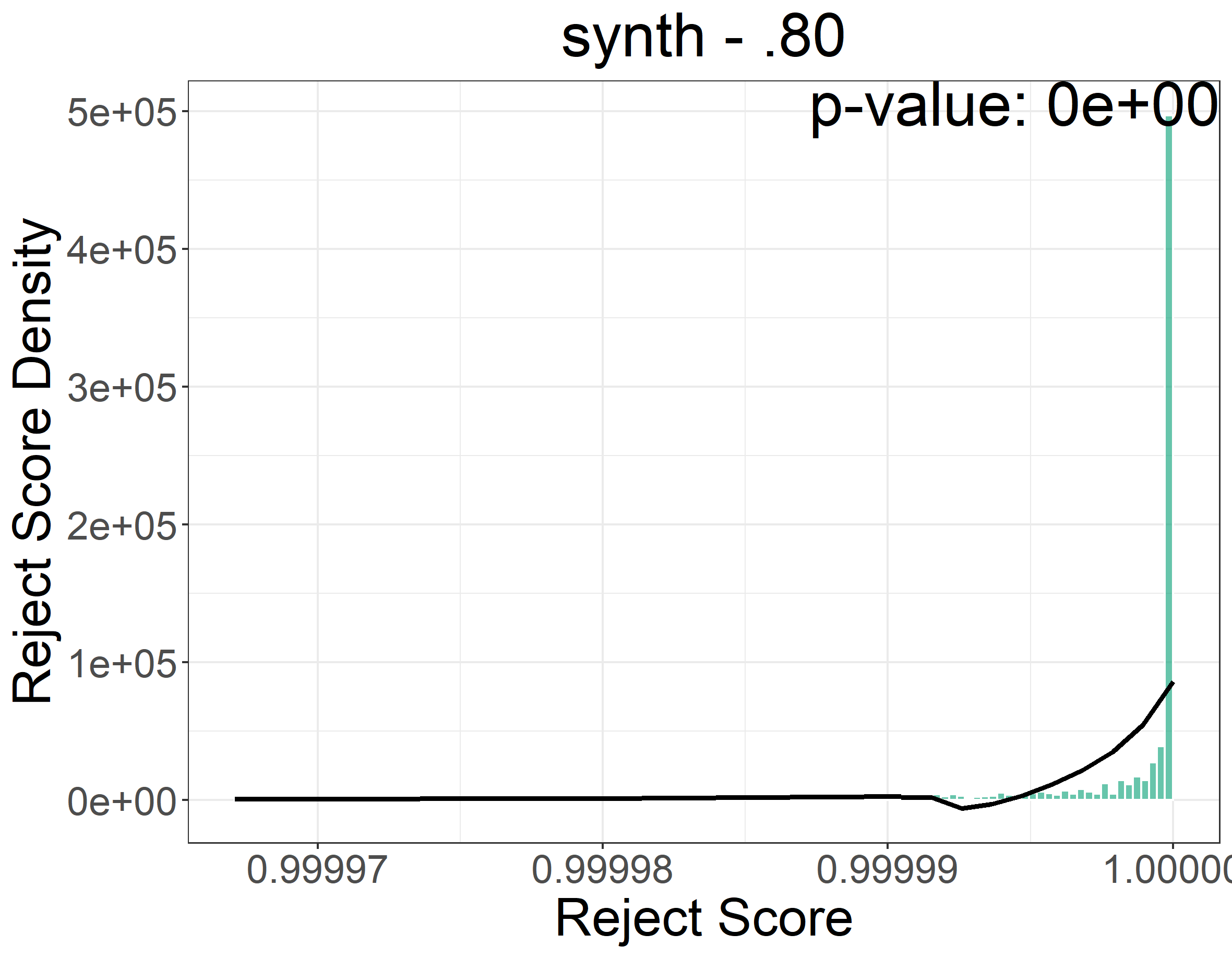}
        \caption{$\ok_{.80}$.}
        \label{fig:synth_best_density_cutoff.80}
    \end{subfigure}
    \hfill
    \begin{subfigure}[t]{.32\textwidth}
        \includegraphics[scale=.22]{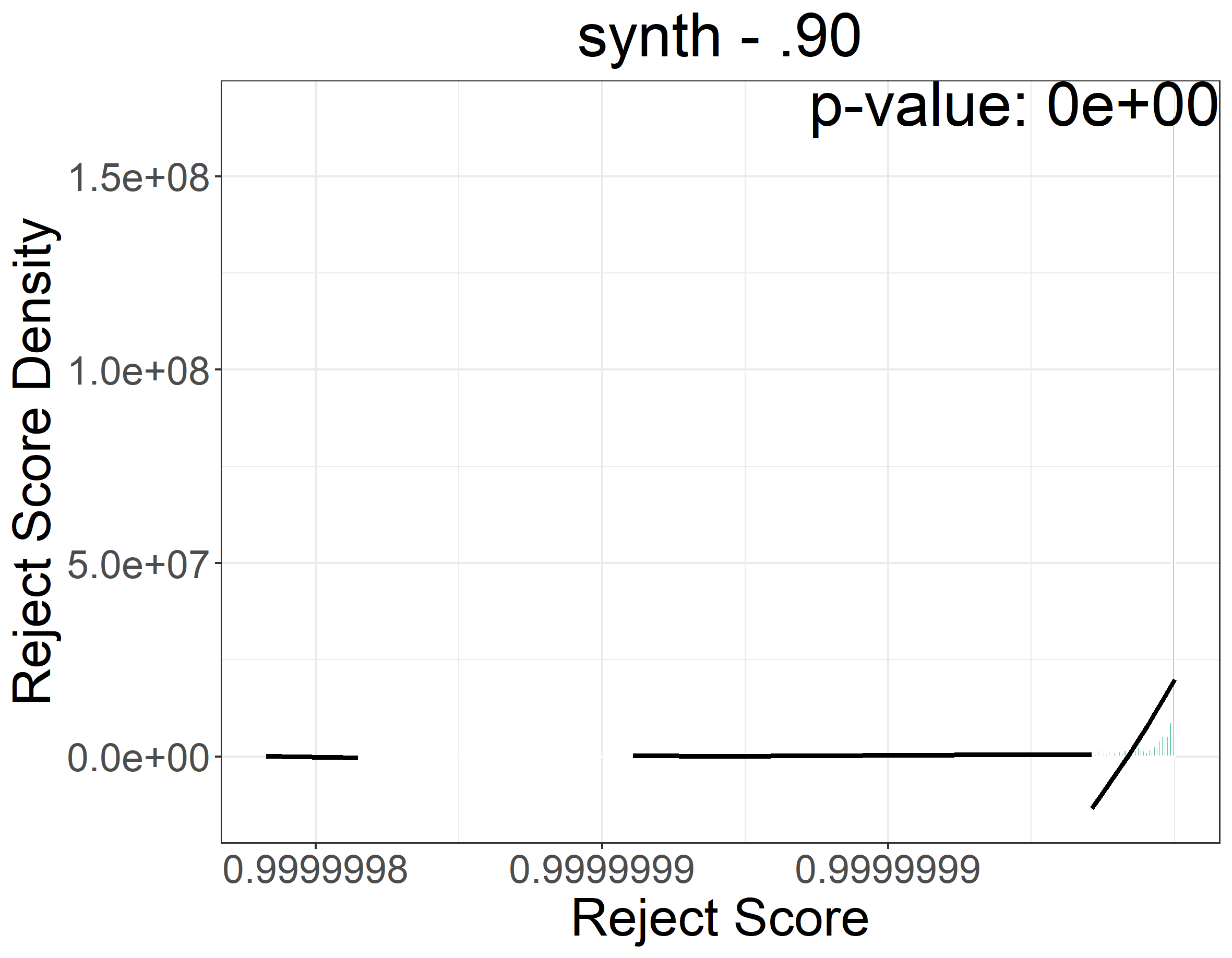}
        \caption{$\ok_{.90}$.}
        \label{fig:synth_best_density_cutoff.90}
    \end{subfigure}
    \hfill
    
    \caption{
    \texttt{synth} estimated best baseline (\ASM{}) reject scores densities at the left and right of cutoff $\ok_c$. All the plots are zoomed around the cutoff values.
    }
    \label{fig:synth estimated density best}
\end{figure*}
}

\afterpage{
\begin{figure*}[t!]
    \begin{subfigure}[t]{.32\textwidth}
        \includegraphics[scale=.22]{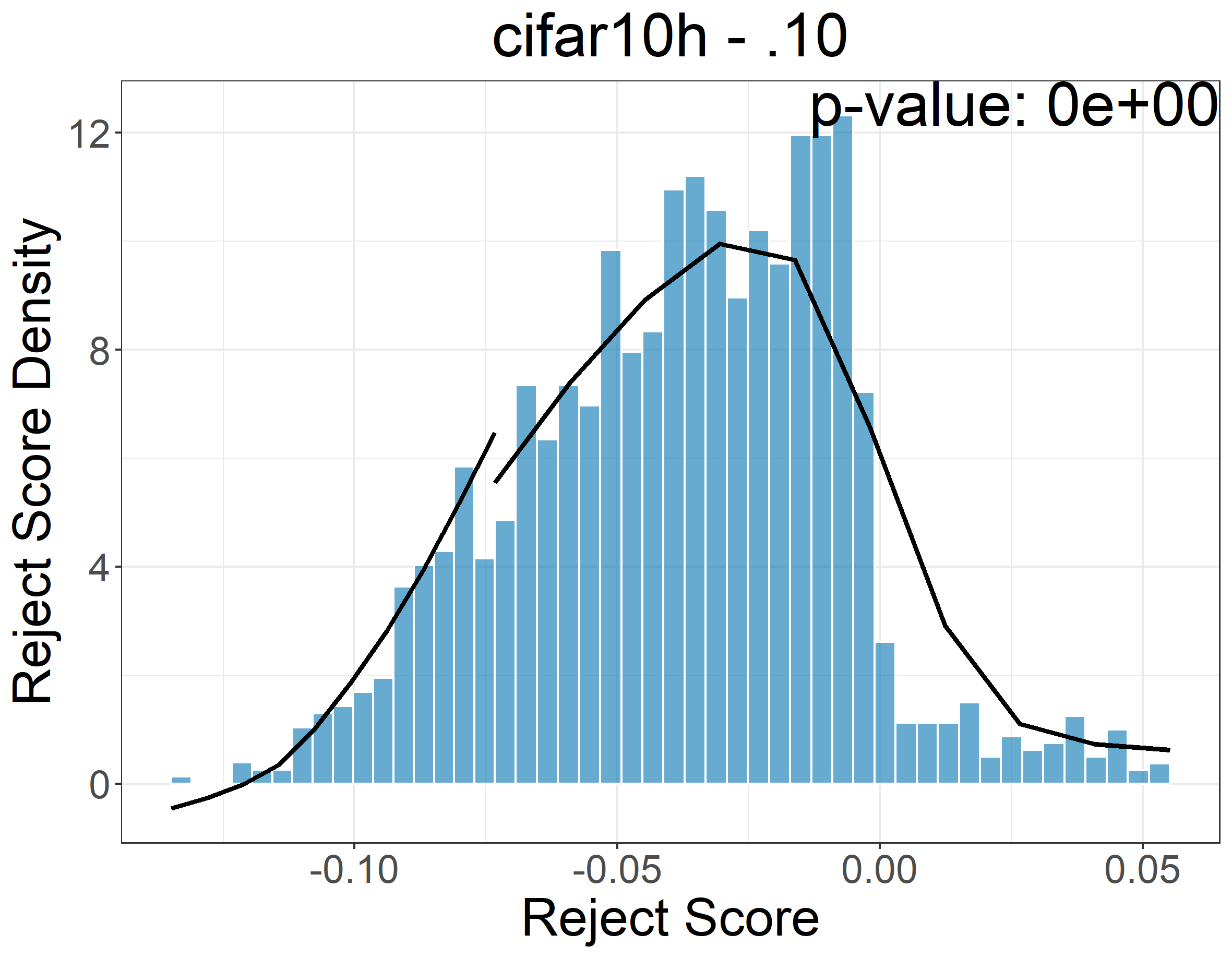}
        \caption{$\ok_{.10}$.}
        \label{fig:cifar10h_best_density_cutoff.10}
    \end{subfigure}
    \hfill
    \begin{subfigure}[t]{.32\textwidth}
        \includegraphics[scale=.22]{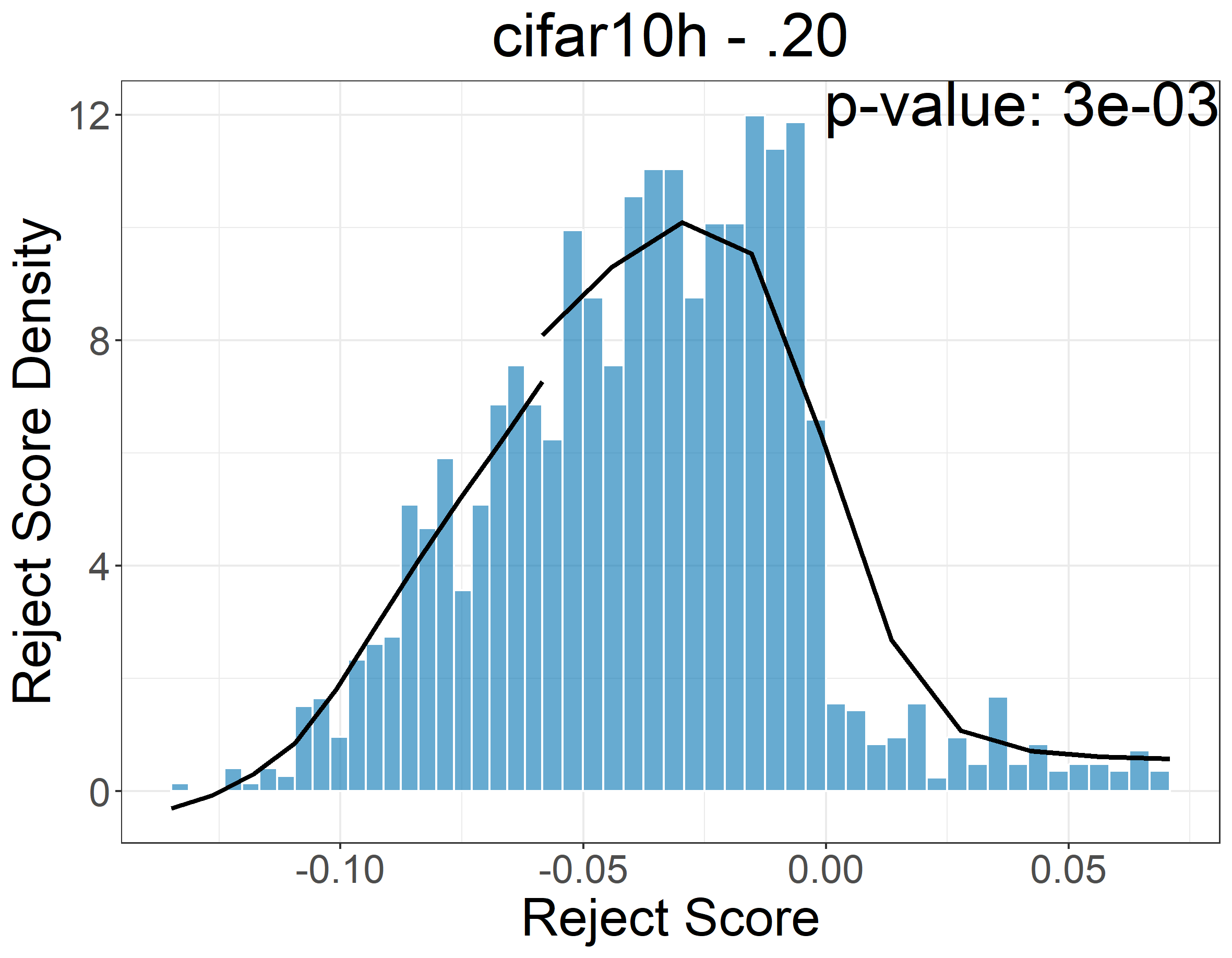}
        \caption{$\ok_{.20}$.}
        \label{fig:cifar10h_best_density_cutoff.20}
    \end{subfigure}
    \hfill
    \begin{subfigure}[t]{.32\textwidth}
        \includegraphics[scale=.22]{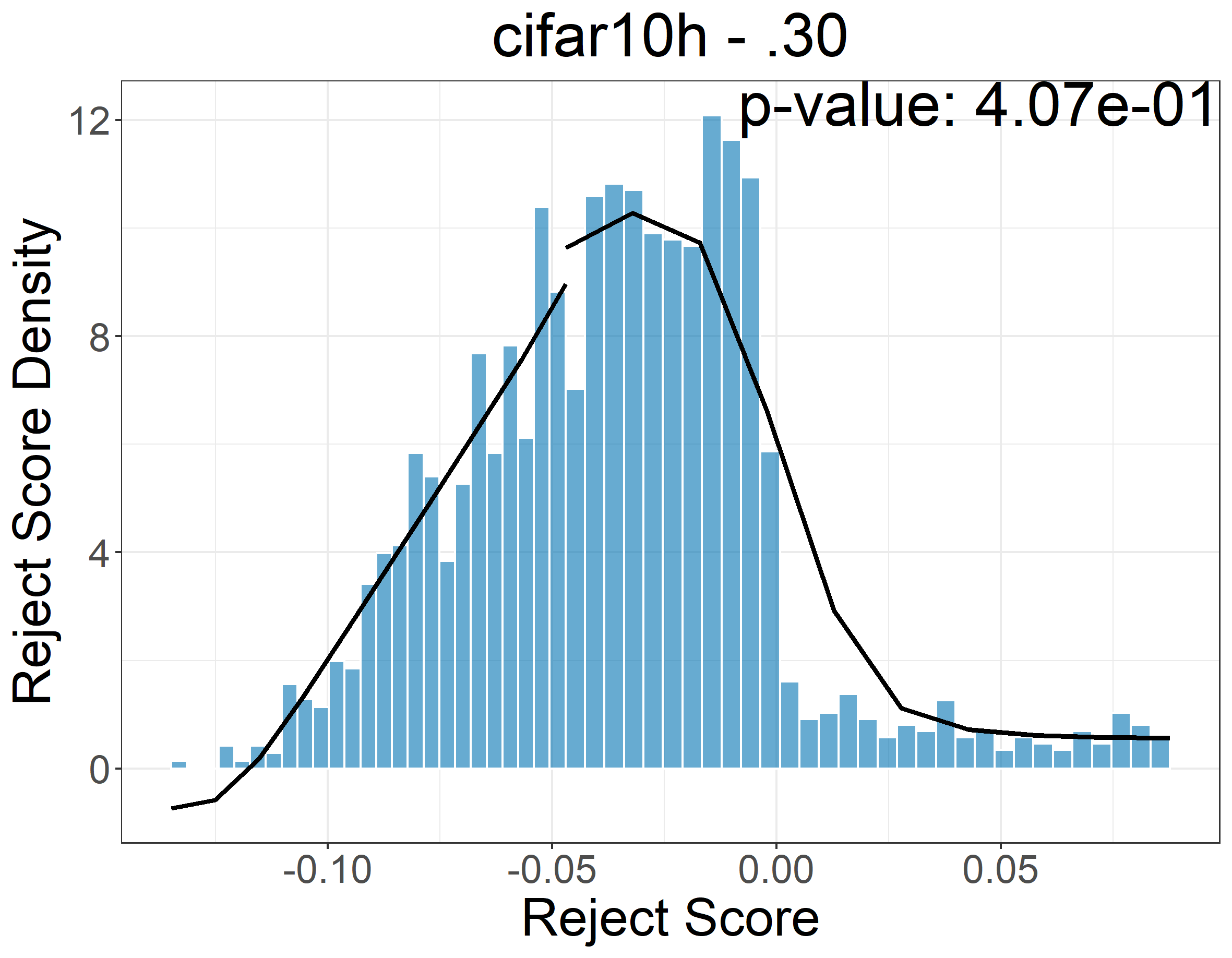}
        \caption{$\ok_{.30}$.}
        \label{fig:cifar10h_best_density_cutoff.30}
    \end{subfigure}
    \hfill
    \begin{subfigure}[t]{.32\textwidth}
        \includegraphics[scale=.22]{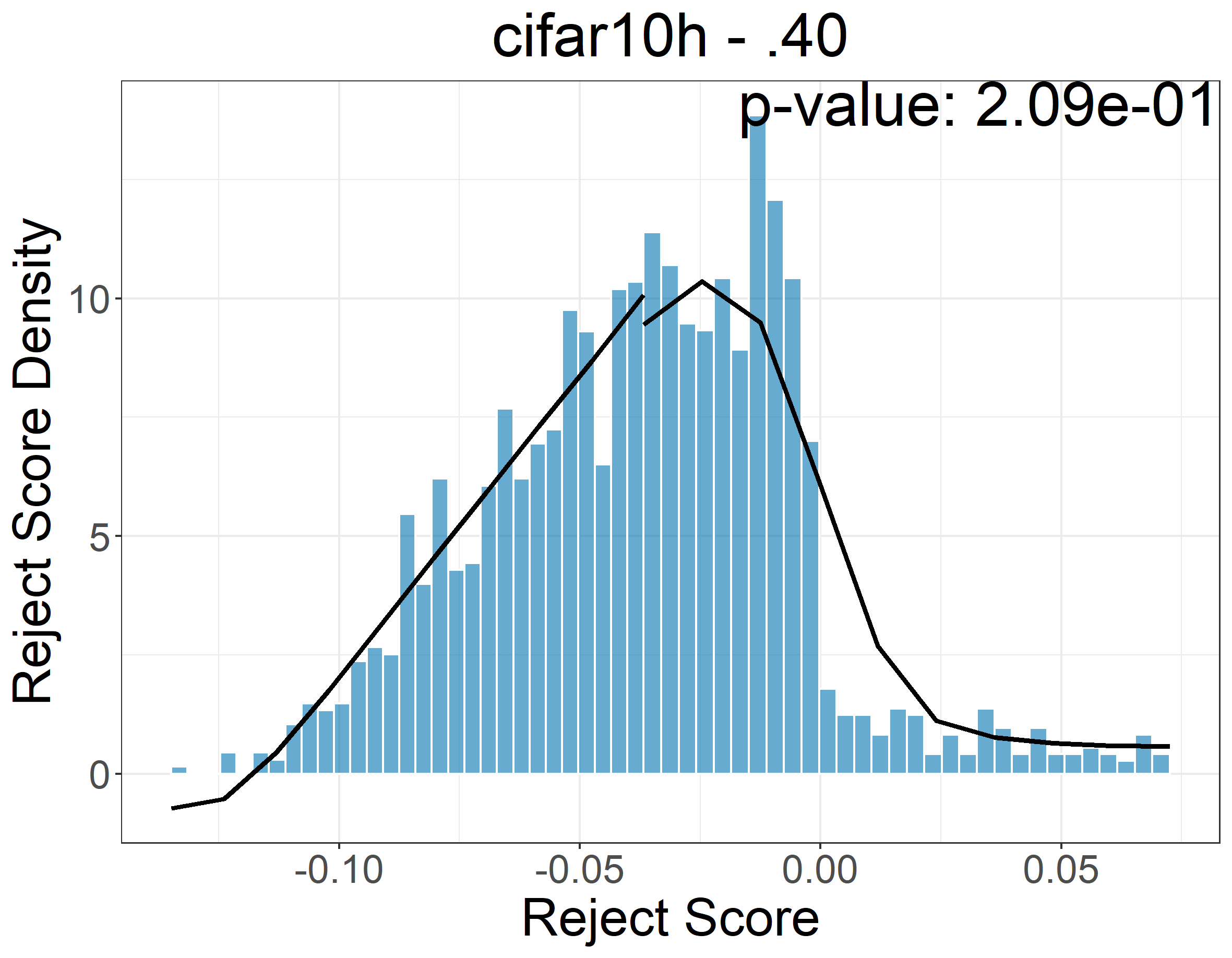}
        \caption{$\ok_{.40}$.}
        \label{fig:cifar10h_best_density_cutoff.40}
    \end{subfigure}
    \hfill
    \begin{subfigure}[t]{.32\textwidth}
        \includegraphics[scale=.22]{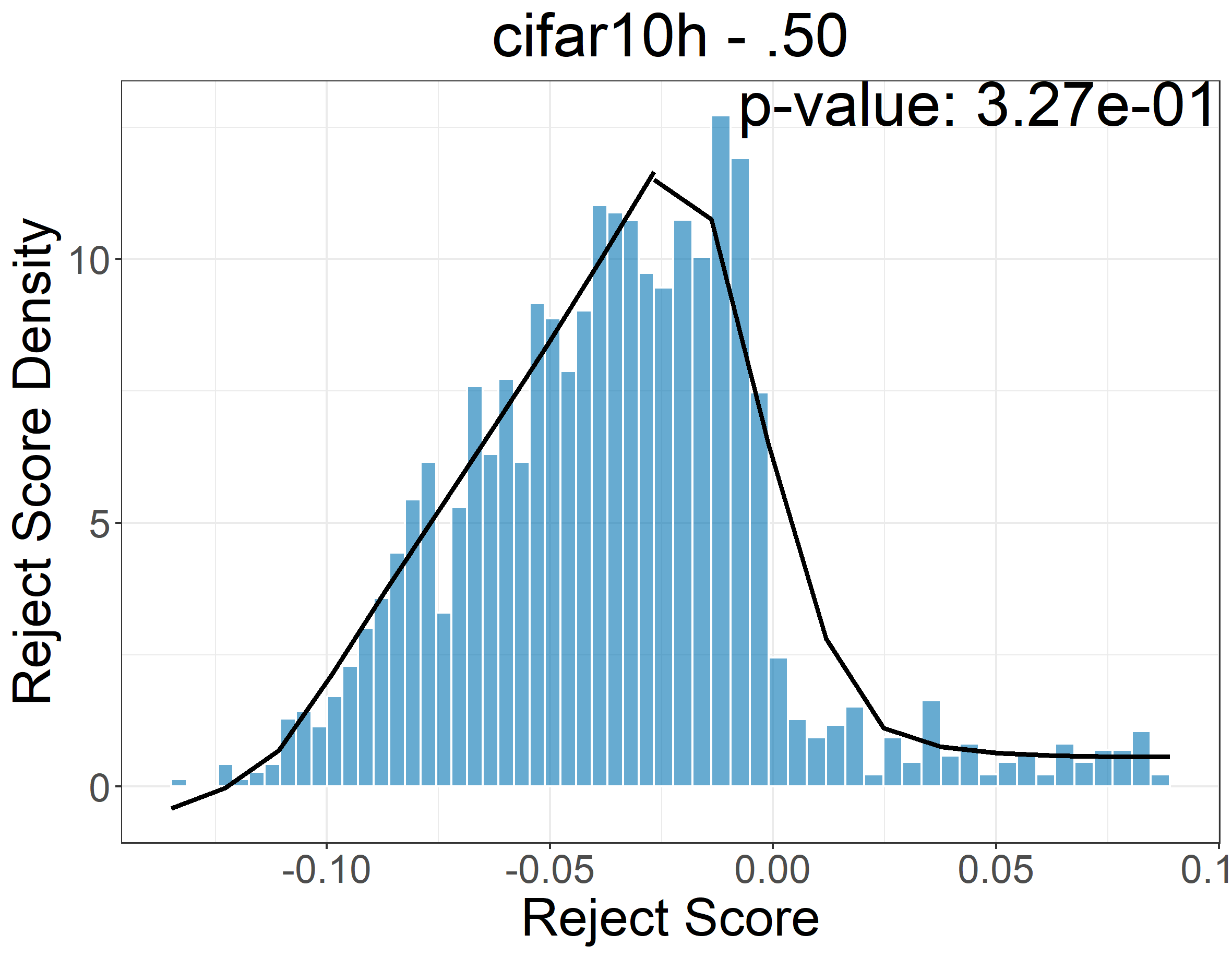}
        \caption{$\ok_{.50}$.}
        \label{fig:cifar10h_best_density_cutoff.50}
    \end{subfigure}
    \hfill
    \begin{subfigure}[t]{.32\textwidth}
        \includegraphics[scale=.22]{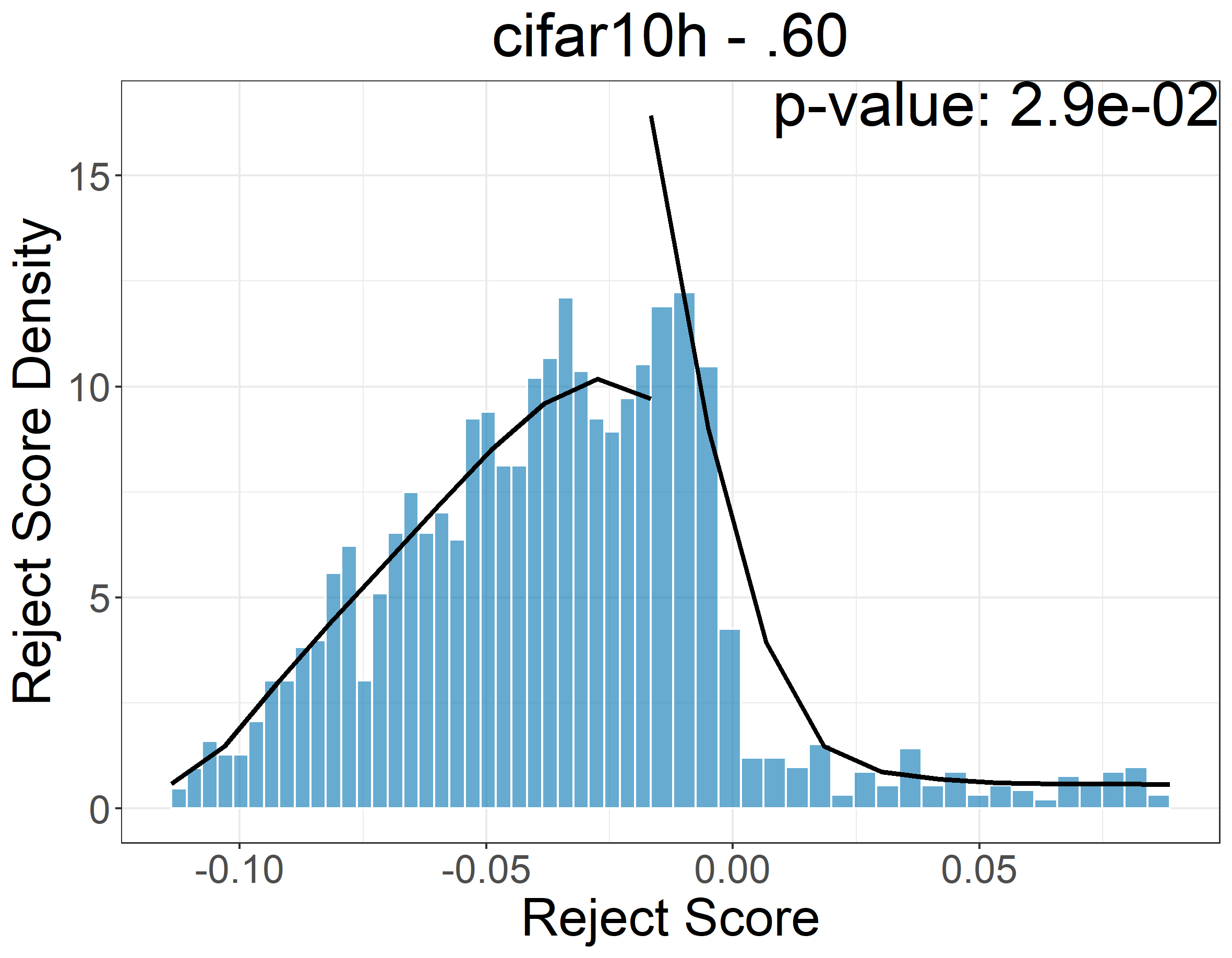}
        \caption{$\ok_{.60}$.}
        \label{fig:cifar10h_best_density_cutoff.60}
    \end{subfigure}
    \hfill
    \begin{subfigure}[t]{.32\textwidth}
        \includegraphics[scale=.22]{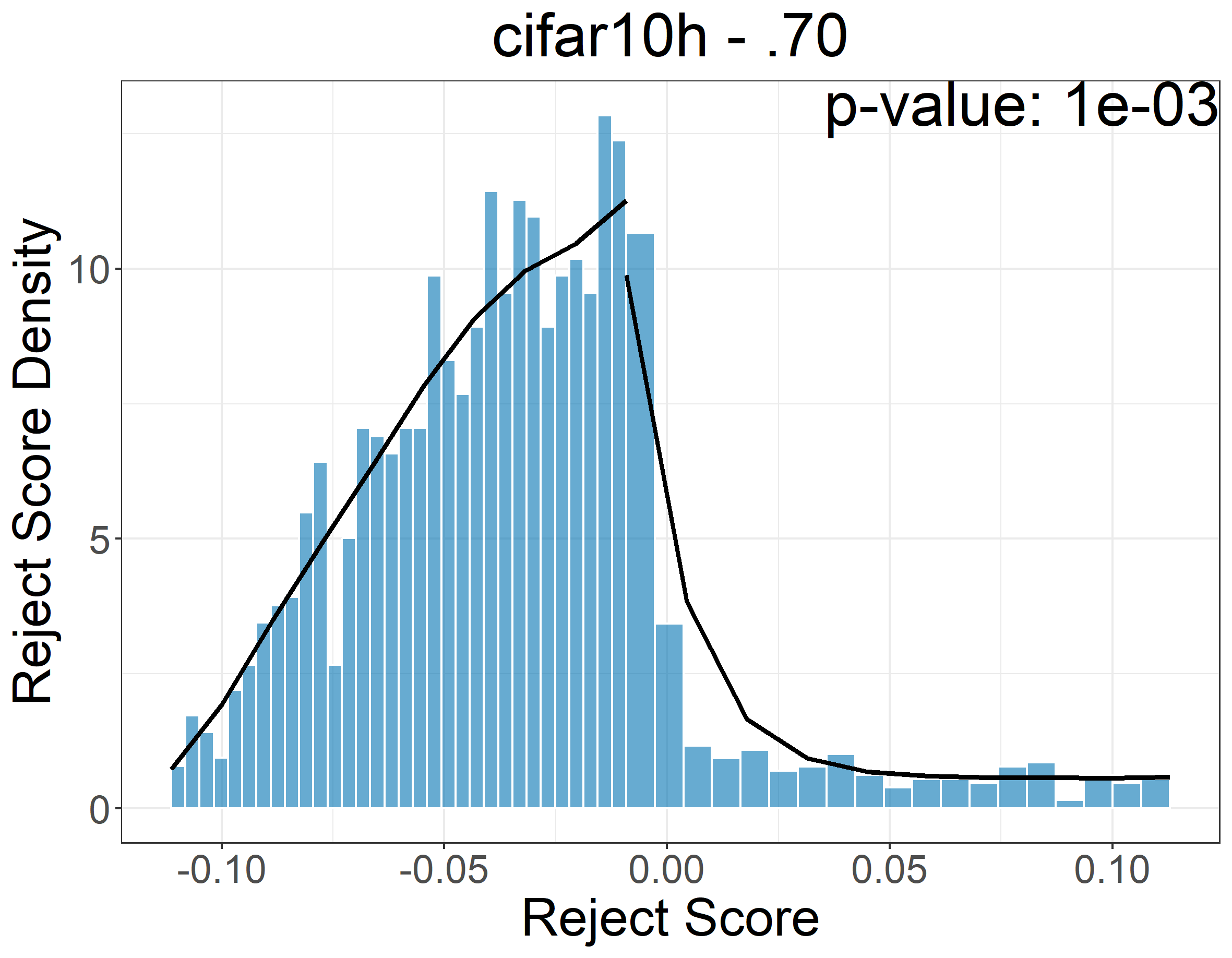}
        \caption{$\ok_{.70}$.}
        \label{fig:cifar10h_best_density_cutoff.70}
    \end{subfigure}
    \hfill
    \begin{subfigure}[t]{.32\textwidth}
        \includegraphics[scale=.22]{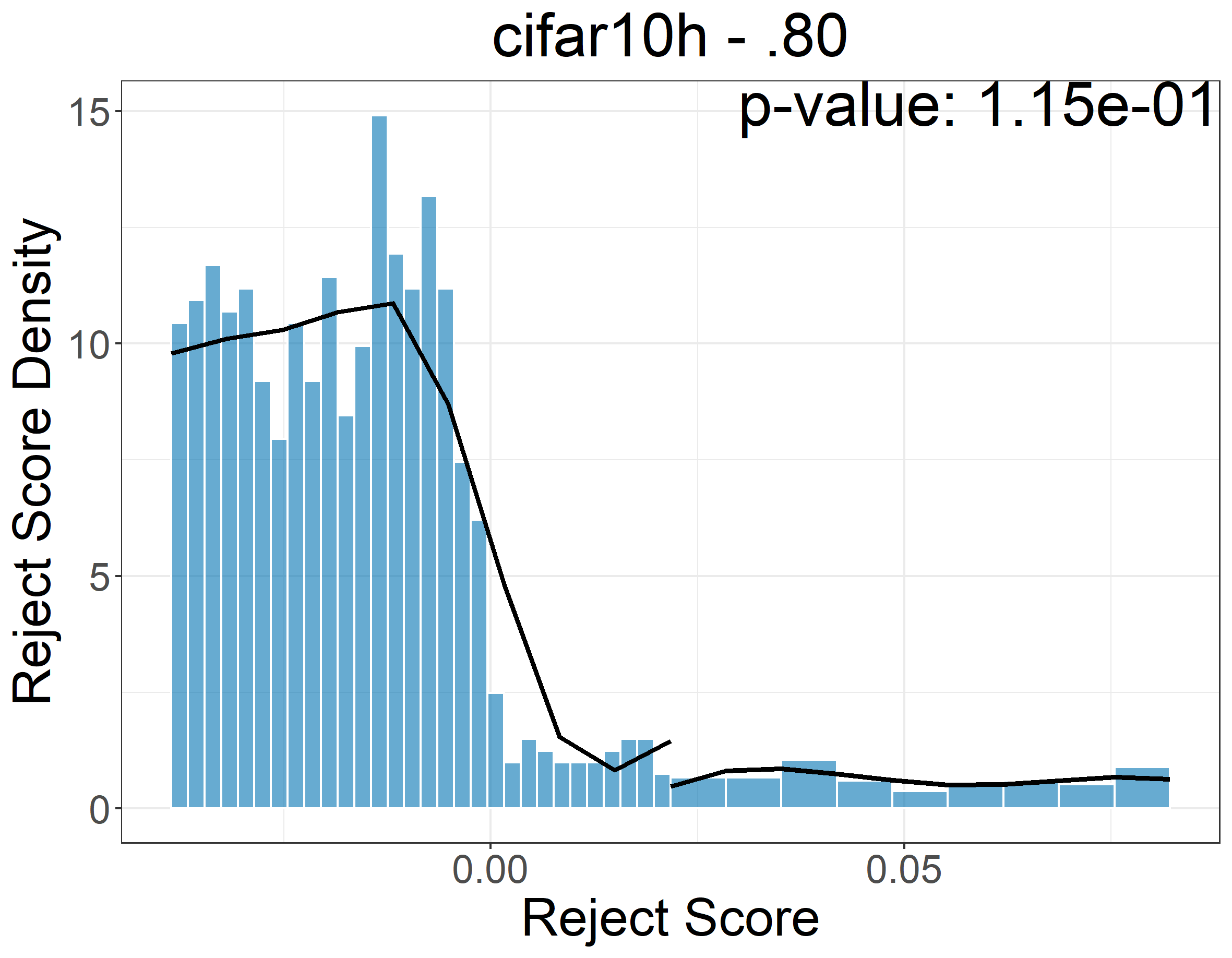}
        \caption{$\ok_{.80}$.}
        \label{fig:cifar10h_best_density_cutoff.80}
    \end{subfigure}
    \hfill
    \begin{subfigure}[t]{.32\textwidth}
        \includegraphics[scale=.22]{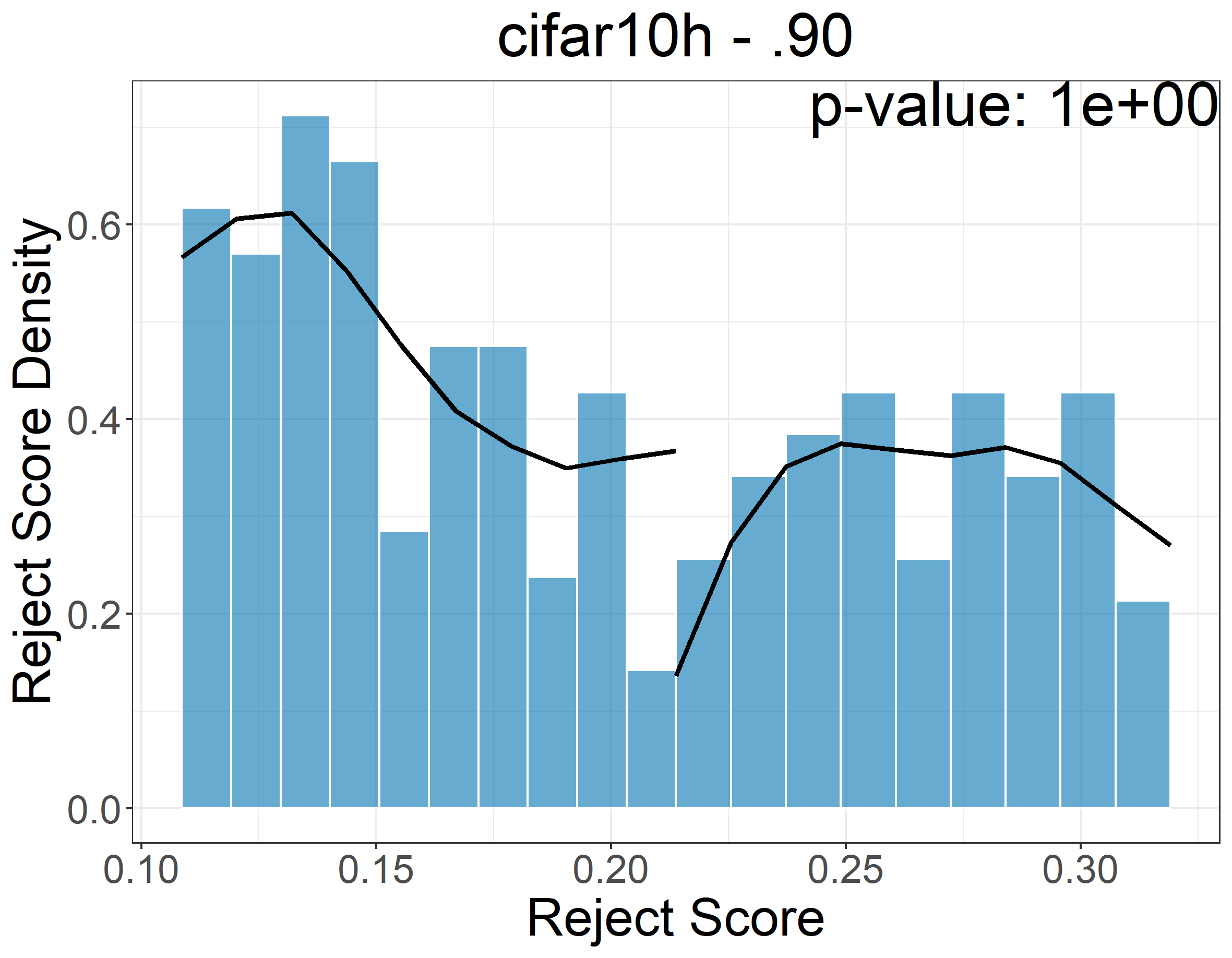}
        \caption{$\ok_{.90}$.}
        \label{fig:cifar10h_best_density_cutoff.90}
    \end{subfigure}
    \hfill
    
    \caption{
    \texttt{cifar10h} estimated best baseline (\CC{}) reject scores densities at the left and right of cutoff $\ok_c$. All the plots are zoomed around the cutoff values.
    }
    \label{fig:cifar10h estimated density best}
\end{figure*}
}
\afterpage{
\begin{figure*}[t!]
    \begin{subfigure}[t]{.32\textwidth}
        \includegraphics[scale=.22]{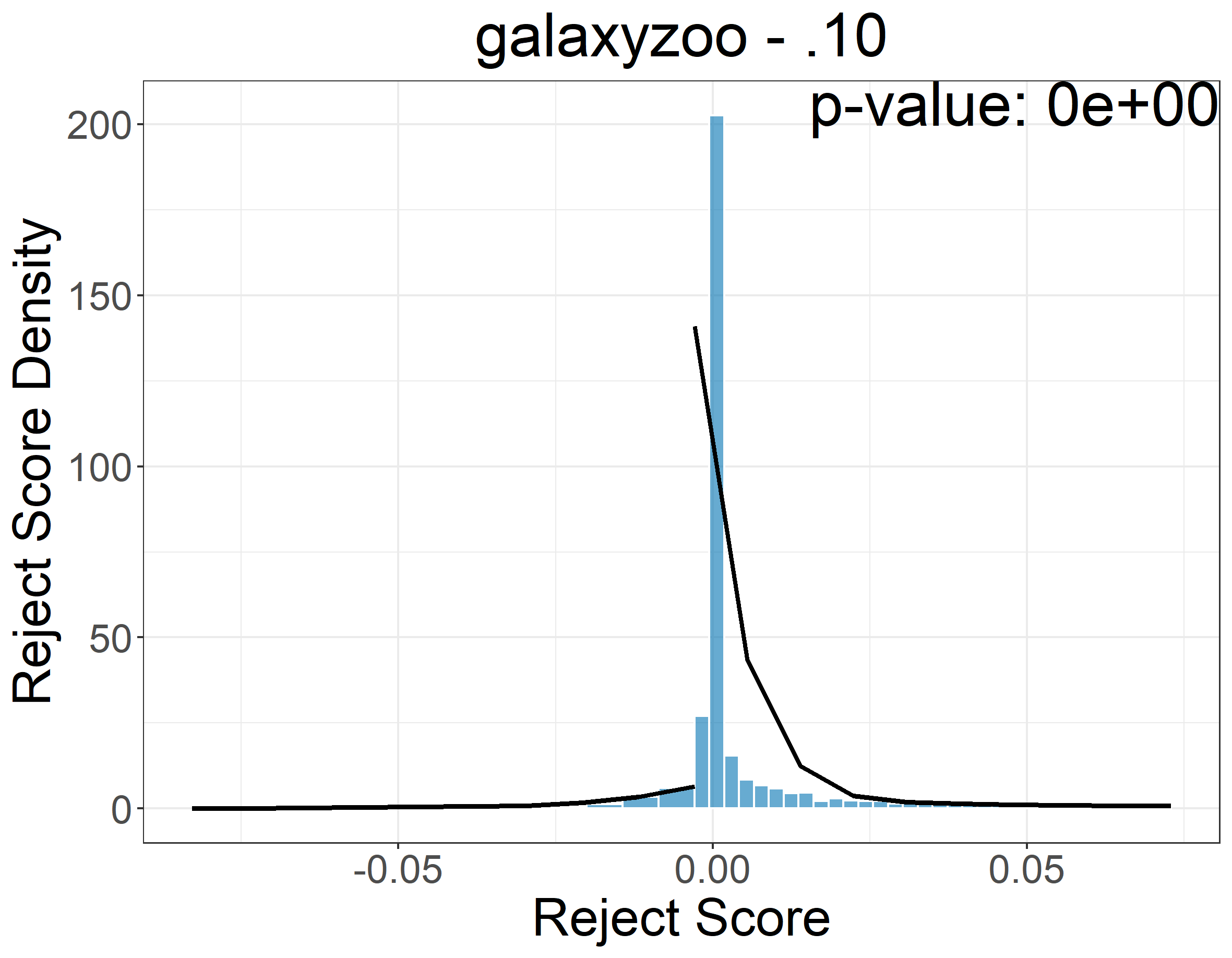}
        \caption{$\ok_{.10}$.}
        \label{fig:galaxyzoo_best_density_cutoff.10}
    \end{subfigure}
    \hfill
    \begin{subfigure}[t]{.32\textwidth}
        \includegraphics[scale=.22]{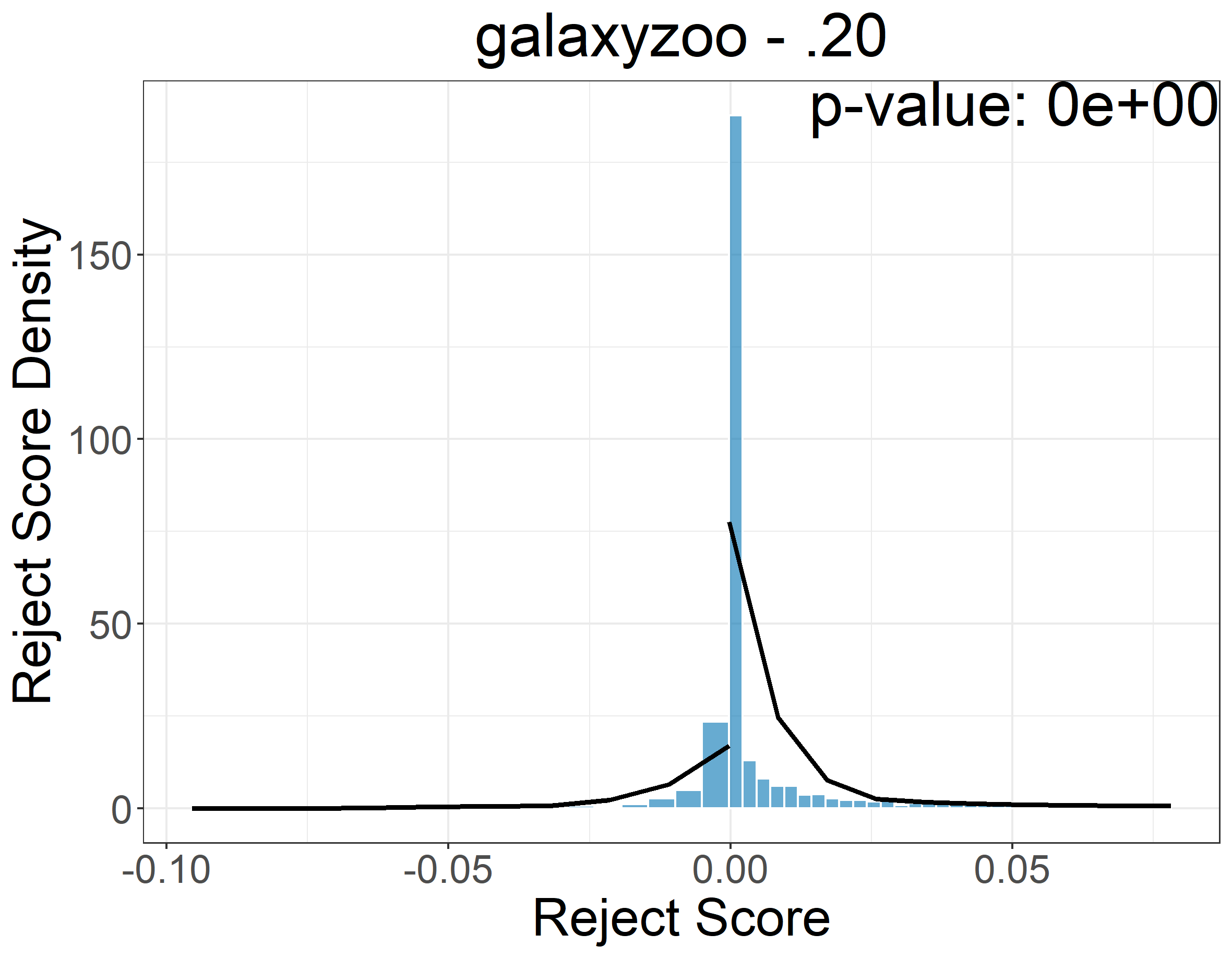}
        \caption{$\ok_{.20}$.}
        \label{fig:galaxyzoo_best_density_cutoff.20}
    \end{subfigure}
    \hfill
    \begin{subfigure}[t]{.32\textwidth}
        \includegraphics[scale=.22]{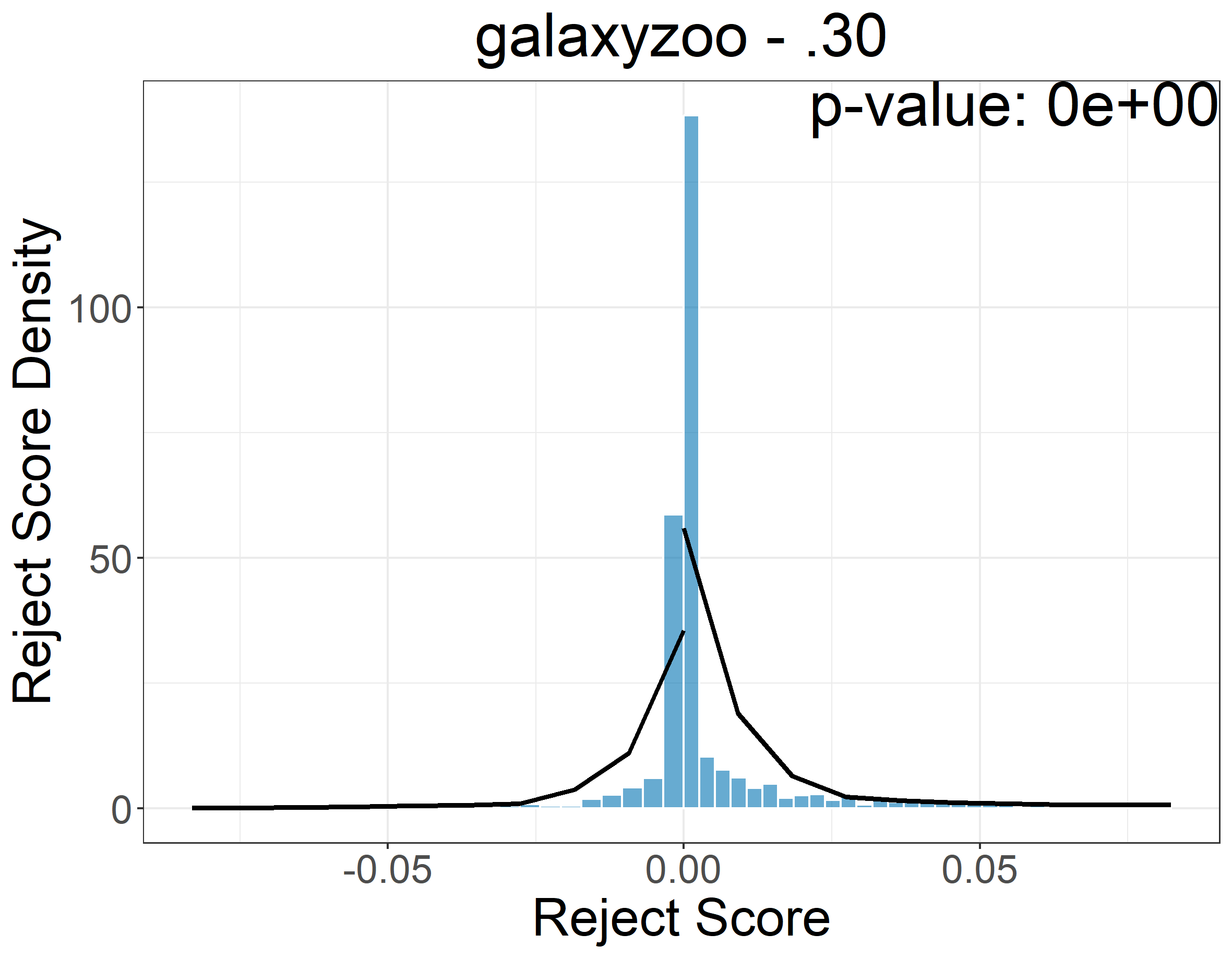}
        \caption{$\ok_{.30}$.}
        \label{fig:galaxyzoo_best_density_cutoff.30}
    \end{subfigure}
    \hfill
    \begin{subfigure}[t]{.32\textwidth}
        \includegraphics[scale=.22]{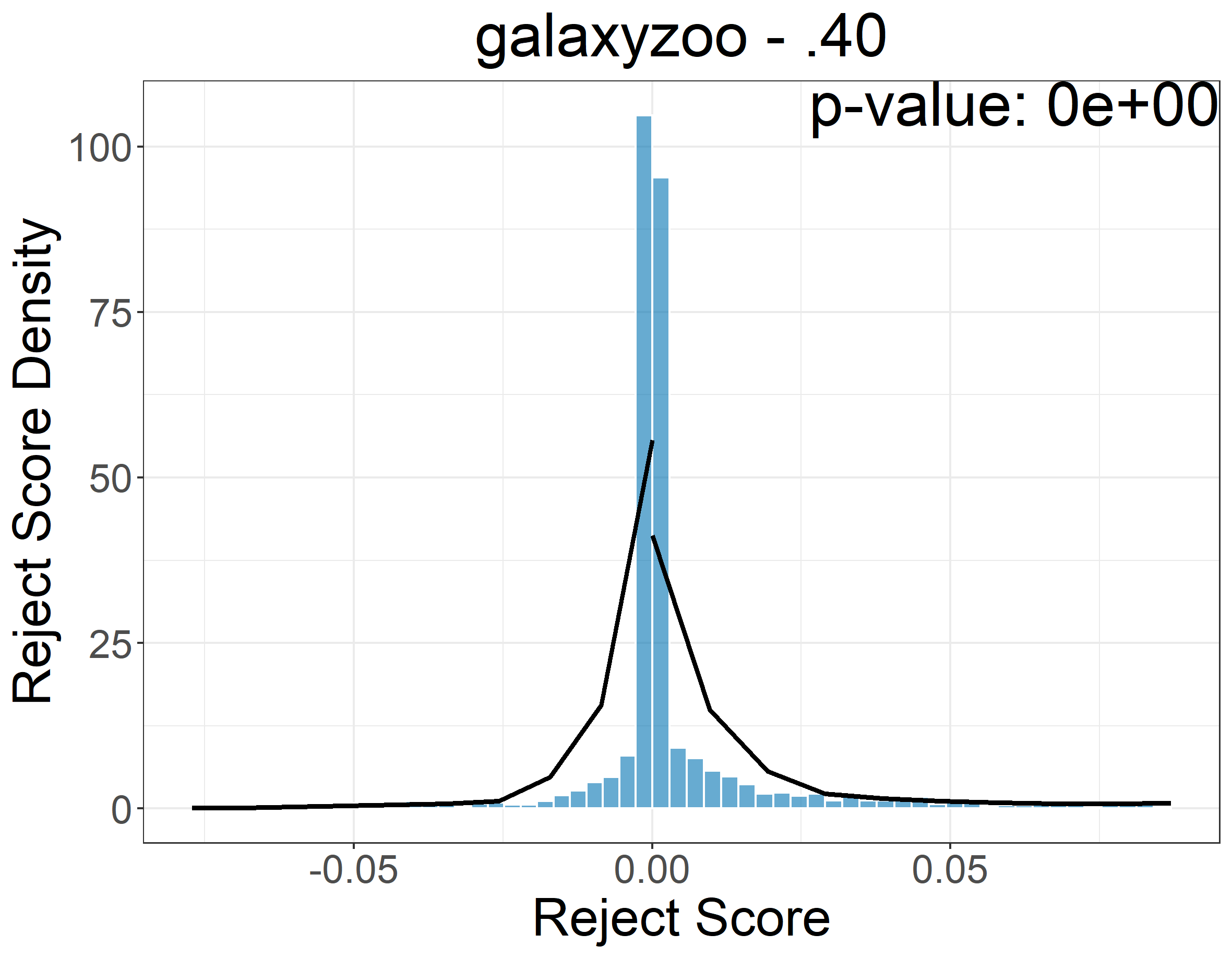}
        \caption{$\ok_{.40}$.}
        \label{fig:galaxyzoo_best_density_cutoff.40}
    \end{subfigure}
    \hfill
    \begin{subfigure}[t]{.32\textwidth}
        \includegraphics[scale=.22]{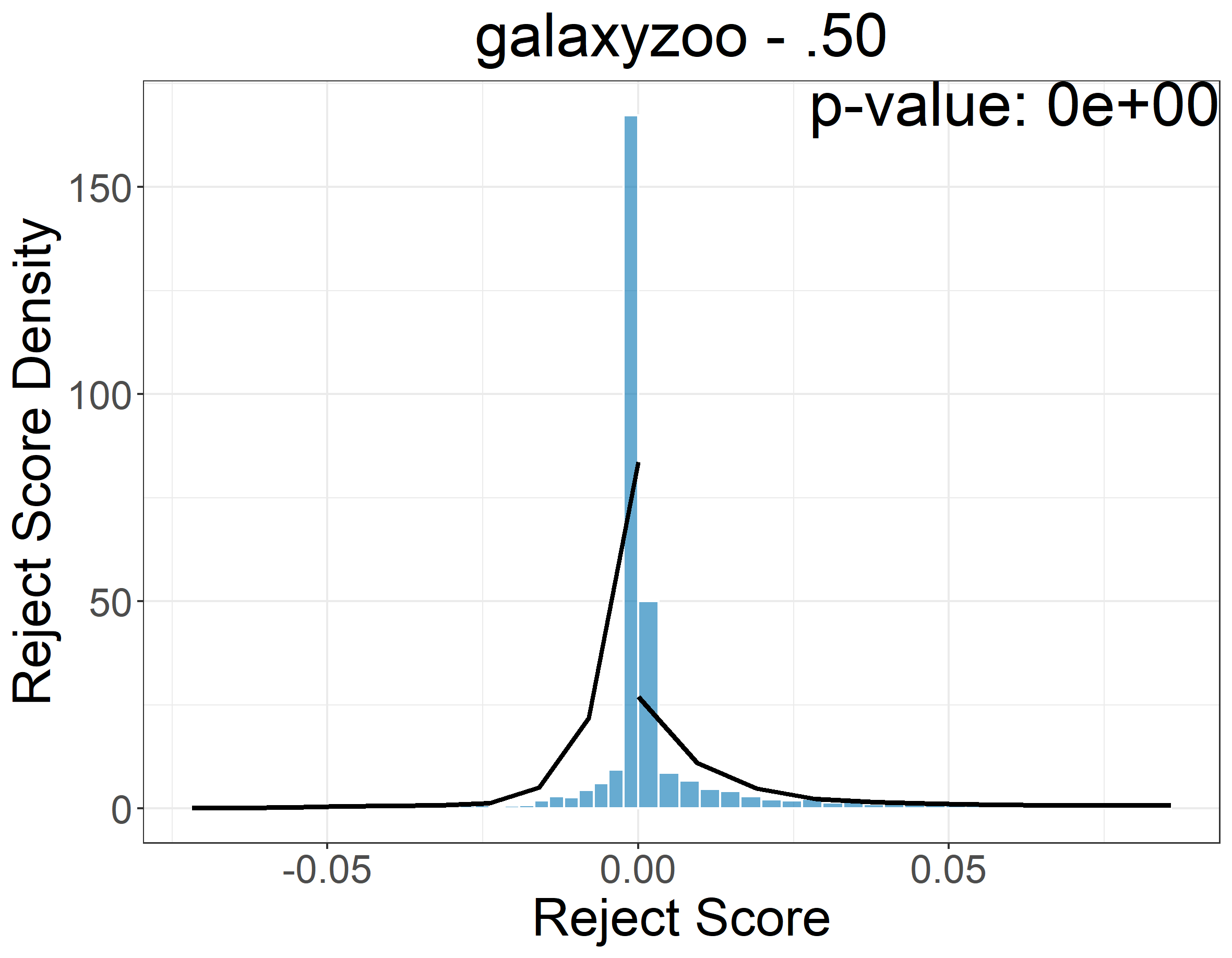}
        \caption{$\ok_{.50}$.}
        \label{fig:galaxyzoo_best_density_cutoff.50}
    \end{subfigure}
    \hfill
    \begin{subfigure}[t]{.32\textwidth}
        \includegraphics[scale=.22]{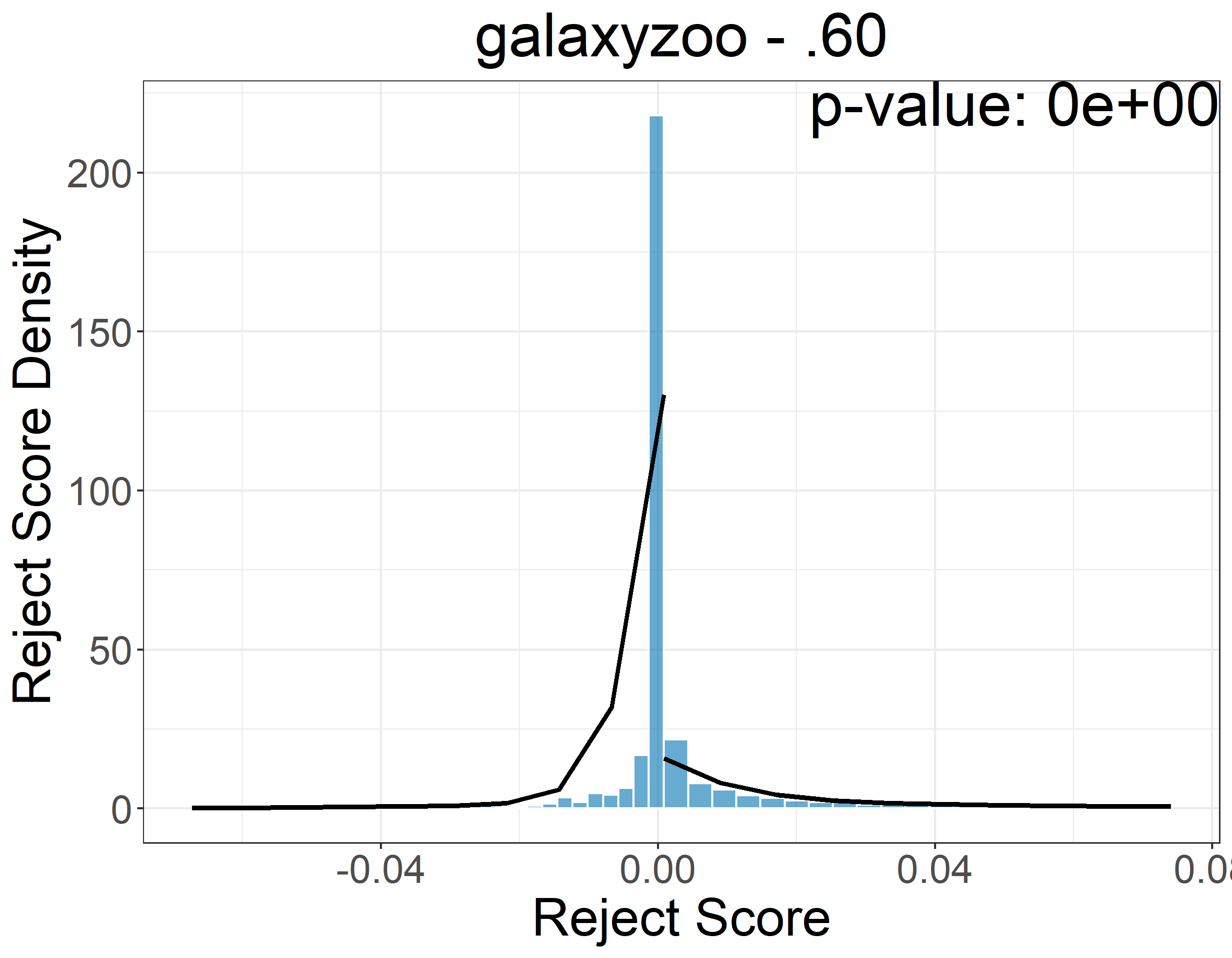}
        \caption{$\ok_{.60}$.}
        \label{fig:galaxyzoo_best_density_cutoff.60}
    \end{subfigure}
    \hfill
    \begin{subfigure}[t]{.32\textwidth}
        \includegraphics[scale=.22]{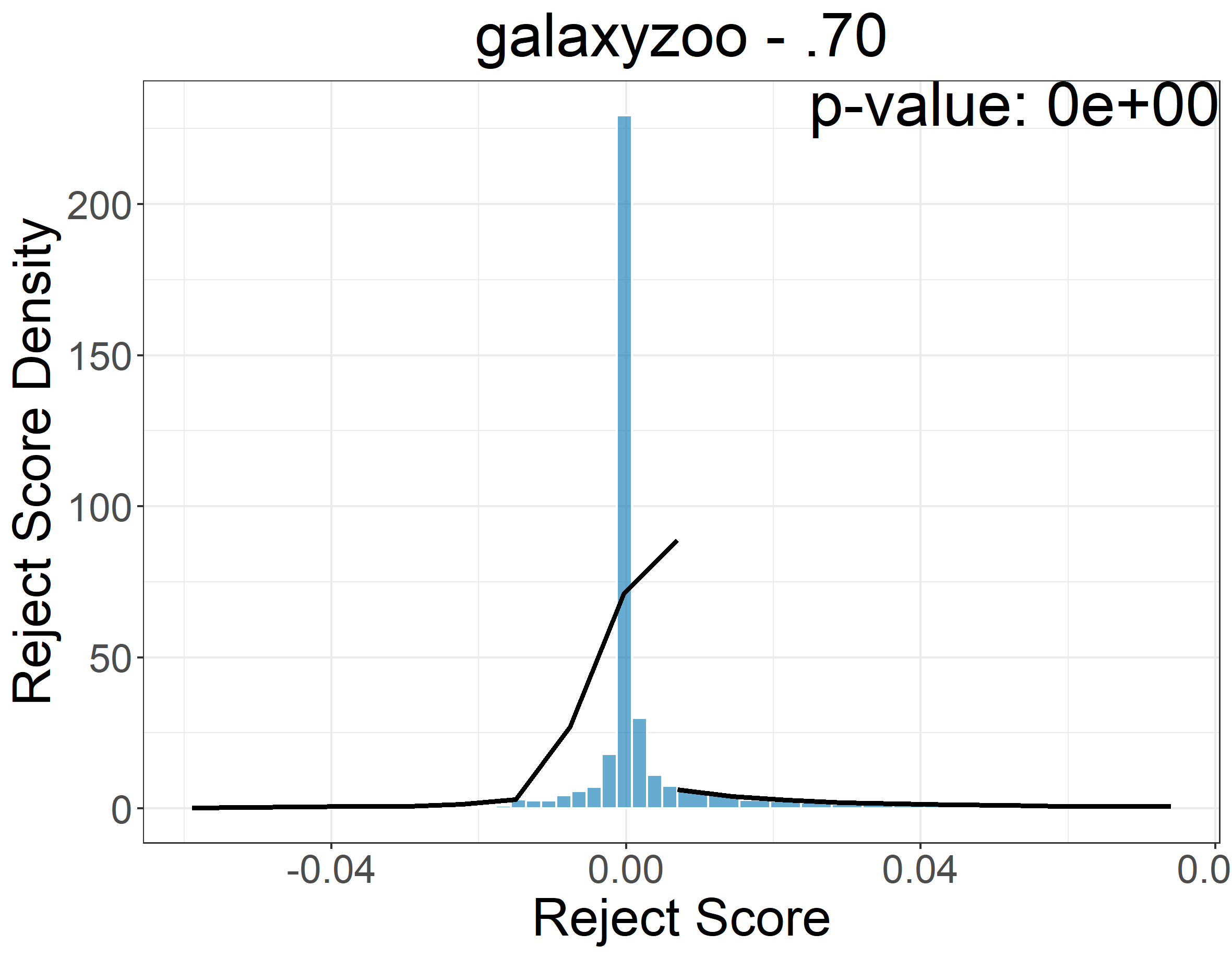}
        \caption{$\ok_{.70}$.}
        \label{fig:galaxyzoo_best_density_cutoff.70}
    \end{subfigure}
    \hfill
    \begin{subfigure}[t]{.32\textwidth}
        \includegraphics[scale=.22]{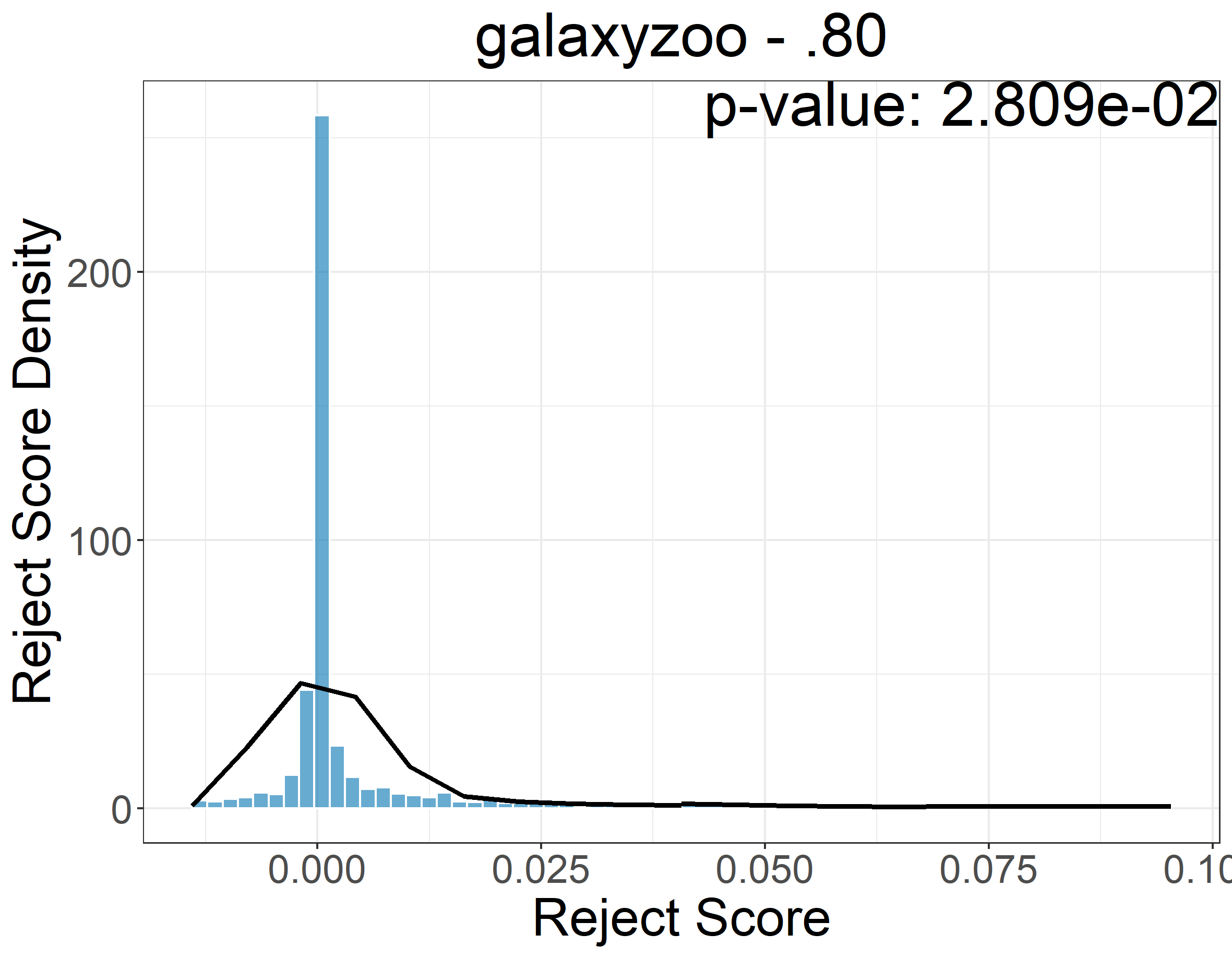}
        \caption{$\ok_{.80}$.}
        \label{fig:galaxyzoo_best_density_cutoff.80}
    \end{subfigure}
    \hfill
    \begin{subfigure}[t]{.32\textwidth}
        \includegraphics[scale=.22]{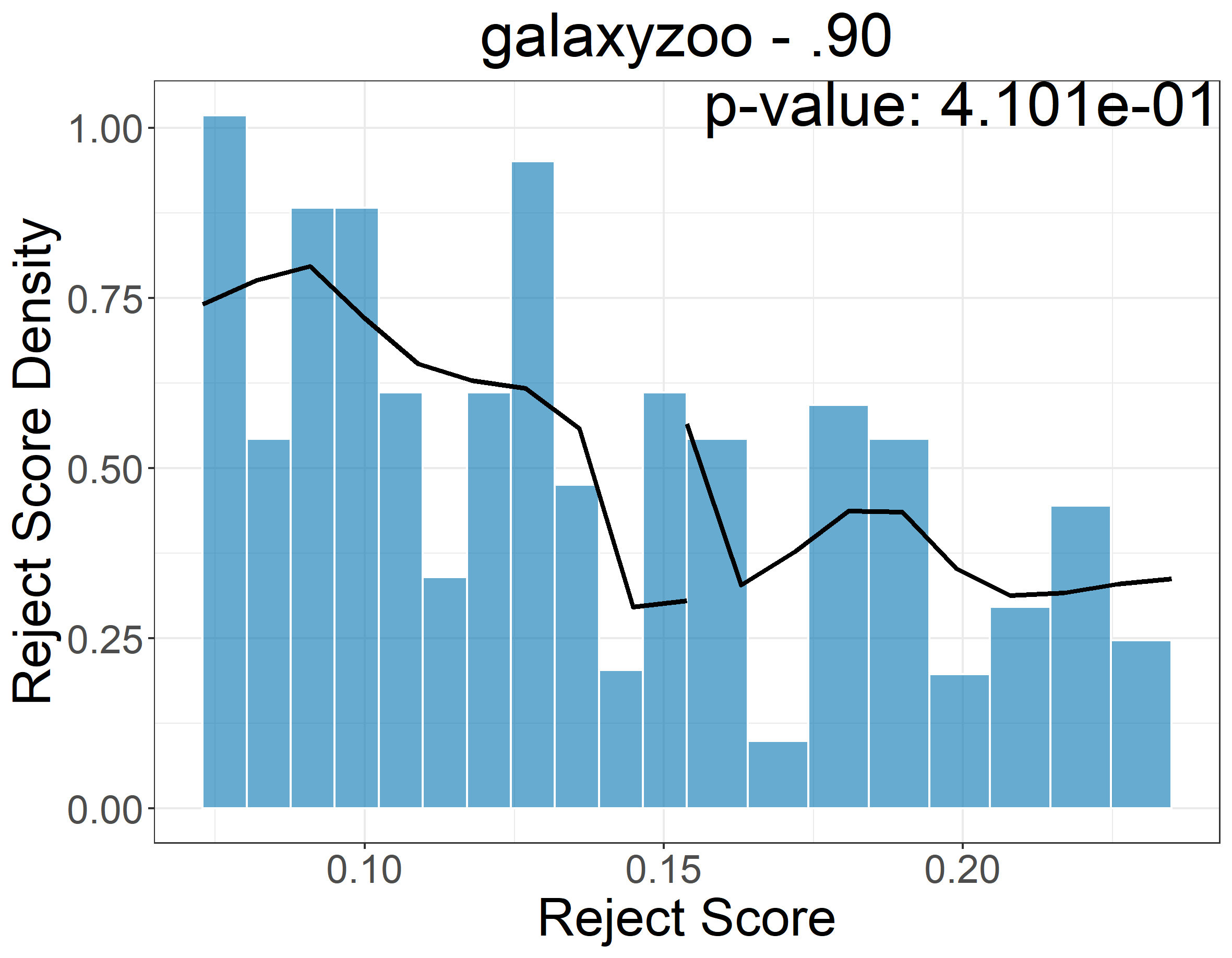}
        \caption{$\ok_{.90}$.}
        \label{fig:galaxyzoo_best_density_cutoff.90}
    \end{subfigure}
    \hfill
    
    \caption{
    \texttt{galaxyzoo} estimated reject scores densities at the left and right of cutoff $\ok_c$ for the best baseline \CC{}. All the plots are zoomed around the cutoff values.
    }
    \label{fig:galaxyzoo estimated density best}
\end{figure*}
}

\afterpage{
\begin{figure*}[t!]
    \begin{subfigure}[t]{.32\textwidth}
        \includegraphics[scale=.22]{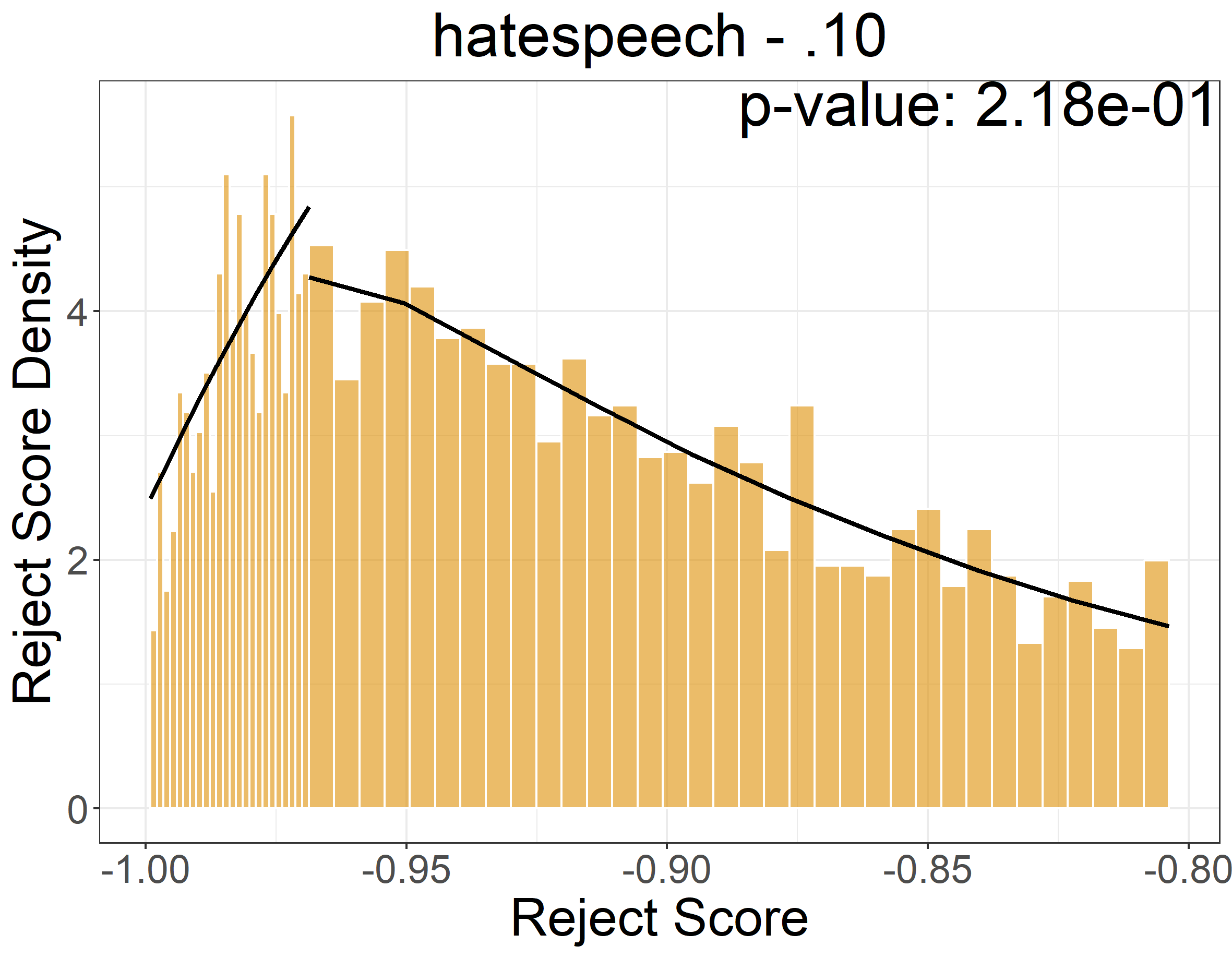}
        \caption{$\ok_{.10}$.}
        \label{fig:hatespeech_best_density_cutoff.10}
    \end{subfigure}
    \hfill
    \begin{subfigure}[t]{.32\textwidth}
        \includegraphics[scale=.22]{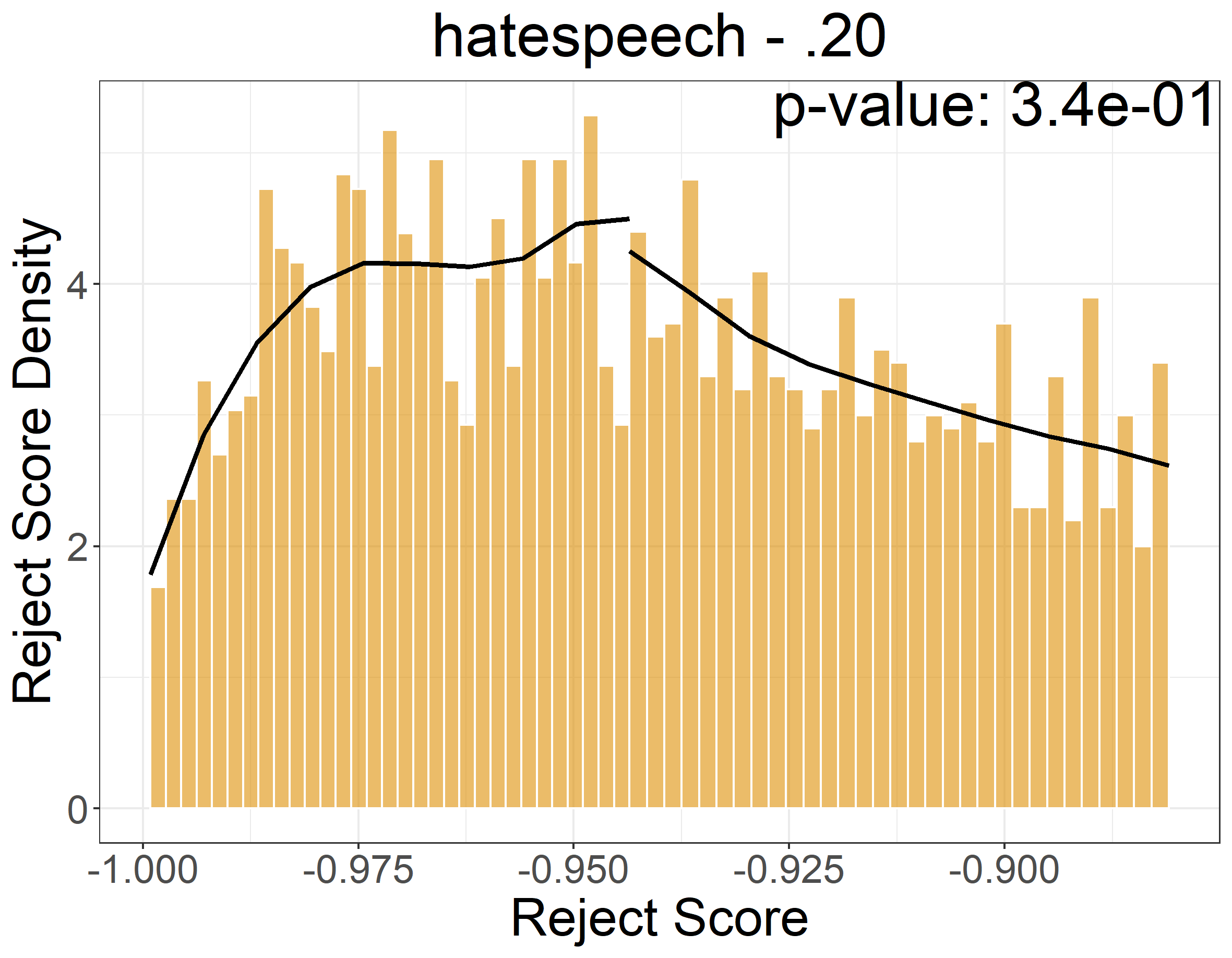}
        \caption{$\ok_{.20}$.}
        \label{fig:hatespeech_best_density_cutoff.20}
    \end{subfigure}
    \hfill
    \begin{subfigure}[t]{.32\textwidth}
        \includegraphics[scale=.22]{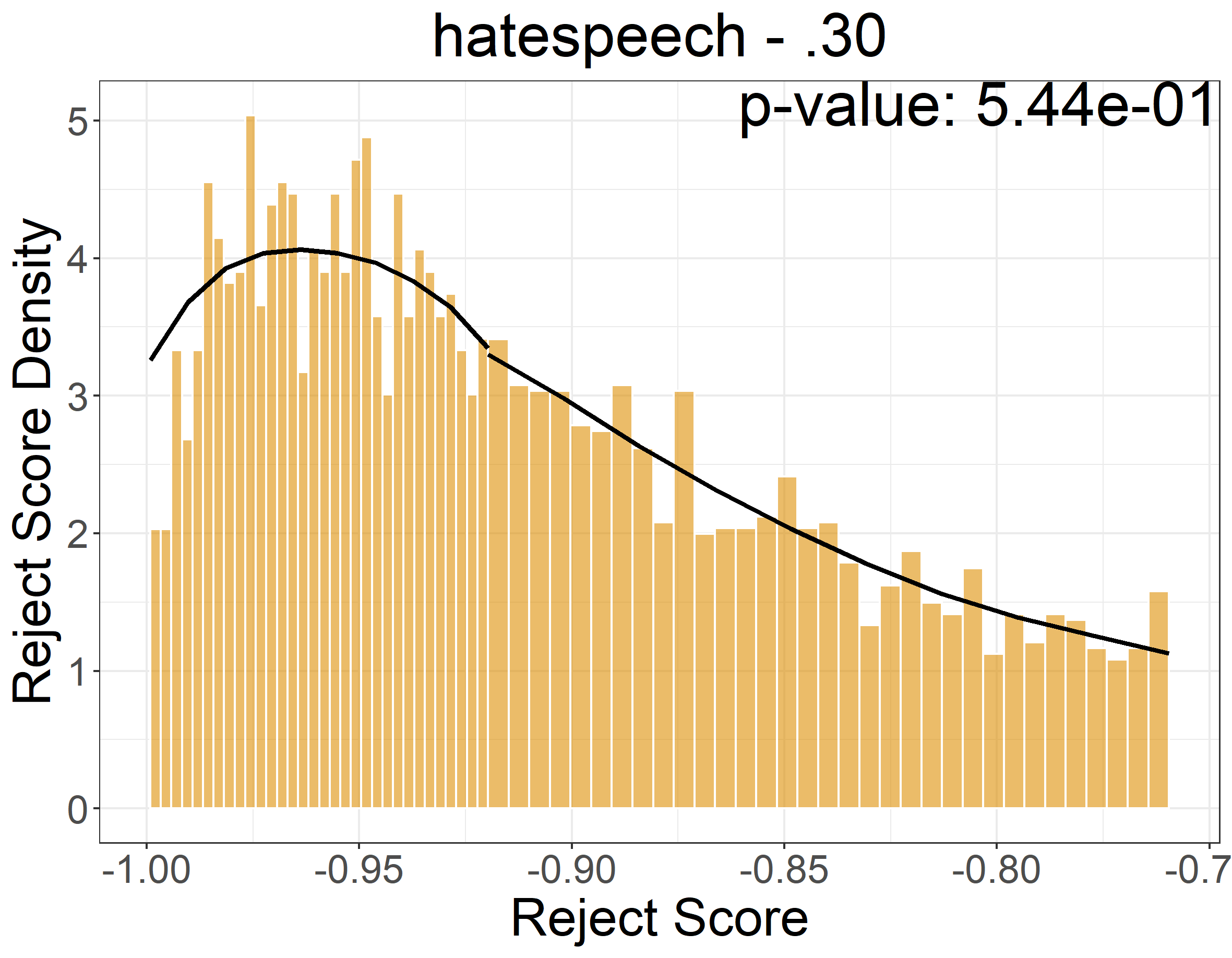}
        \caption{$\ok_{.30}$.}
        \label{fig:hatespeech_best_density_cutoff.30}
    \end{subfigure}
    \hfill
    \begin{subfigure}[t]{.32\textwidth}
        \includegraphics[scale=.22]{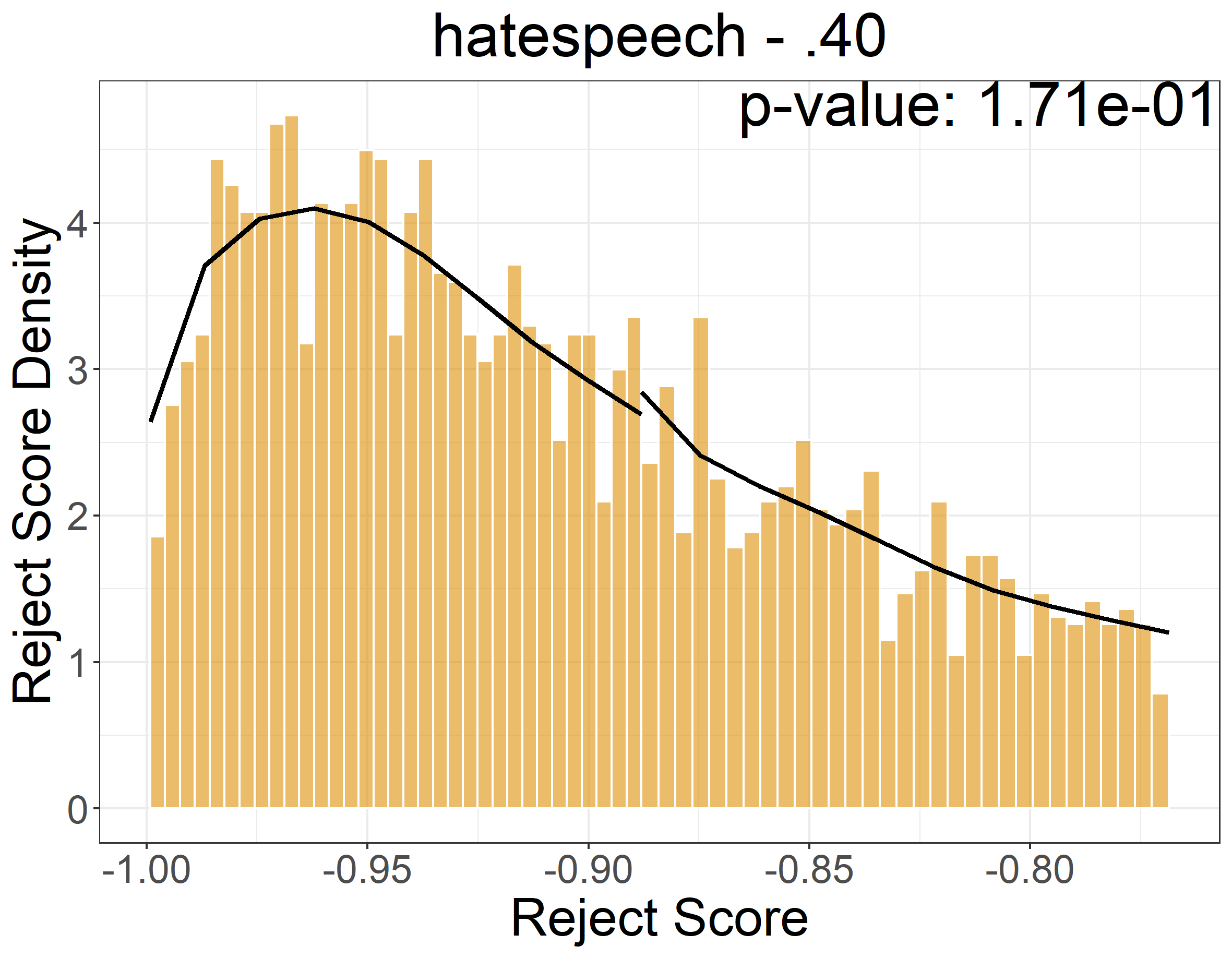}
        \caption{$\ok_{.40}$.}
        \label{fig:hatespeech_best_density_cutoff.40}
    \end{subfigure}
    \hfill
    \begin{subfigure}[t]{.32\textwidth}
        \includegraphics[scale=.22]{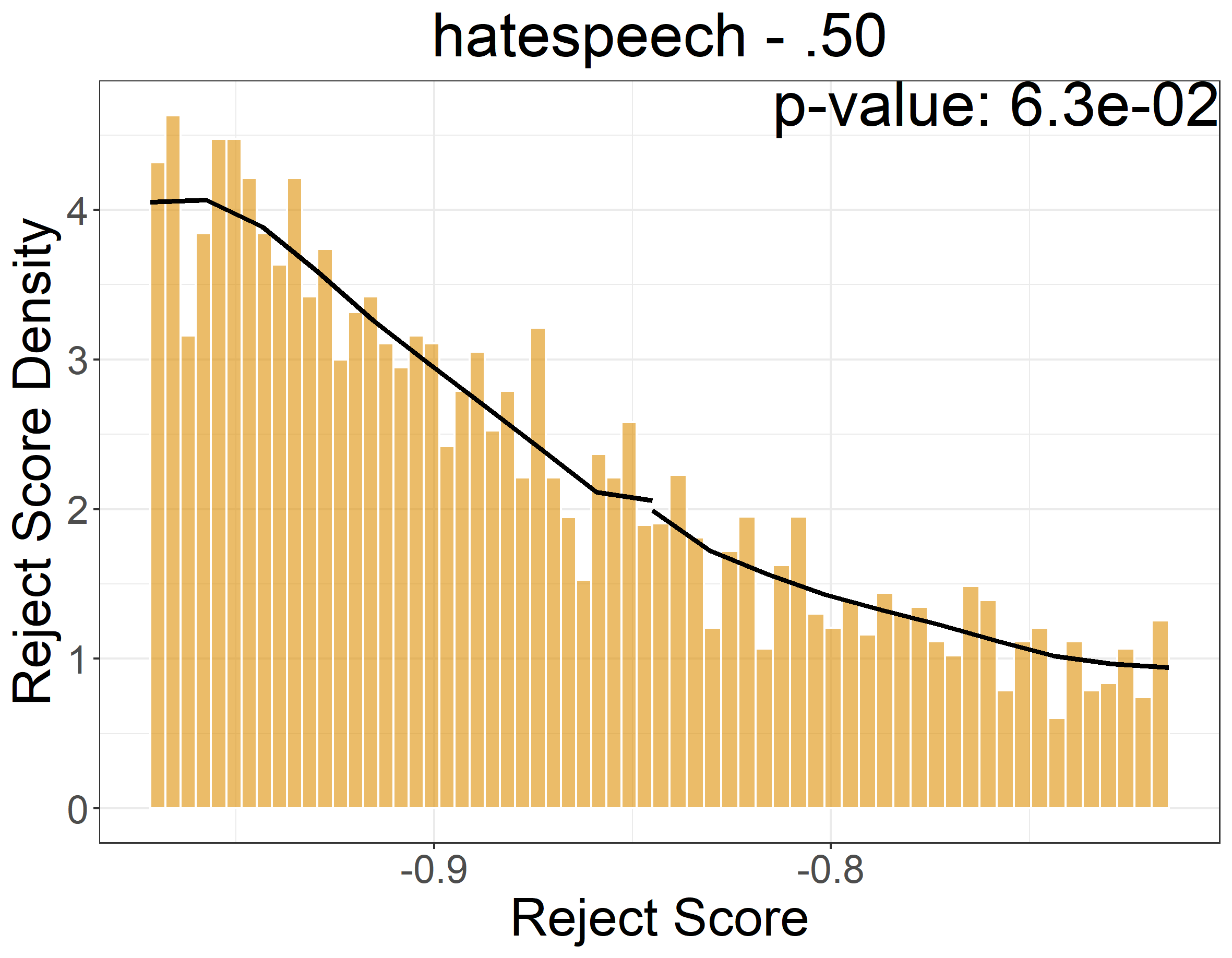}
        \caption{$\ok_{.50}$.}
        \label{fig:hatespeech_best_density_cutoff.50}
    \end{subfigure}
    \hfill
    \begin{subfigure}[t]{.32\textwidth}
        \includegraphics[scale=.22]{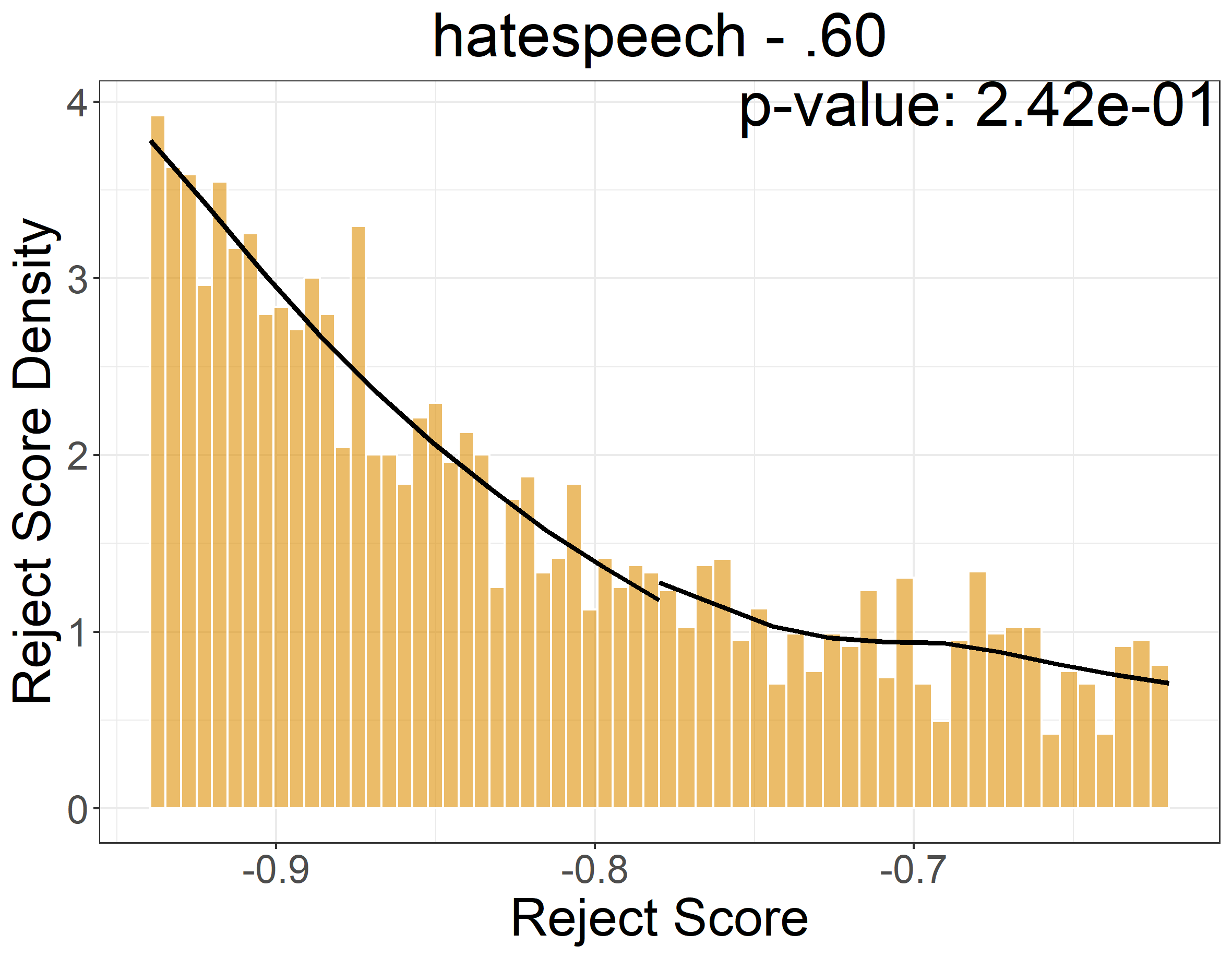}
        \caption{$\ok_{.60}$.}
        \label{fig:hatespeech_best_density_cutoff.60}
    \end{subfigure}
    \hfill
    \begin{subfigure}[t]{.32\textwidth}
        \includegraphics[scale=.22]{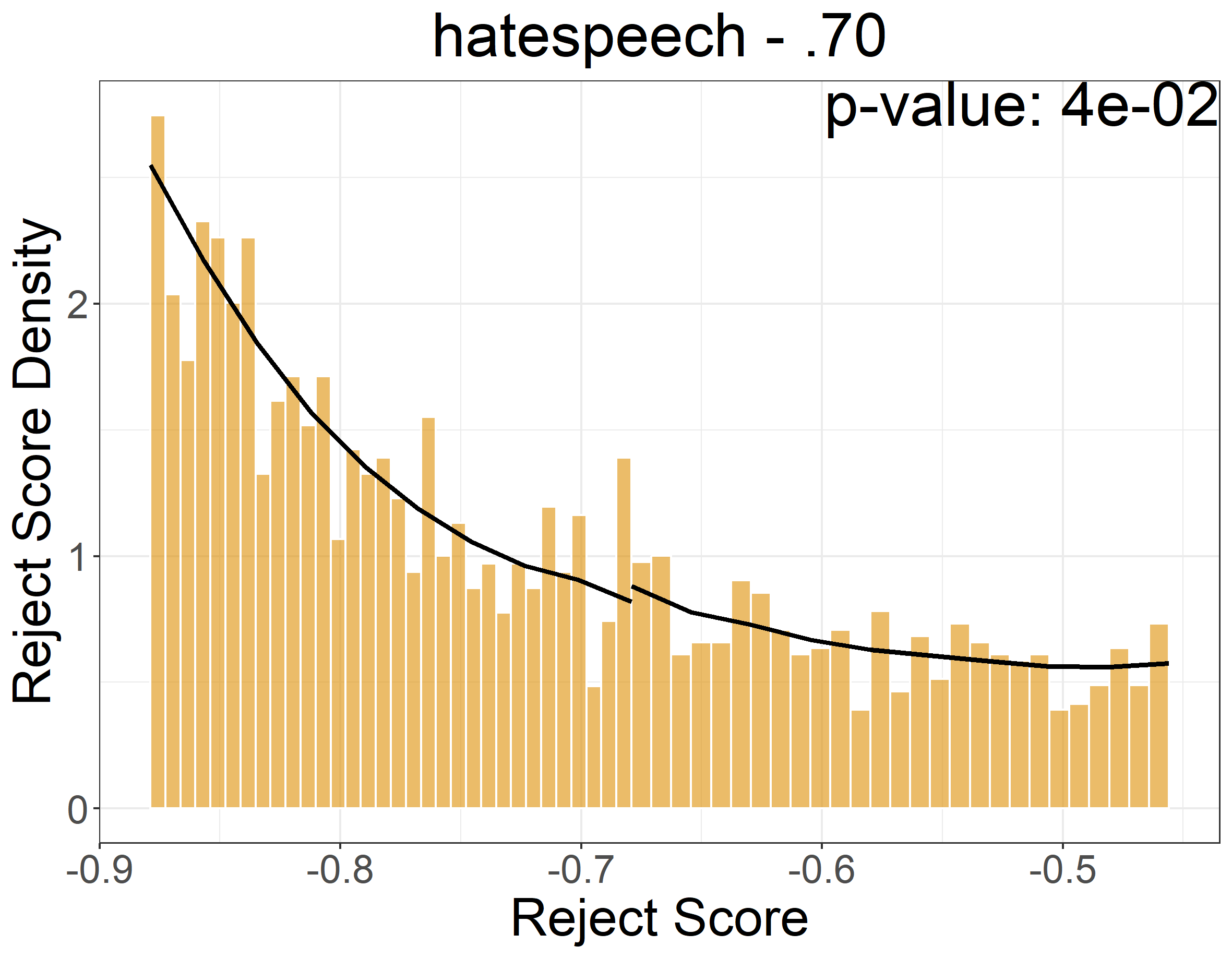}
        \caption{$\ok_{.70}$.}
        \label{fig:hatespeech_best_density_cutoff.70}
    \end{subfigure}
    \hfill
    \begin{subfigure}[t]{.32\textwidth}
        \includegraphics[scale=.22]{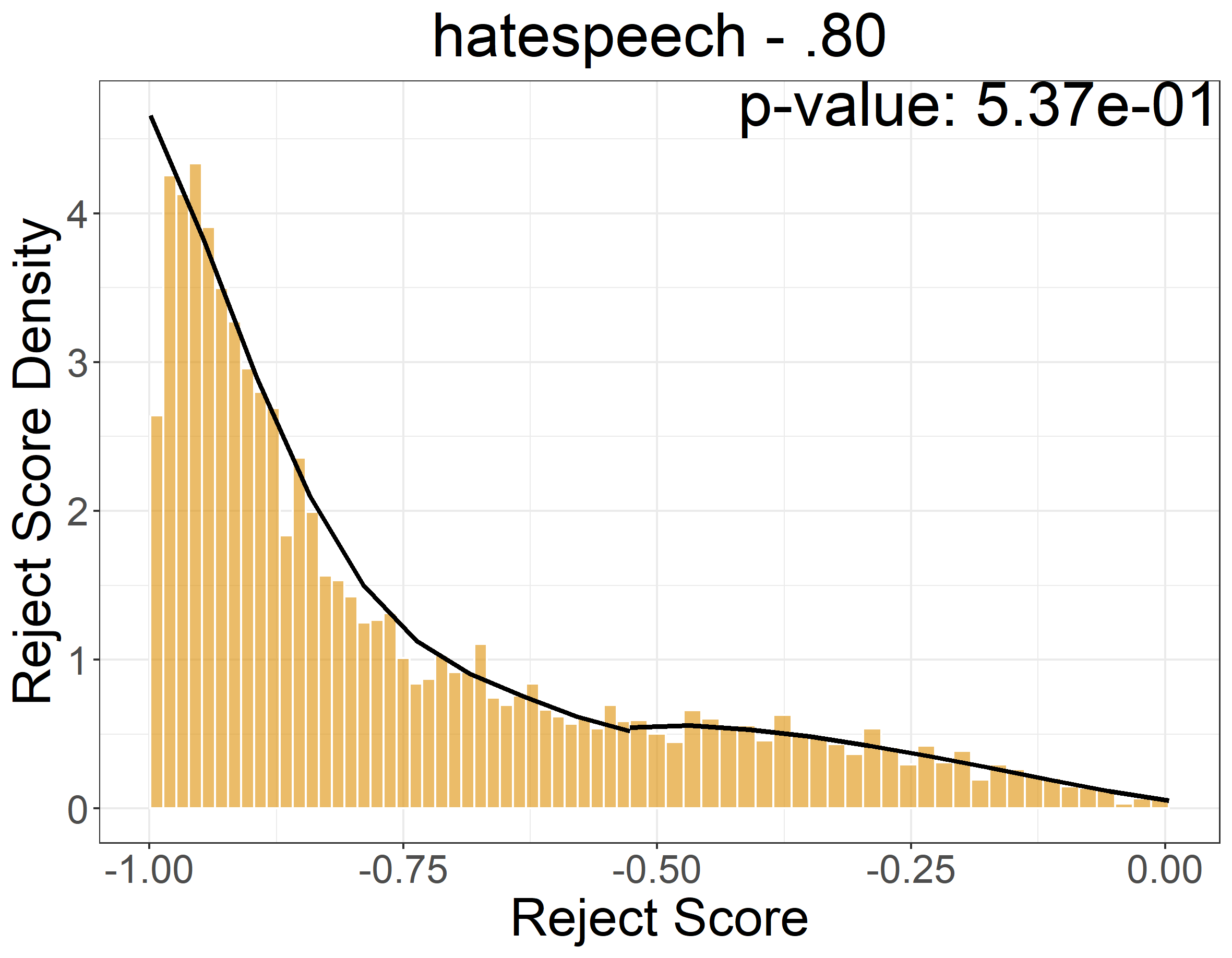}
        \caption{$\ok_{.80}$.}
        \label{fig:hatespeech_best_density_cutoff.80}
    \end{subfigure}
    \hfill
    \begin{subfigure}[t]{.32\textwidth}
        \includegraphics[scale=.22]{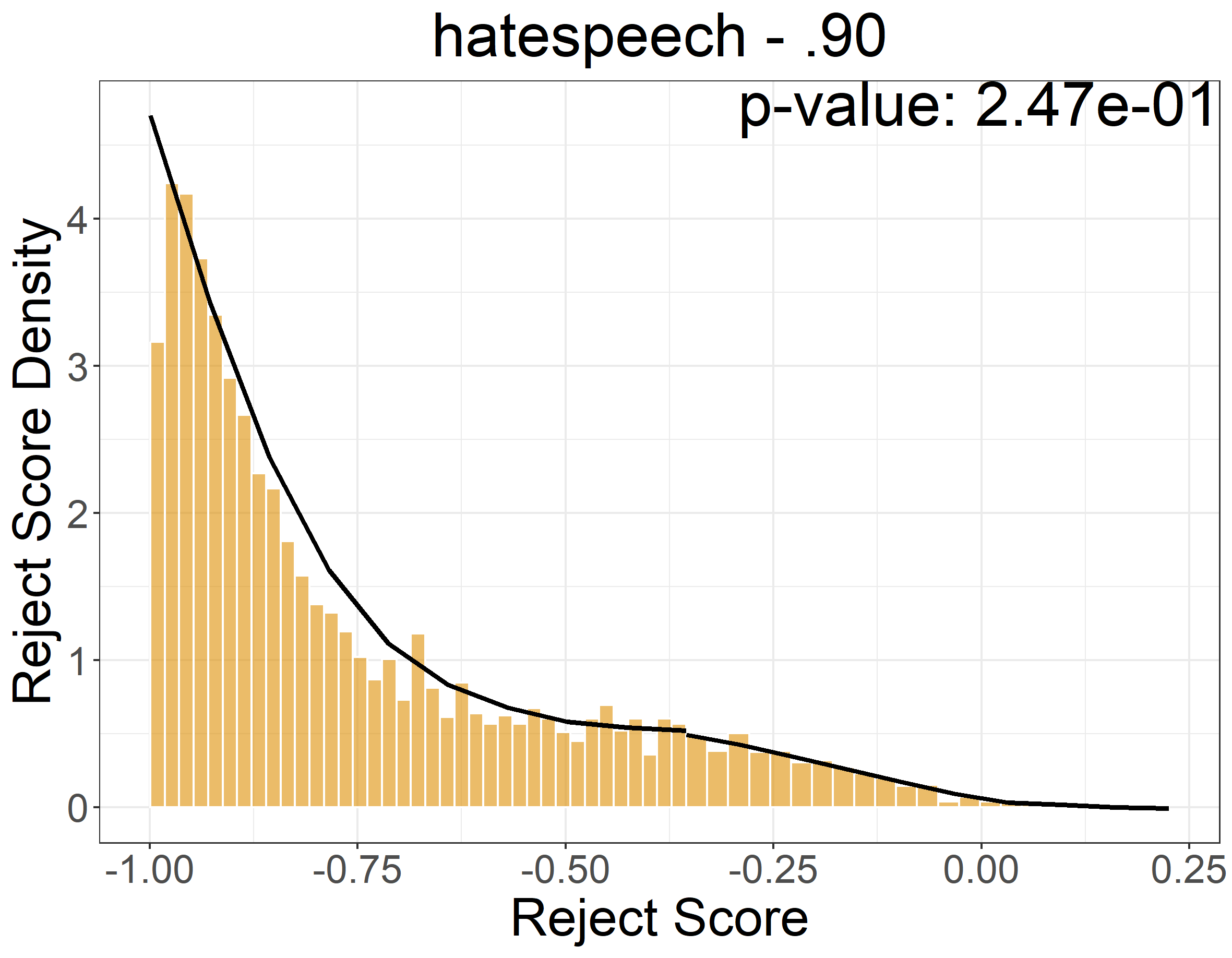}
        \caption{$\ok_{.90}$.}
        \label{fig:hatespeech_best_density_cutoff.90}
    \end{subfigure}
    \hfill
    
    \caption{
    \texttt{hatespeech} estimated best baseline (\RS{}) reject scores densities at the left and right of cutoff $\ok_c$. All the plots are zoomed around the cutoff values.
    }
    \label{fig:hatespeech estimated density best}
\end{figure*}
}

\afterpage{
\begin{figure*}[t!]
    \begin{subfigure}[t]{.32\textwidth}
        \includegraphics[scale=.22]{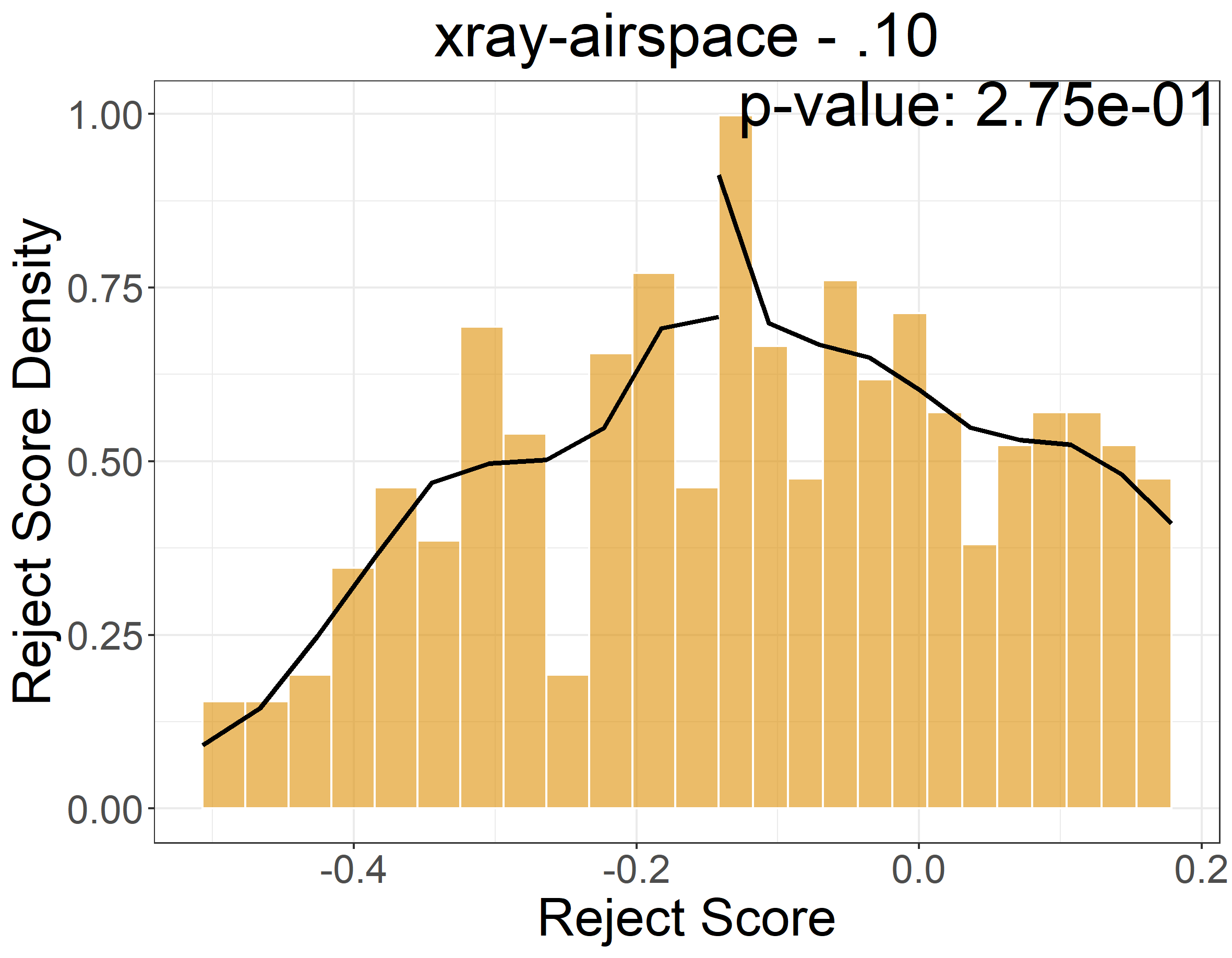}
        \caption{$\ok_{.10}$.}
        \label{fig:xray-airspace_best_density_cutoff.10}
    \end{subfigure}
    \hfill
    \begin{subfigure}[t]{.32\textwidth}
        \includegraphics[scale=.22]{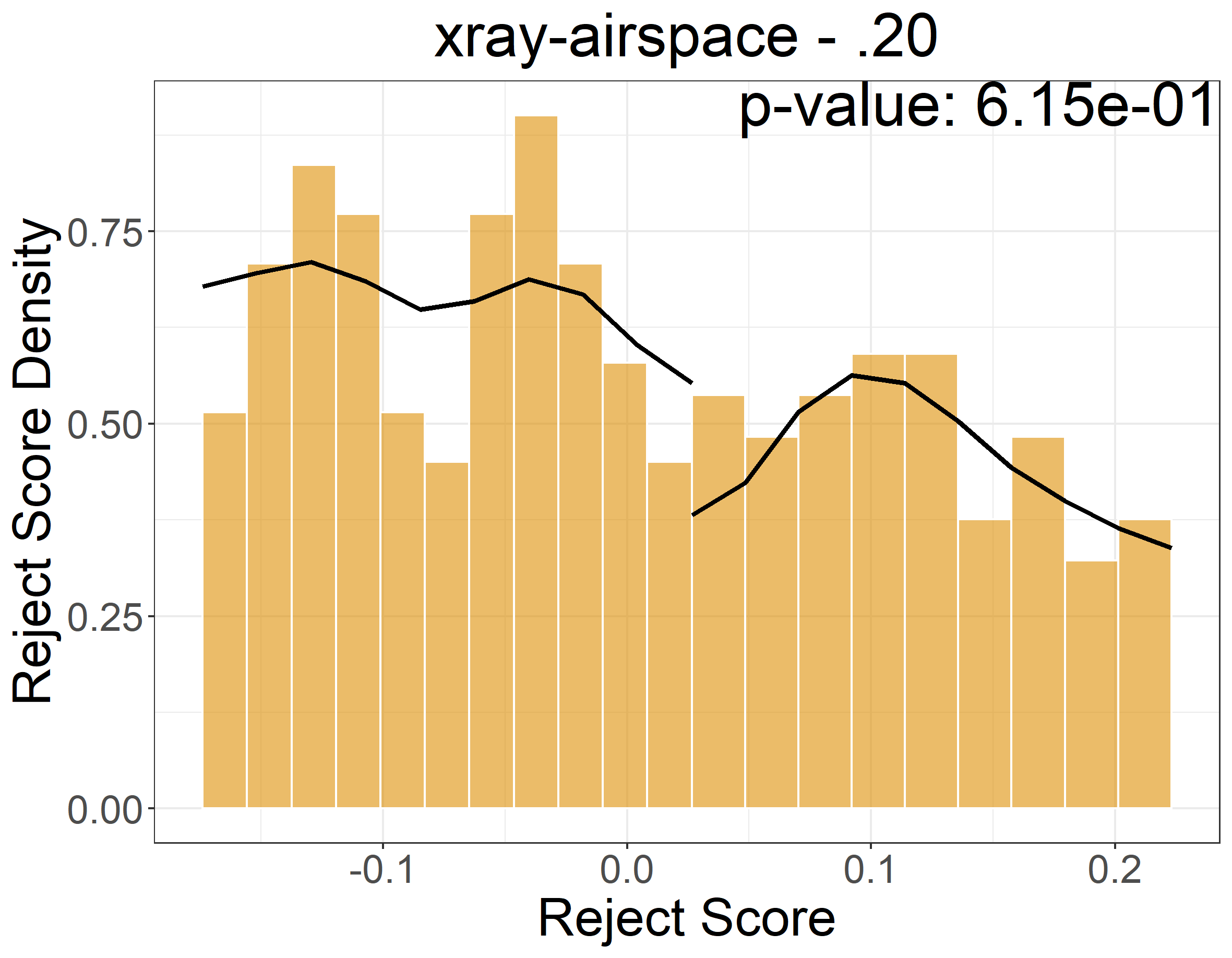}
        \caption{$\ok_{.20}$.}
        \label{fig:xray-airspace_best_density_cutoff.20}
    \end{subfigure}
    \hfill
    \begin{subfigure}[t]{.32\textwidth}
        \includegraphics[scale=.22]{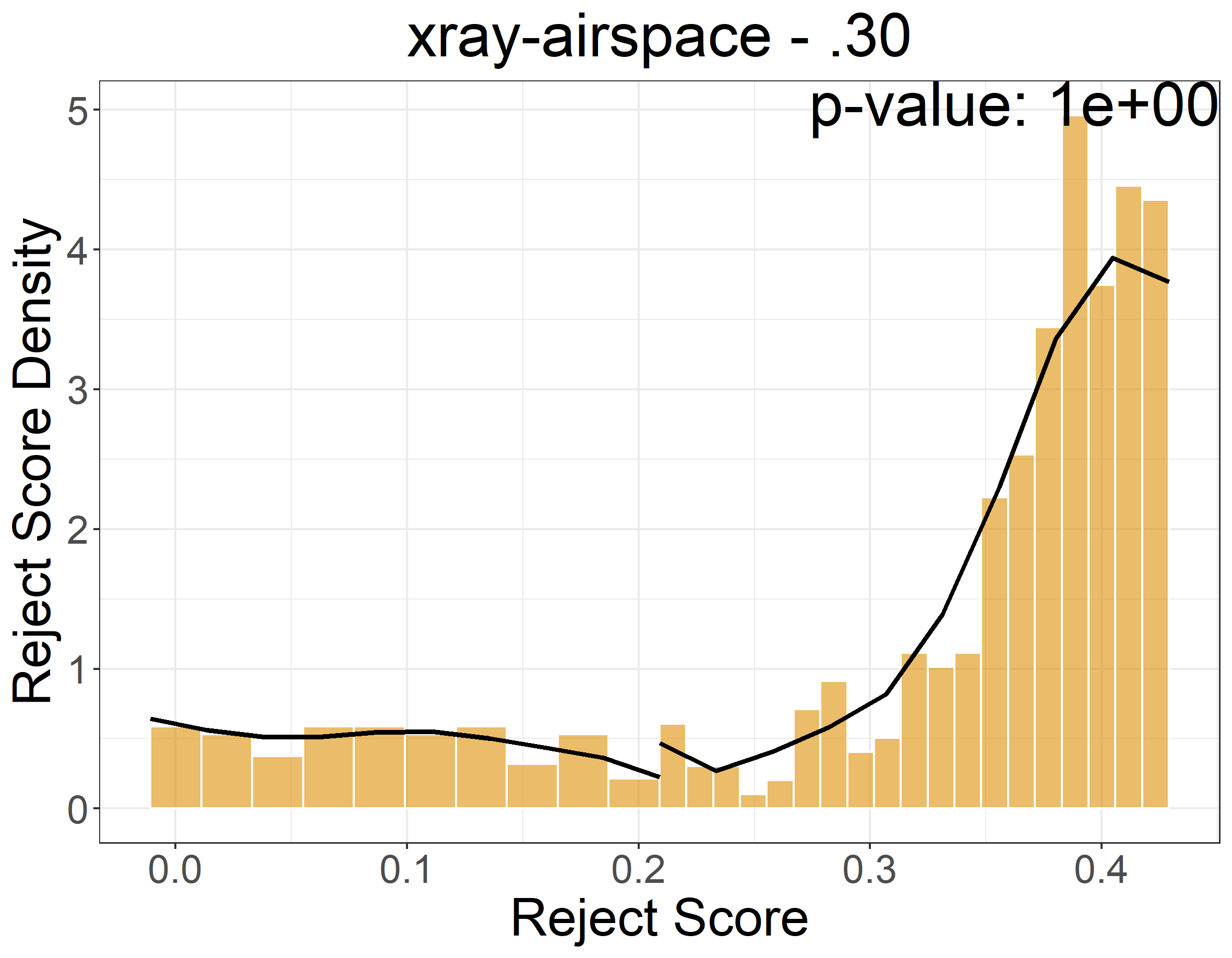}
        \caption{$\ok_{.30}$.}
        \label{fig:xray-airspace_best_density_cutoff.30}
    \end{subfigure}
    \hfill
    \begin{subfigure}[t]{.32\textwidth}
        \includegraphics[scale=.22]{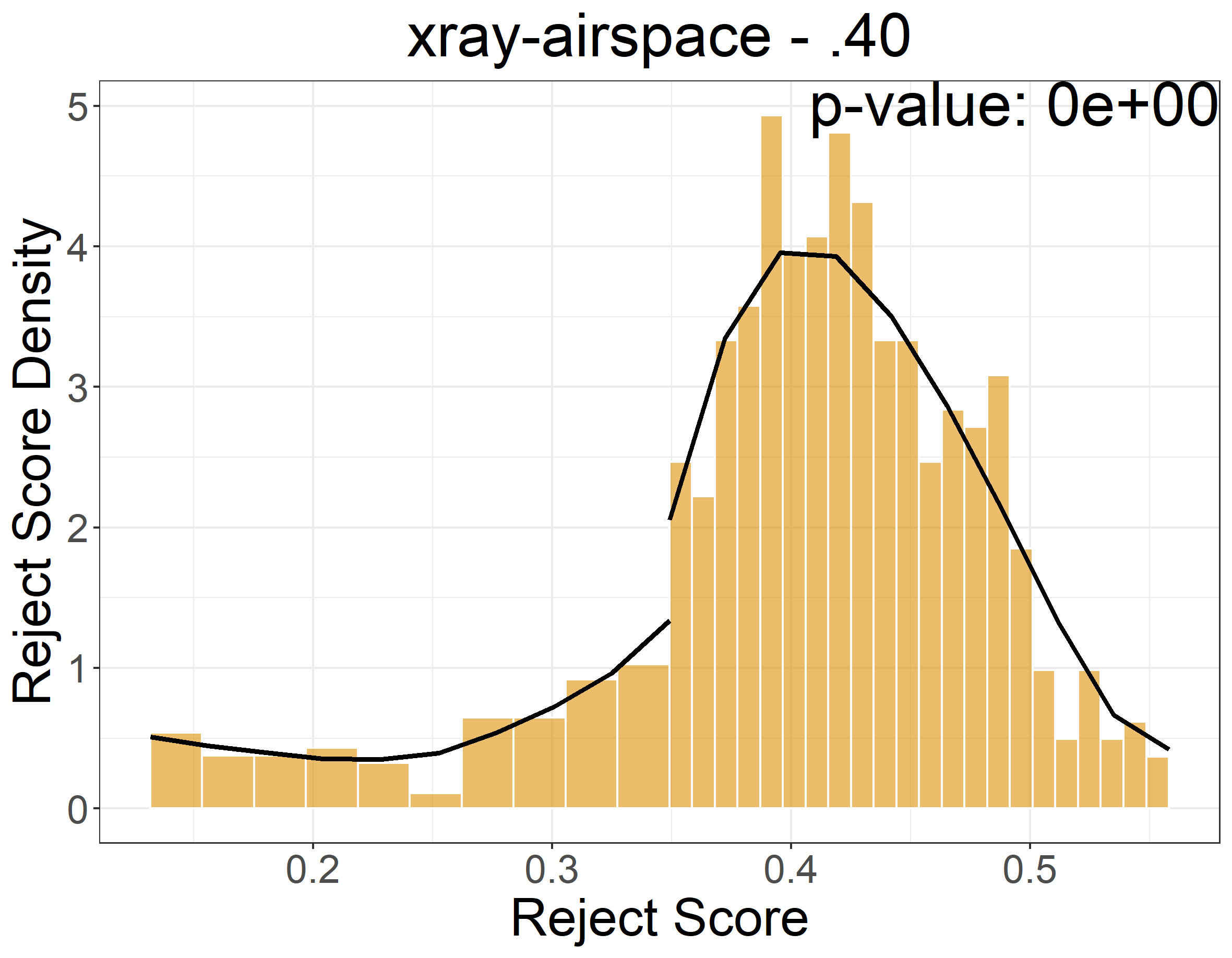}
        \caption{$\ok_{.40}$.}
        \label{fig:xray-airspace_best_density_cutoff.40}
    \end{subfigure}
    \hfill
    \begin{subfigure}[t]{.32\textwidth}
        \includegraphics[scale=.22]{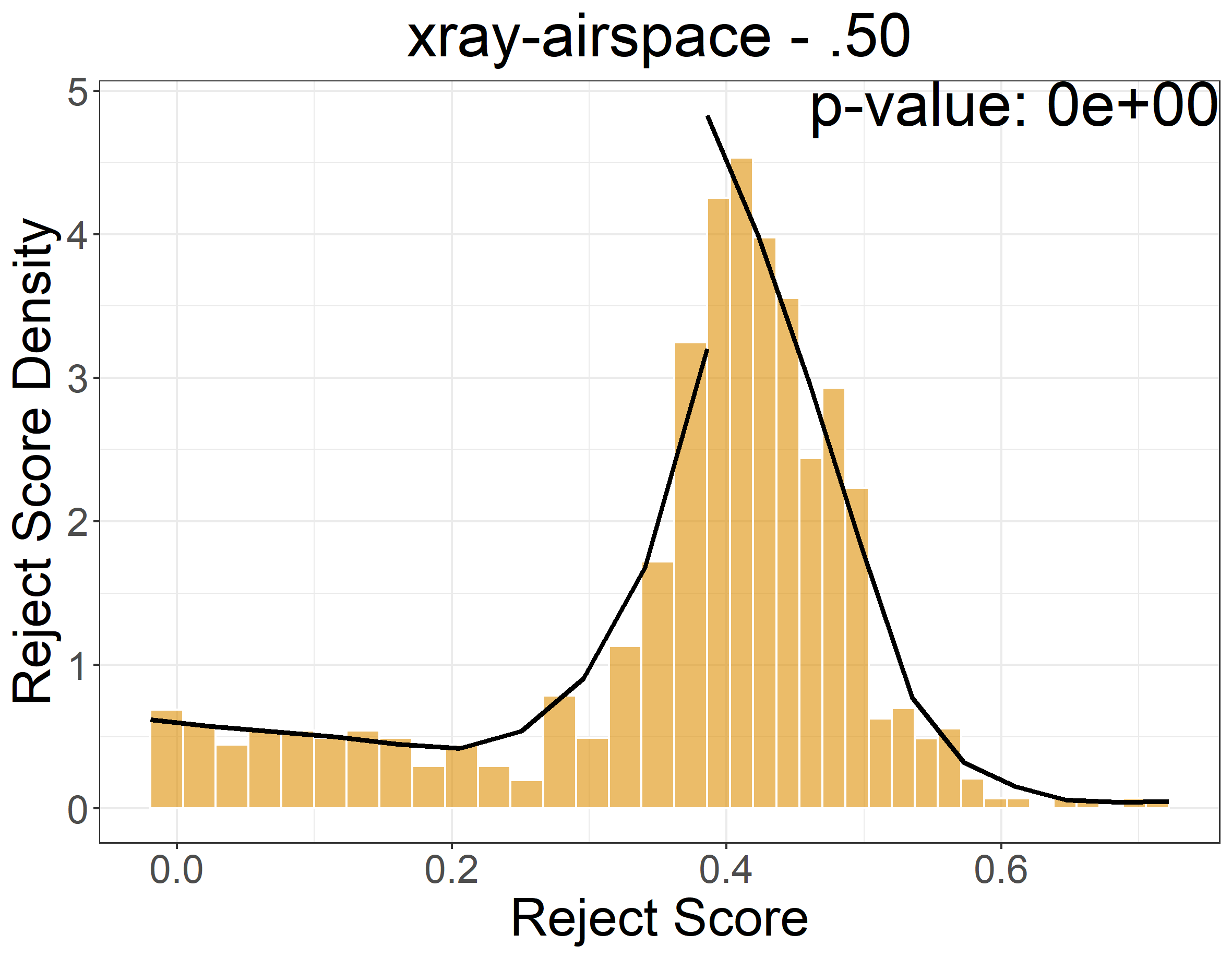}
        \caption{$\ok_{.50}$.}
        \label{fig:xray-airspace_best_density_cutoff.50}
    \end{subfigure}
    \hfill
    \begin{subfigure}[t]{.32\textwidth}
        \includegraphics[scale=.22]{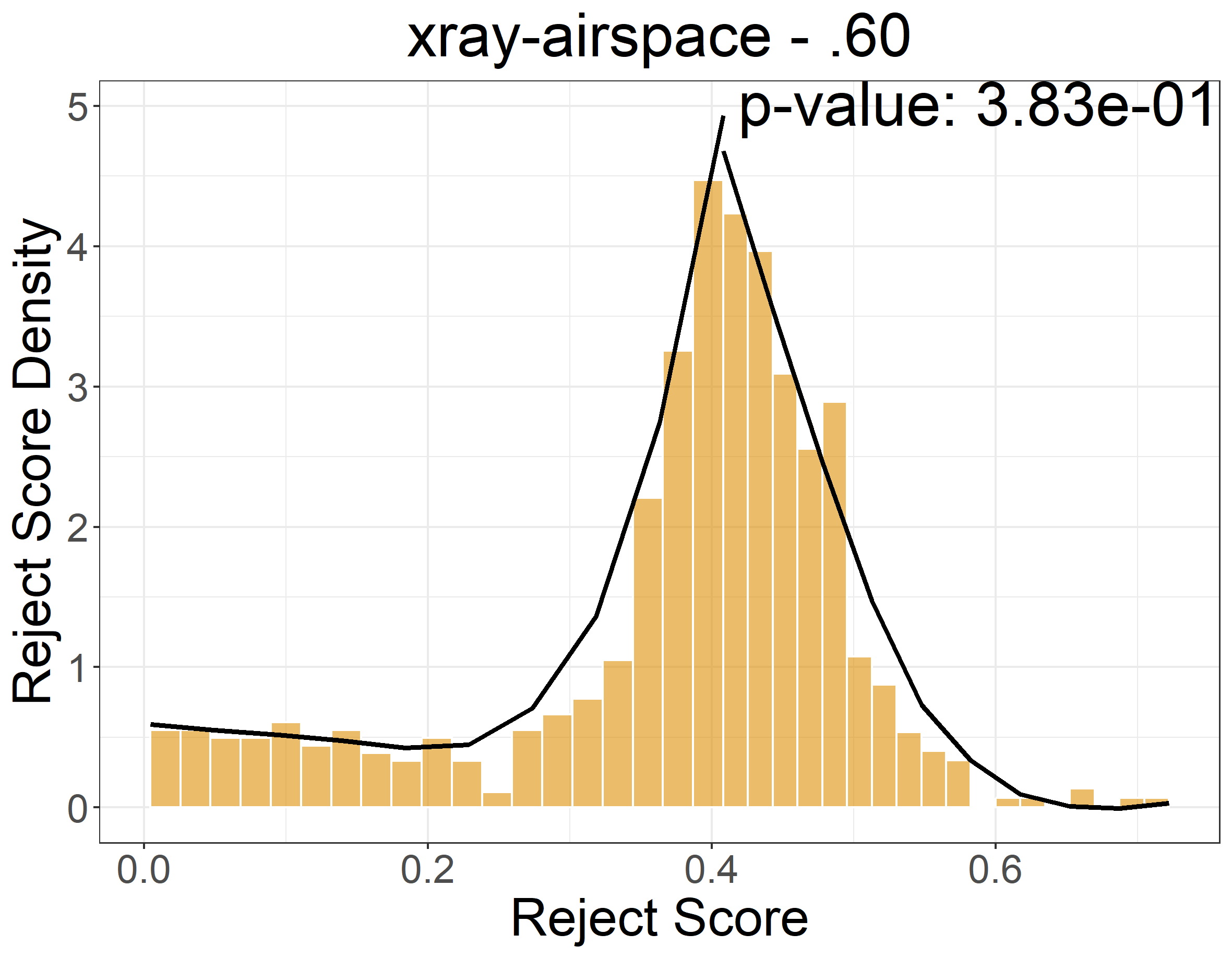}
        \caption{$\ok_{.60}$.}
        \label{fig:xray-airspace_best_density_cutoff.60}
    \end{subfigure}
    \hfill
    \begin{subfigure}[t]{.32\textwidth}
        \includegraphics[scale=.22]{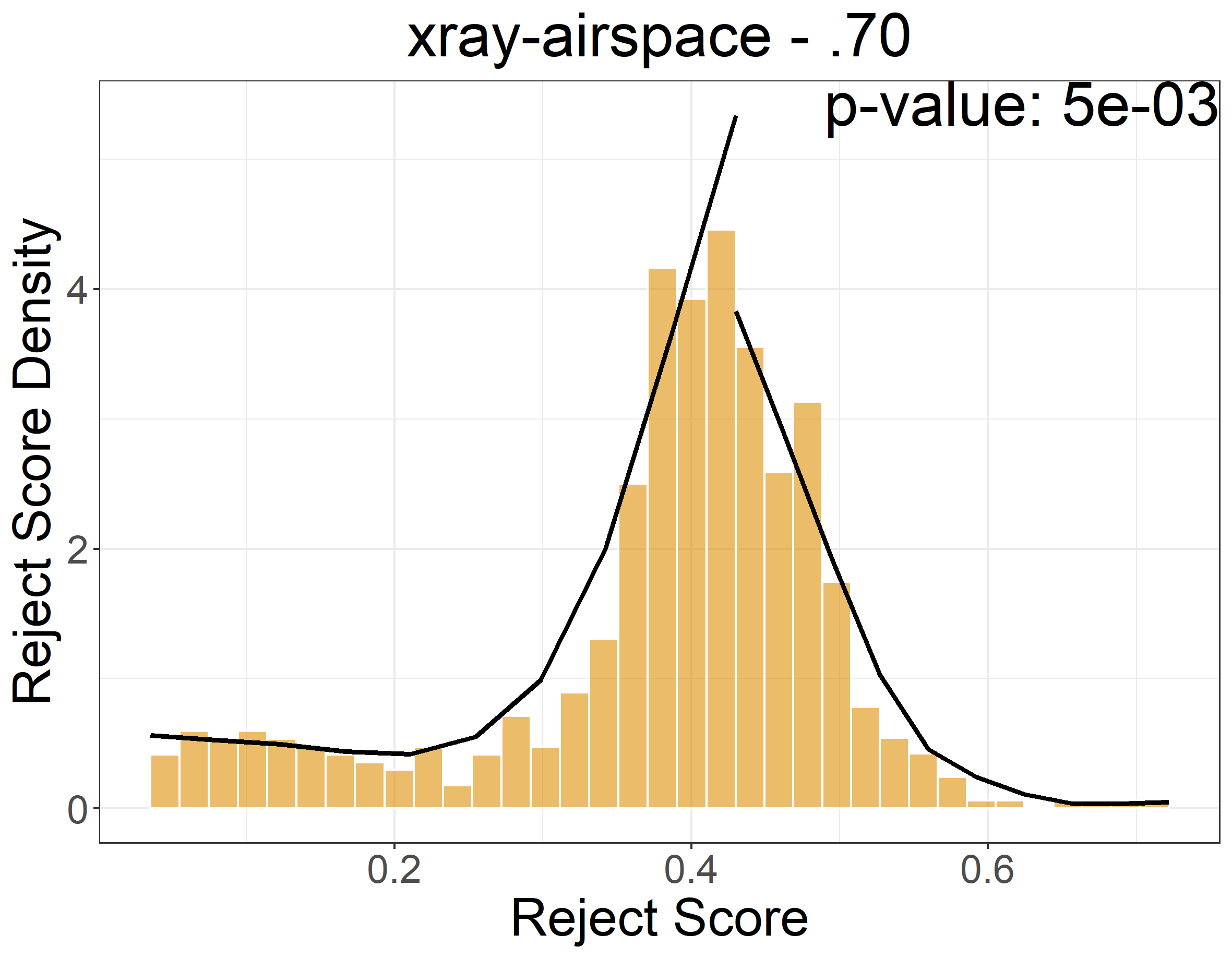}
        \caption{$\ok_{.70}$.}
        \label{fig:xray-airspace_best_density_cutoff.70}
    \end{subfigure}
    \hfill
    \begin{subfigure}[t]{.32\textwidth}
        \includegraphics[scale=.22]{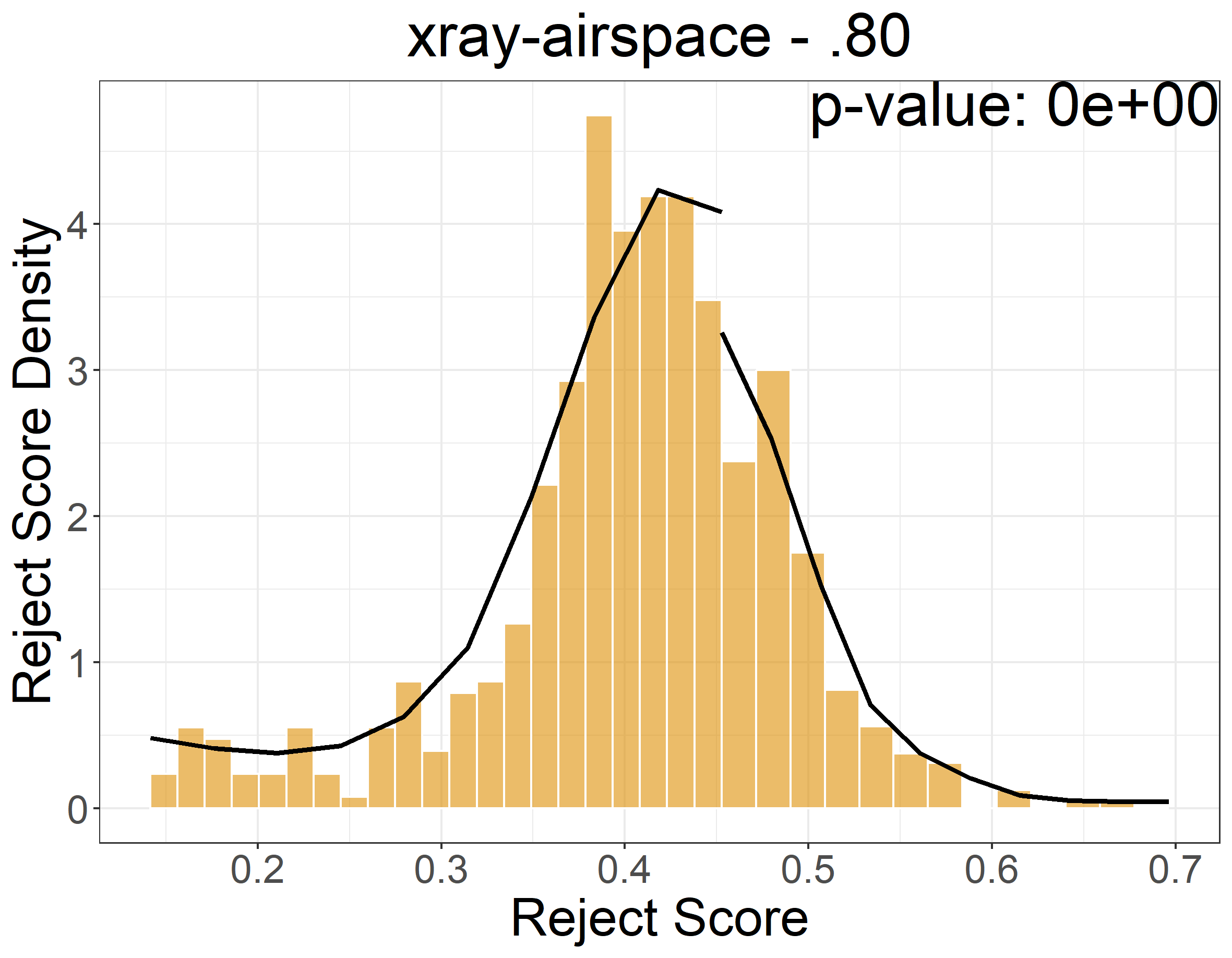}
        \caption{$\ok_{.80}$.}
        \label{fig:xray-airspace_best_density_cutoff.80}
    \end{subfigure}
    \hfill
    \begin{subfigure}[t]{.32\textwidth}
        \includegraphics[scale=.22]{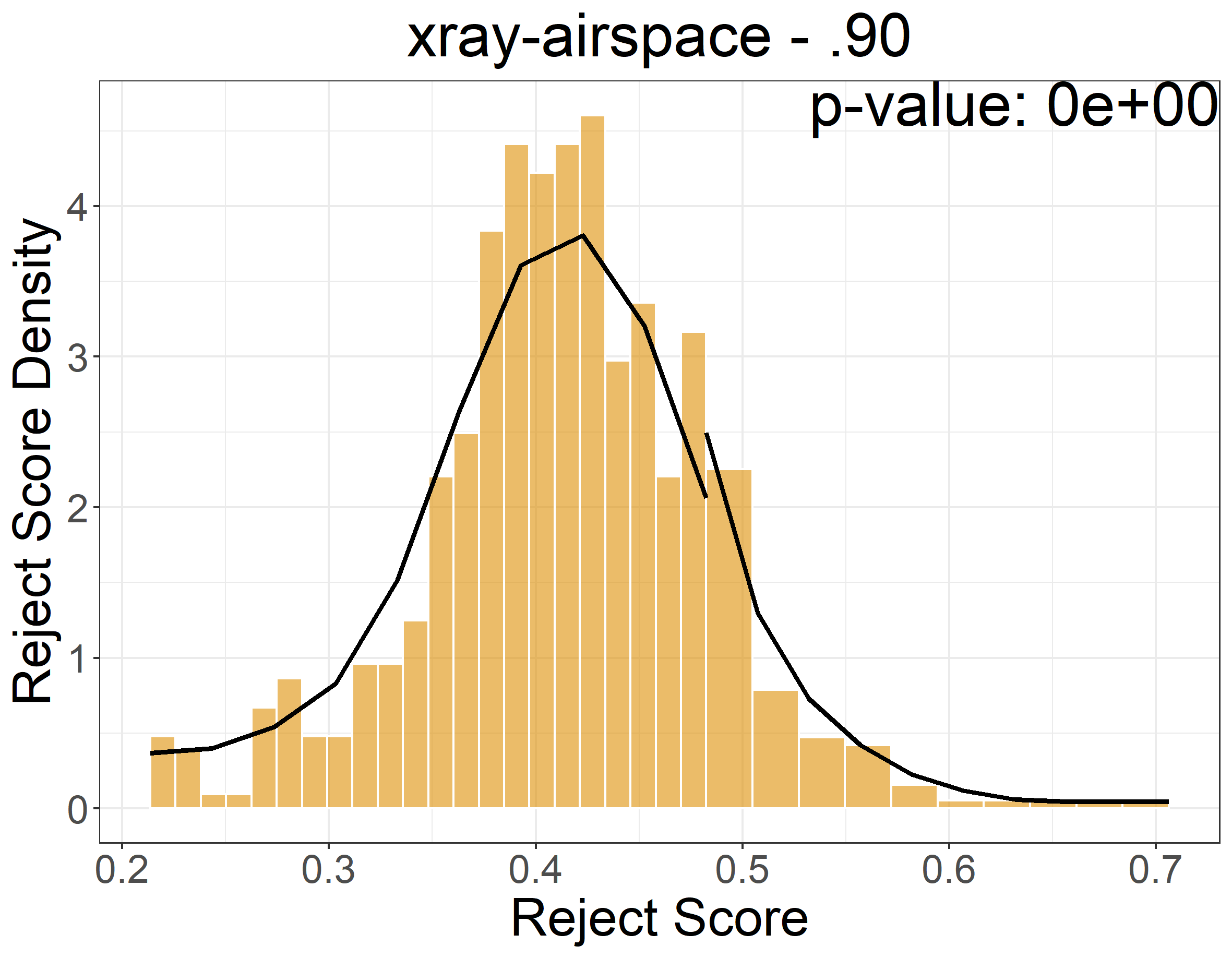}
        \caption{$\ok_{.90}$.}
        \label{fig:xray-airspace_best_density_cutoff.90}
    \end{subfigure}
    \hfill
    
    \caption{
    \texttt{xray-airspace} estimated best baseline (\RS{}) reject scores densities at the left and right of cutoff $\ok_c$. All the plots are zoomed around the cutoff values.
    }
    \label{fig:xray-airspace estimated density best}
\end{figure*}
}

\end{document}